\documentclass{article} 
\usepackage{iclr2027_conference,times}
\iclrfinalcopy
\usepackage{fancyhdr}

\usepackage{amsmath,amsfonts,bm}

\def\eqref#1{equation~\ref{#1}}

\def\1{\bm{1}}

\def\ve{{\bm{e}}}

\def\vw{{\bm{w}}}
\def\vx{{\bm{x}}}
\def\vy{{\bm{y}}}

\def\mL{{\bm{L}}}

\def\mX{{\bm{X}}}

\DeclareMathAlphabet{\mathsfit}{\encodingdefault}{\sfdefault}{m}{sl}
\SetMathAlphabet{\mathsfit}{bold}{\encodingdefault}{\sfdefault}{bx}{n}

\def\gA{{\mathcal{A}}}
\def\gB{{\mathcal{B}}}

\def\gE{{\mathcal{E}}}
\def\gF{{\mathcal{F}}}
\def\gG{{\mathcal{G}}}
\def\gH{{\mathcal{H}}}
\def\gI{{\mathcal{I}}}

\def\gS{{\mathcal{S}}}
\def\gT{{\mathcal{T}}}

\def\gV{{\mathcal{V}}}

\def\gX{{\mathcal{X}}}

\usepackage{url}
\usepackage{subcaption}

\RequirePackage{fix-cm}

\usepackage{textcomp}
\usepackage{tikz}
\usetikzlibrary{calc}

\usepackage{caption}
\usepackage{placeins}
\usepackage{comment}
\usepackage[inline]{enumitem}
\usepackage{xcolor}
\definecolor{linkblue}{RGB}{0, 70, 140}

\usepackage[
    colorlinks=true,
    linkcolor=linkblue,
    urlcolor=linkblue,
    citecolor=linkblue
]{hyperref}

\usepackage{colortbl}

\usepackage{xspace}
\usepackage{graphicx}
\usepackage{rotating}

\usepackage{tabularx}
\usepackage{booktabs}
\usepackage{multirow}
\usepackage{makecell}

\usepackage{algorithm}
\usepackage{algorithmic}
\usepackage{wrapfig}

\newif\ifshowcomments
\showcommentsfalse
\newif\ifshowkncomments
\showkncommentstrue
\newif\ifshowmhcomments
\showmhcommentstrue
\newif\ifshowsdcomments
\showsdcommentstrue

\newcommand{\SD}[1]{\ifshowcomments\textcolor{blue!70!black}{\textbf{[SD: #1]}}\else\ifshowsdcomments\textcolor{blue!70!black}{\textbf{[SD: #1]}}\fi\fi}

\newtheorem{theorem}{Theorem}[section]

\newtheorem{proposition}[theorem]{Proposition}

\newenvironment{proof}[1][Proof]{\par\noindent\textit{#1.}}{\hfill\rule{0.5em}{0.5em}\par}

\newcommand{\sgt}{\texttt{Scaffold}\xspace}
\newcommand{\sgte}{\texttt{Scaffold-Greedy}\xspace}
\newcommand{\sgtb}{\texttt{Scaffold-Batch}\xspace}
\newcommand{\sgth}{\texttt{Scaffold-Heap}\xspace}
\newcommand{\sgtf}{\texttt{Scaffold-Fast}\xspace}
\newcommand{\sgtsa}{\texttt{Scaffold-Sample}\xspace}

\makeatletter

\providecommand{\eCon}{}
\renewcommand{\eCon}{\@ifnextchar_{\eCon@sub}{c_E}}
\def\eCon@sub_#1{c_{E,#1}}

\providecommand{\vCon}{}
\renewcommand{\vCon}{\@ifnextchar_{\vCon@sub}{c_V}}
\def\vCon@sub_#1{c_{V,#1}}

\providecommand{\pathDil}{}
\renewcommand{\pathDil}{\@ifnextchar_{\pathDil@sub}{D^{\mathrm{path}}}}
\def\pathDil@sub_#1{D_{#1}^{\mathrm{path}}}

\providecommand{\eConPath}{}
\renewcommand{\eConPath}{\@ifnextchar_{\eConPath@sub}{C_E^{\mathrm{path}}}}
\def\eConPath@sub_#1{C_{E,#1}^{\mathrm{path}}}

\providecommand{\vConPath}{}
\renewcommand{\vConPath}{\@ifnextchar_{\vConPath@sub}{C_V^{\mathrm{path}}}}
\def\vConPath@sub_#1{C_{V,#1}^{\mathrm{path}}}

\makeatother

\providecommand{\pathDilMax}[1]{D_{#1,\max}^{\mathrm{path}}}
\providecommand{\eConPathMax}[1]{C_{E,#1,\max}^{\mathrm{path}}}
\providecommand{\vConPathMax}[1]{C_{V,#1,\max}^{\mathrm{path}}}

\providecommand{\score}{}
\renewcommand{\score}{s}

\providecommand{\pathE}{}
\renewcommand{\pathE}{p_E}

\providecommand{\pathN}{}
\renewcommand{\pathN}{p_V}

\usepackage{titlesec}
\usepackage{titletoc}

\titlespacing*{\section}
  {0pt}{1.2ex plus 0.2ex minus 0.2ex}{0.3ex}

\titlespacing*{\subsection}
  {0pt}{0.9ex plus 0.2ex minus 0.1ex}{0.25ex}

\titlespacing*{\subsubsection}
  {0pt}{0.7ex plus 0.2ex minus 0.1ex}{0.2ex}

\titlespacing*{\paragraph}
  {0pt}{0.4ex plus 0.1ex minus 0.1ex}{0.5em}

\newcommand{\scaffoldlink}{%
\href{https://github.com/siddhartha047/Scaffold}%
{\nolinkurl{github.com/siddhartha047/Scaffold}}%
}

\newcommand{\scaffoldgnnlink}{%
\href{https://github.com/siddhartha047/Scaffold-GNN}%
{\nolinkurl{github.com/siddhartha047/Scaffold-GNN}}%
}

\title{\sgt: Support Graph Theory Based Sparsification for Graph Neural Networks}

\author{
\parbox{0.96\textwidth}{
\raggedright
{\bfseries
Siddhartha Shankar Das\textsuperscript{1},
Sai Karthik Navuluru\textsuperscript{2},
S. M. Ferdous\textsuperscript{3},
Ryan A. Rossi\textsuperscript{4},\\[-1pt]
Baris Coskunuzer\textsuperscript{2},
Lakshman Tamil\textsuperscript{2},
Edoardo Serra\textsuperscript{5},
Alex Pothen\textsuperscript{6},
Robert Rallo\textsuperscript{1},
Mahantesh M. Halappanavar\textsuperscript{1}
}\\[4pt]
{\normalfont\footnotesize
\makebox[0.50\linewidth][l]{\textsuperscript{1}Pacific Northwest National Laboratory}%
\makebox[0.48\linewidth][l]{\textsuperscript{2}University of Texas at Dallas}\\[-1pt]
\makebox[0.50\linewidth][l]{\textsuperscript{3}University of North Carolina at Charlotte}%
\makebox[0.48\linewidth][l]{\textsuperscript{4}Adobe}\\[-1pt]
\makebox[0.50\linewidth][l]{\textsuperscript{5}Boise State University}%
\makebox[0.48\linewidth][l]{\textsuperscript{6}Purdue University}
}
}
}
\begin{document}

\maketitle
\pagestyle{fancy}
\fancyhf{}

\fancyhead[L]{\small Scaffold}
\renewcommand{\headrulewidth}{0.4pt}

\fancyfoot{}

\thispagestyle{fancy}

\begin{abstract}
Graph neural networks (GNNs) rely on message passing over graph edges, making their computational and memory costs strongly dependent on graph density. Graph sparsification offers a natural way to reduce these costs, but removing edges indiscriminately can distort important communication structure and degrade predictive performance.
We introduce \sgt, a topology-based, unsupervised graph sparsification framework derived from support graph theory preconditioners. \sgt explicitly controls two complementary structural quantities: \emph{dilation}, which measures the length of rerouting paths induced by removed edges, and \emph{congestion}, which measures how strongly these rerouted paths concentrate on the retained support.
By jointly controlling dilation and congestion, \sgt preserves short communication paths while avoiding structural bottlenecks. 
To our knowledge, \sgt is the first scalable GNN sparsification framework to use a joint supporting-path dilation-congestion criterion.
Across $19$ homophilic and heterophilic benchmarks spanning small to large graphs, \sgt achieves the best aggregate rank among the evaluated sparsification and related methods. 
Using only $10\%$--$50\%$ of the original edges per sparse support, \sgt recovers or closely approaches full-graph GNN performance while using less than half the memory of full-graph training and reducing end-to-end training time, including sparsification overhead.
We provide an open-source software package at \scaffoldlink.
\end{abstract}

\section{Introduction}
\label{sec:introduction}

Graph Neural Networks (GNNs)~\citep{Zhou2020Review,wu2022graph} have been remarkably successful across a wide range of graph-learning applications.
GNNs operate through two iterative steps, \emph{aggregation} and \emph{update}, where each node first aggregates information from its neighborhood and then updates its embedding using the aggregated information~\citep{Hamilton2017GraphSAGE}. 
Graphs used for GNN training are typically stored in sparse formats such as coordinate (COO), compressed sparse row (CSR), or compressed sparse column (CSC). 
COO stores edges explicitly as source-destination index pairs and is commonly used for edge-wise \texttt{Gather-scatter} operation, where source-node features are gathered along edges and aggregated at destination nodes. CSR/CSC instead store the sparse adjacency matrix in compressed row/column form and support efficient sparse-matrix-dense-matrix multiplication (\texttt{SpMM}).
These sparse aggregation operations often dominate GNN execution, with \texttt{SpMM} alone accounting for approximately $\mathbf{70\%}$ of the total computation cost~\citep{Liu2023DSpar}. Since their cost scales with the number of processed edges, GNN training can become expensive on larger and denser graphs~\citep{Chen2018FastGCN,Zeng2020GraphSAINT}.

A common approach to scaling GNNs is \textit{sampling} \citep{Hamilton2017GraphSAGE,Chiang2019ClusterGCN}, which reduces per-iteration computation by typically controlling the size of the sampled computation graph rather than explicitly preserving its global structure. Graph \textit{sparsification} offers a complementary structure-aware sampling perspective by selecting a subset of edges while preserving properties of the original graph so that the resulting subgraph remains an effective support for message passing~\citep{Spielman2011Spectral,Loukas2019Reduction,Hashemi2024Survey}. However, not all sparsification objectives align with the communication requirements of GNNs. Equally sized sparse graphs can induce very different communication structures, forcing originally adjacent nodes to communicate through long detours or rerouting many communication paths on a small number of edges or hub nodes. This is particularly important in the \emph{unsupervised} setting, where labels and downstream task information are unavailable for identifying task-specific edge importance. The overarching goal is therefore to construct sparse graphs that preserve structural properties that support effective communication while satisfying a fixed edge budget.

Among structural properties, \emph{low dilation} and \emph{low congestion} are particularly relevant to sparsification. In support graph theory, dilation measures how far communication is rerouted when an edge is removed, while congestion measures the crowdedness or density of rerouted paths on the retained support. These concepts originate from combinatorial preconditioners to solve systems of linear equations, and were later formalized through support theory and congestion-dilation analysis~\citep{Vaidya1991,Boman2003,Bern2006Support,Maggs2005SupportTree}. Together, they characterize how well a sparse graph supports the original graph.
We bring this perspective to graph representation learning through \sgt, a practical and scalable sparsification framework that explicitly controls dilation and congestion under a prescribed edge budget. If removing an edge forces its endpoints to communicate through a long or highly congested supporting path, that edge is poorly represented by the current sparse graph and receives higher priority for retention. \sgt therefore constructs sparse supports in which removed connections remain represented by short paths while avoiding heavily congested communication bottlenecks.



\begin{wrapfigure}{r}{0.48\textwidth}
    \centering
    \includegraphics[width=\linewidth]{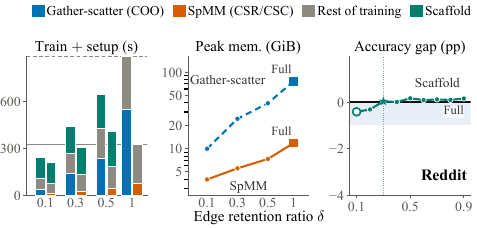}
    \captionsetup{font=small}
    \caption{Runtime, memory, and accuracy trade-off.}
    \label{fig:intro_cost}
\end{wrapfigure}

We develop five realizations of this principle for different scalability regimes:
\texttt{greedy} fully recomputes support scores, \texttt{heap} uses lazy updates,
\texttt{batch} scores sampled candidate batches, \texttt{fast} uses static
tree-prefix scoring for large graphs, and \texttt{sample} converts structural
scores into a distribution over sparse graph views. Together, these variants
support both fixed sparsification and stochastic GNN sampling.
These sparse supports also provide practical efficiency benefits. 
As illustrated in Figure~\ref{fig:intro_cost} on \textsc{Reddit}, \sgt approaches full-graph performance at low edge-retention ratios (from $10-100\%$; X-axis) while substantially reducing end-to-end training time and peak memory in both \texttt{COO}- and \texttt{SpMM}-based implementations. Our contributions are as follows:
\begin{enumerate}[label=(\roman*), leftmargin=*, labelindent=0pt,
                  labelsep=0.5em, itemsep=0pt, topsep=0pt]

\item We introduce \sgt, a topology-based unsupervised and theoretically rigorous sparsification framework that translates support-graph dilation and congestion from combinatorial preconditioners to graph representation learning (\S\ref{sec:preliminaries}).

\item We develop a scalable family of \sgt algorithms with different quality-cost trade-offs, supporting both fixed sparse graphs and stochastic sparse views for efficient GNN training and inference
(\S\ref{sec:method}).

\item We characterize how dilation and edge congestion bound the structural
and communication distortion introduced by sparsification, and empirically
study their implications for the GNN message passing framework
(\S\ref{subsec:scaffold-theory}).

\item Across $19$ datasets and $20$ related baselines, \sgt achieves strong
predictive performance while reducing memory and end-to-end training cost.
With each sparse support retaining only $30\%$ and $50\%$ of the original
edges, \sgt reaches within one percentage point of full-graph performance on
$12/19$ and $16/19$ datasets, respectively
(\S\ref{sec:experiments}).

\end{enumerate}



\section{Preliminaries}
\label{sec:preliminaries}

\paragraph{Sparse support and supporting paths.}
Let $\gG=(\gV,\gE,\vw)$ be an undirected graph with $n=|\gV|$ nodes and
$m=|\gE|$ edges. For an edge-retention ratio $\delta\in(0,1]$, let
$q=\lceil\delta m\rceil$ denote the target edge budget. We seek a spanning
subgraph $\gH=(\gV,\gE_{\gH},\vw_{\gH})\subseteq\gG$ with
$|\gE_{\gH}|=q$, and denote the omitted edges by
$\gI=\gE\setminus\gE_{\gH}$.
Following the support-graph view~\citep{Vaidya1991,Boman2003,Bern2006Support},
each omitted edge $e=(u,v)\in\gI$ is represented by a supporting path $P_e$
in $\gH$ connecting its endpoints. In \sgt, $P_e$ is chosen as a shortest
path in the current support. Figure~\ref{fig:sgnotations} illustrates this
view and the structural quantities used below.

\begin{figure}[!t]
    \centering
    \includegraphics[width=\linewidth]
    {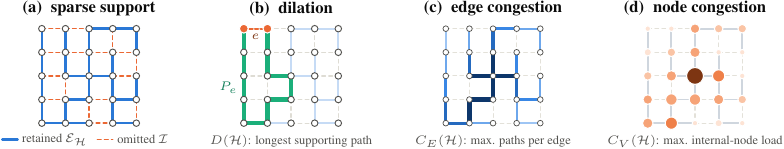}
    \caption{Support-graph features: (a) sparse support, (b) dilation,
    (c) edge congestion, and (d) node congestion. Darker edges and larger
    nodes indicate higher path loads; node loads exclude endpoints.}
    \label{fig:sgnotations}
\end{figure}

\paragraph{Dilation and congestion.}
Classical support theory characterizes support quality using dilation and
edge congestion~\citep{Vaidya1991,Boman2003,Bern2006Support}; we additionally
consider node congestion for GNN message passing.
For the unweighted graphs studied in this work, we define
\[
D(\gH)=\max_{e\in\gI}|P_e|,
\qquad
C_E(\gH)=
\max_{a\in\gE_{\gH}}
|\{e\in\gI:a\in P_e\}|,~\text{and}
\]
\[
C_V(\gH)=
\max_{x\in\gV}
|\{e\in\gI:x\in\operatorname{int}(P_e)\}|.
\]
Thus, dilation measures how far an omitted connection is rerouted, while
edge and node congestion measure how strongly rerouted paths concentrate on
the retained support. Weighted extensions, further support-theoretic
background, and controlled
examples illustrating these quantities are provided in Appendix~\ref{app:preliminaries}.


\paragraph{Sparsification objective.}
Given an edge budget $q$, we seek a spanning support, $\gH^\star$, in which
omitted edges remain represented by short and lightly congested paths:
\begin{equation}
\begin{aligned}
\Phi(\gH)
&=
\underbrace{(1+D(\gH))^{\alpha}}_{\text{dilation}}
\underbrace{(1+C_E(\gH))^{\beta_E}}_{\text{edge congestion}}
\underbrace{(1+C_V(\gH))^{\beta_V}}_{\text{node congestion}},\\[-2pt]
\gH^\star
&\in
\arg\min_{\substack{
\gH\subseteq\gG,\;|\gE_{\gH}|=q\\
\gH\ {\rm spans\ each\ component}
}}
\Phi(\gH).
\end{aligned}
\label{eq:scaffold-global-objective}
\end{equation}
Here, $\alpha,\beta_E,\beta_V\geq0$ control the relative contribution of
the three criteria. Classical support theory uses dilation and edge
congestion, while $C_V$ is an additional GNN-oriented extension;
setting $\beta_V=0$ recovers the core dilation--edge-congestion objective.

\section{Proposed Method for Sparsification: \sgt}
\label{sec:method}

Equation~\ref{eq:scaffold-global-objective} defines the desired sparse support globally, but direct optimization is NP-hard, even in the dilation-only case (Appendix~\ref{app:nphardproof}). \sgt therefore replaces the global objective with a candidate-wise structural surrogate. We first define this score, then present scalable realizations, describe their use within GNNs, and summarize their structural properties and computational cost.

\subsection{From the Global Objective to a Greedy Surrogate}
\label{subsec:scaffold-greedy}

Let $\gH_t$ be the current support and
$\gI_t=\gE\setminus\gE_{\gH_t}$ the remaining candidates. An exact greedy step
would select
\begin{equation}
e_t^{\mathrm{exact}}
=
\arg\max_{e\in\gI_t}\Delta_t(e),
\qquad
\Delta_t(e)
=
\Phi(\gH_t)-\Phi(\gH_t\cup\{e\}).
\label{eq:scaffold-exact-marginal}
\end{equation}
However, evaluating every marginal gain requires repeatedly recomputing
supporting paths and congestion after hypothetical insertions, making exact
greedy selection expensive.

Therefore, \sgt instead scores each candidate from its current supporting path.
Starting from a spanning-forest backbone $\gH_0=\gF$, with one spanning tree
per connected component, for $e=(u,v)\in\gI_t$ let
$P_e=\operatorname{ShortestPath}_{\gH_t}(u,v)$.
For the unweighted graphs studied in this work, its path dilation is
\begin{equation}
\pathDil_t(e)=|P_e|.
\label{eq:scaffold-dilation}
\end{equation}

For a path $P_f$ supporting $f=(u,v)$, let
$\operatorname{int}(P_f)=V(P_f)\setminus\{u,v\}$ denote its internal vertices.
For the current candidate paths, the edge- and node-congestion counts are
\[
\eCon_t(a)=|\{f\in\gI_t:a\in P_f\}|,
\qquad
\vCon_t(x)=|\{f\in\gI_t:x\in\operatorname{int}(P_f)\}|.
\]

We aggregate them along $P_e$ using power means ($\pathE,\pathN\ge1$): \\
\noindent
\begin{minipage}[H]{0.49\linewidth}
\small
\begin{equation}
\eConPath_t(e)
=
\left(
\frac{\sum_{a\in P_e}\eCon_t(a)^{\pathE}}
{|P_e|}
\right)^{1/\pathE};
\label{eq:scaffold-path-edge-congestion}
\end{equation}
\end{minipage}
\hfill
\begin{minipage}[H]{0.49\linewidth}
\footnotesize
\begin{equation}
\vConPath_t(e)
=
\left(
\frac{\sum_{x\in\operatorname{int}(P_e)}
\vCon_t(x)^{\pathN}}
{|\operatorname{int}(P_e)|}
\right)^{1/\pathN}.
\label{eq:scaffold-path-node-congestion}
\end{equation}
\end{minipage}

We set $\vConPath_t(e)=1$, when $\operatorname{int}(P_e)=\emptyset$; larger $\pathE$ or $\pathN$ values emphasize stronger bottlenecks.

Let $\pathDilMax{t}$, $\eConPathMax{t}$, and $\vConPathMax{t}$ denote the corresponding maxima over $\gI_t$. The normalized \sgt score is
\begin{equation}
\score_t(e)
=
\underbrace{
\left(
\frac{\pathDil_t(e)}
     {\pathDilMax{t}}
\right)^{\alpha}
}_{\text{dilation}}
\underbrace{
\left(
\frac{\eConPath_t(e)}
     {\eConPathMax{t}}
\right)^{\beta_E}
}_{\text{edge congestion}}
\underbrace{
\left(
\frac{\vConPath_t(e)}
     {\vConPathMax{t}}
\right)^{\beta_V}
}_{\text{node congestion}} .
\label{eq:scaffold-score}
\end{equation}
A zero exponent disables the corresponding component. For example,  $(1,1,0)$ gives the dilation--edge-congestion criterion, while $\beta_V>0$ additionally includes node congestion. The resulting rule replaces the exact global marginal gain with the local support score:
\begin{equation}
\underbrace{
e_t^{\mathrm{exact}}
=
\arg\max_{e\in\gI_t}\Delta_t(e)
}_{\text{exact marginal}}
\qquad\Longrightarrow\qquad
\underbrace{
e_t^\star
=
\arg\max_{e\in\gI_t}\score_t(e)
}_{\text{\sgt surrogate}},
\label{eq:scaffold-surrogate-selection}
\end{equation}
where $\score_t(e)$ measures the current structural deficit of candidate $e$, whereas $\Delta_t(e)$ also accounts for how inserting $e$ reroutes other candidates. \sgt partially captures the cross-candidate effects by refreshing paths and scores as $\gH_t$ grows. Even dilation-only reduction is non-submodular, so the standard submodular-greedy guarantee does
not apply; further analysis is given in Appendix~\ref{app:scaffold-nonsubmodular}.

\subsection{Family of \sgt Algorithms}
\label{subsec:scaffold-family}

All \sgt variants share three components: a support backbone $\gF$, the
support score in Eq.~\ref{eq:scaffold-score}, and a candidate-selection
policy. For feasible budgets, the spanning-forest backbone ensures that every
omitted edge has a supporting path, while the selection policy determines how
aggressively paths and scores are recomputed. In general, $\gF$ may be
deterministic, randomized, weight-aware, or low-stretch
\citep{Abraham2008LowStretch,Abraham2012Nearly}. For weighted graphs, \sgt
uses the supplied edge weights through the path metric.
A weighted formulation of \sgt using weighted path lengths and routing loads is given in Appendix~A.3. Since all benchmark
graphs in this work are unweighted, we use a deterministic spanning forest
(\texttt{SF}) together with randomized forests and fast
low-stretch approximations. 
These algorithms are novel support-theoretic sparsifiers.

\paragraph{Greedy reference.}
\sgte is the full-recomputation realization of
Eq.~\ref{eq:scaffold-surrogate-selection}. Starting from
$\gH_0=\gF$, it recomputes supporting paths, congestion statistics, and
candidate scores after every insertion, retains the highest-scoring edge,
and repeats until $|\gE_{\gH}|=q$. Algorithm~\ref{alg:scaffold-greedy}
gives the reference construction, while Figure~\ref{fig:algorithm}
illustrates one run.

\makeatletter
\@ifundefined{sgpairalg}{%
  \newsavebox{\sgpairalg}%
  \newsavebox{\sgpairfig}%
  \newlength{\sgpairht}%
  \newlength{\sgpairdp}%
}{}
\makeatother

\begin{figure}[!htbp]
  \centering

  \sbox{\sgpairalg}{%
    \begin{minipage}{0.545\linewidth}
      \setlength{\intextsep}{0pt}%
      \begin{algorithm}[H]
        \fontsize{8}{9.5}\selectfont
        \caption{\textsc{\sgte}$(\gG,\delta)$}
        \label{alg:scaffold-greedy}
        \algsetup{indent=1em,linenosize=\fontsize{7}{8}\selectfont}
        \begin{algorithmic}[1]
          \REQUIRE Graph $\gG=(\gV,\gE,\vw)$, retention ratio $\delta$
          \ENSURE Sparse support $\gH$
          \STATE $q\leftarrow\lceil\delta m\rceil$; choose spanning forest $\gF$
          \STATE $\gH_0\leftarrow\gF$;
            $\gI_0\leftarrow\gE\setminus\gE_{\gF}$; $t\leftarrow0$
          \WHILE{$|\gE_{\gH_t}|<q$}
            \STATE Compute $P_e$, $\pathDil_t(e)$, $\eConPath_t(e)$,
              $\vConPath_t(e)$ for $e\in\gI_t$
            \STATE Compute $\score_t(e)$ for $e\in\gI_t$
            \STATE $e_t^\star\leftarrow\arg\max_{e\in\gI_t}\score_t(e)$
            \STATE $\gH_{t+1}\leftarrow\gH_t\cup\{e_t^\star\}$;\\
              $\gI_{t+1}\leftarrow\gI_t\setminus\{e_t^\star\}$;
              $t\leftarrow t+1$
          \ENDWHILE
          \STATE \textbf{return} $\gH_t$
        \end{algorithmic}
      \end{algorithm}
    \end{minipage}}%

  \sbox{\sgpairfig}{%
    \includegraphics[width=0.417\linewidth]
      {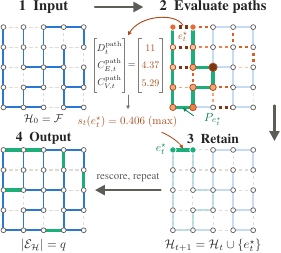}}%

  \setlength{\sgpairht}{\ht\sgpairalg}%
  \addtolength{\sgpairht}{\dp\sgpairalg}%
  \setlength{\sgpairdp}{\ht\sgpairfig}%
  \addtolength{\sgpairdp}{\dp\sgpairfig}%
  \ifdim\sgpairdp>\sgpairht
    \setlength{\sgpairht}{\sgpairdp}%
  \fi

  \begin{minipage}[t]{0.545\linewidth}
    \vspace{0pt}%
    \centering
    {\small\textbf{(a) Greedy reference}\par}
    \vspace{4pt}
    \parbox[c][\sgpairht][c]{\linewidth}{%
      \centering\usebox{\sgpairalg}}%
  \end{minipage}\hfill
  \begin{minipage}[t]{0.417\linewidth}
    \vspace{0pt}%
    \centering
    {\small\textbf{(b) Example construction}\par}
    \vspace{4pt}
    \parbox[c][\sgpairht][c]{\linewidth}{%
      \centering\usebox{\sgpairfig}}%
  \end{minipage}

  \caption{Starting from a spanning backbone, \sgt repeatedly scores the omitted edges, retains the highest-scoring candidate, and refreshes the support until the target budget is reached.}
  \label{fig:algorithm}
\end{figure}

\paragraph{Scalable variants.}
Full recomputation makes \texttt{Greedy} expensive, so we develop progressively cheaper realizations of the same scoring principle.
\texttt{Heap} caches paths and refreshes only affected candidates; \texttt{Batch} scores sampled candidate batches and inserts their top-$r$ edges between refreshes; \texttt{Fast} computes all scores once on the initial forest using Lowest Common Ancestor (\texttt{LCA}) based tree-prefix operations in $\mathcal{O}((m+n)\log n)$ time; and \texttt{Sample} converts scores from multiple backbones into sampling weights, allowing subsequent exactly
budgeted sparse views in $\mathcal{O}(m)$ time.
Table~\ref{tab:scaffold-variants} reports complexity, runtimes, and parameters that can be tuned for cost--quality trade-offs. Appendix~\ref{App:method} provides detailed algorithms, backbone choices, and runtime analysis.

\begin{table}[!htbp]
\centering
\caption{\sgt{} variants and measured runtime ($\delta=0.2$, 8 CPU workers).}
\label{tab:scaffold-variants}
\scriptsize
\setlength{\tabcolsep}{3pt}
\renewcommand{\arraystretch}{1.08}

\resizebox{\linewidth}{!}{%
\begin{tabular}{@{}llllrl@{}}
\toprule
\textbf{Variant}
& \textbf{Scoring / selection}
& \textbf{Complexity}
& \textbf{Graph $(n,m)$}
& \textbf{Time (s)}
& \textbf{Use} \\
\midrule

Greedy
& dynamic, all; best one
& $\mathcal{O}(TnC_{\mathrm{sp}})$
& $(400,\,3.1\mathrm{K})$
& 0.78
& reference \\

Heap
& dynamic, affected; heap max.
& $\mathcal{O}((n+T(\kappa+c))C_{\mathrm{sp}}+m\log m)$
& $(4\mathrm{K},\,39.9\mathrm{K})$
& 25.16
& small \\

Batch
& dynamic, $b$ sampled; top-$r$
& $\mathcal{O}(\lceil T/r\rceil bC_{\mathrm{sp}})$
& $(10\mathrm{K},\,250\mathrm{K})$
& 31.01
& medium/large \\

Fast
& static on $\gF$; top-$T$
& $\mathcal{O}((m+n)\log n)$
& $(200\mathrm{K},\,2.4\mathrm{M})$
& 2.30
& large \\

Sample
& static, $R$ forests; budgeted draw
& $\mathcal{O}(R(m+n)\log n)$ pre.; $\mathcal{O}(m)$ draw
& $(200\mathrm{K},\,2.4\mathrm{M})$
& 9.02
& large \\

\bottomrule
\end{tabular}%
}
\vspace{1pt}
\parbox{\linewidth}{\scriptsize
$T=q-|\gE_{\gF}|$; $C_{\mathrm{sp}}$ is one BFS/Dijkstra search.
Times are medians of three complete calls including forest construction and
selection. Sample uses $R=8$; Batch uses $b=512$, $r=64$.
Graph sizes differ across rows; see Appendix~\ref{app:scaffold-runtime}.}
\end{table}

\subsection{\sgt for GNN Training and Inference}
\label{subsec:scaffold-gnn}
Since \sgt depends on graph topology, support graphs are generated independently of labels and GNN parameters, and directly replace $\gG$ during message passing. 
In \sgt-1, only one $q$-edge support is constructed once and used throughout training and inference. In \sgt-$K$, a new $q$-edge support is generated every $\rho$ epochs, while only the $K$ strongest validation supports are retained in an online support bank. At inference, these supports
are combined as $\gH_{\mathrm{union}}=\bigcup_{k=1}^{K}\gH^{(k)}$ and evaluated in a single forward pass. Since $|\gE_{\gH_{\mathrm{union}}}|$ may exceed $q$, its realized inference edge count is reported separately. Appendix~\ref{app:gnntraining} provides further details.

\begin{figure}[!htbp]
\centering
\includegraphics[width=\linewidth]{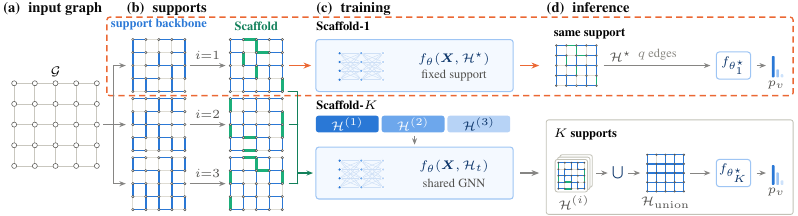}
\caption{\sgt-1 uses one fixed $q$-edge support for training and inference, whereas \sgt-$K$ periodically refreshes $q$-edge supports, retains the validation-best $K$, and performs one inference pass on their union.}
\label{fig:pipeline}
\end{figure}

\subsection{Structural Properties and Complexity}
\label{subsec:scaffold-theory}
Our theoretical analysis covers sparse-support properties rather than nonlinear GNN dynamics.
In the dilation-only setting, selecting large-detour edges bounds the maximum dilation remaining after
sparsification.
For an unweighted graph $\gG$ and a support $\gH\subseteq\gG$ preserving its connected components, dilation and edge congestion yield
$
\mL_{\gH}
\preceq
\mL_{\gG}
\preceq
\bigl(1+D(\gH)C_E(\gH)\bigr)\mL_{\gH}
$
\citep{Boman2003,Bern2006Support}.
This bounds effective-resistance distortion for connected graphs, while separate path arguments bound graph distances and edge-cut sizes. Node congestion gives an analogous neighborhood-concentration bound but
is not required by the Laplacian bound.
Complete statements, assumptions and proofs are provided in Appendix~\ref{app:scaffold-theory}.
A general approximation guarantee for the full dilation--congestion objective remains open, and these structural bounds do not imply monotonic improvement in downstream GNN accuracy.

\paragraph{Runtime and memory.}
For an $L$-layer GNN with hidden dimension $d$, message passing on a $q$-edge
\sgt support costs $\mathcal{O}(Lqd)$ per epoch, compared with
$\mathcal{O}(Lmd)$ on the full graph, while graph-dependent memory decreases
from $\mathcal{O}(m)$ to $\mathcal{O}(q)$. If one support costs
$C_{\sgt}$ to generate, refreshing every $\rho$ epochs gives amortized
generation cost $C_{\sgt}/\rho$. Support generation can additionally be
performed asynchronously with GNN training. Detailed construction costs and
scaling measurements are given in Appendix~\ref{app:scaffold-runtime}.

\section{Related Work}
\label{sec:related-work}
Support-graph theory~\citep{Vaidya1991,Boman2003,Bern2006Support} was developed for combinatorial preconditioners, with support-tree preconditioners and low-stretch spanning trees constructing sparse backbones with short replacement paths \citep{Maggs2005SupportTree,elkin2005lower,Abraham2008LowStretch, Abraham2012Nearly}. 
Graph spanners~\citep{Althofer1990Spanners} are closely related to the dilation component of \sgt, whereas \sgt follows the broader support-theoretic view by jointly controlling dilation and congestion under a fixed edge budget. 
To the best of our knowledge, \sgt is the first GNN sparsifier to explicitly use this joint criterion to construct the graph used directly for graph learning.
Broadly, sparsification approaches in GNNs can be grouped into: \emph{topology-based} sparsifiers using spectral,
degree, similarity, or traversal criteria~\citep{Spielman2011Spectral,
spielman2008graph,Hamann2016Structure,Voudigari2016Rank,Xu2007SCAN,
Satuluri2011Local,Leskovec2006Sampling},
\emph{semantic-based} methods that learn sparse structures from downstream
signals~\citep{Chen2021UnifiedLottery,Zhang2024GraphLotteryAutomated,
Zhang2025MoG}, and \emph{sampling-based} methods that reduce the computation
graph processed during training~\citep{Hamilton2017GraphSAGE,Chen2018FastGCN,
Zeng2020GraphSAINT,Chiang2019ClusterGCN}. \sgt is a topology-based approach that scores how well omitted edges are represented by paths in the retained support. A broader discussion of related work is provided in Appendix~\ref{app:extended-related}.

\section{Experimental Results}
\label{sec:experiments}

\textbf{Datasets and evaluation.}
We evaluate \sgt on 19 node-classification benchmarks: eight homophilic, six heterophilic, and five large-scale graphs. All graphs are made undirected and self-loop-free before sparsification. We report Accuracy for single-label classification and ROC-AUC for \textsc{Minesweeper}, \textsc{Questions}, and \textsc{ogbn-proteins}. Dataset statistics, splits, metrics, and training configurations are given in Appendix~\ref{app:experimental}. 

\textbf{Baselines and protocol.}
We compare against topology- and semantic-based methods together with Full-Graph, Spanning Forest (\texttt{SF}) and No-Graph references. 
We use the per-dataset configurations from TunedGNN~\citep{luo2024classic} and sparsifiers share the same splits, architecture, optimization, and training budget. 
Smaller graphs use $\delta\in\{0.3,0.5,0.7\}$ and large graphs use $\delta\in\{0.1,0.3,0.5\}$; 
below the spanning-forest floor ($q<n-\#\mathrm{components}(\gG)$), construction stops at the target budget $q$.
Unless specified, \sgt uses \texttt{batch}, \texttt{fast}, and \texttt{sample} at $(\alpha,\beta_E,\beta_V)=(1,1,0)$.
\sgt-1 uses one fixed $q$-edge support with deterministic spanning forest (\texttt{SF}) throughout training and inference, whereas \sgt-$K$ uses random spanning forests (\texttt{RandSF}) as a starting backbone and periodically refreshes $q$-edge supports while asynchronously computing support to ensure the GPU is not blocked for sparse computation. Full settings are detailed in Appendix~\ref{app:competing-method-settings}\footnote{\sgt with GNN:~\scaffoldgnnlink}.


\subsection{Predictive Performance under Sparsification}
\label{subsec:exp-quality}

\textbf{Comparison with existing sparsifiers.}
We summarize key performance in Table~\ref{tab:1andkshot-main-results-aggregated}. 
We compare sparsifiers at dataset-specific targets $\delta_d$ chosen at or above the spanning-forest floor, so the target budget is sufficient to preserve all connected components, with each method retaining exactly $q=\lceil\delta_d m\rceil$ edges.
\sgt-1 uses one fixed $q$-edge support throughout training and inference, whereas \sgt-$K$ refreshes $q$-edge supports during training and performs inference on the union of the $K$ best validation supports. \sgt-1 achieves the highest geometric mean among the matched-budget methods, while \sgt-$K$ can provide gains through support refresh.
The main table reports the validation-selected best of \texttt{Batch}, \texttt{Fast}, and \texttt{Sample} for each dataset.
Complete variant-wise results at other retention ratios are provided in Appendix~\ref{app:complete-numerical-results} (Tables~\ref{tab:appendix-full-results-homophilic}, \ref{tab:appendix-full-results-heterophilic}, and \ref{tab:appendix-full-results-large}).
Figure~\ref{fig:scaffold-cd-all-methods-pooled} summarizes the average-rank comparison~\citep{benavoli2016should} over $57$ dataset-budget evaluations, including \sgt at its target-budget and best-ratio settings, the full graph and other baselines.

\begin{figure}[!htbp]
  \centering
  \parbox{0.98\linewidth}{\scriptsize
Wilcoxon--Holm comparison~\citep{demsar2006statistical,benavoli2016should} over 19 datasets at three retention ratios (57 tasks). Lower rank is better; bars join groups with no significant pairwise differences at $\alpha=0.05$. \textbf{Best} uses \sgt's best ratio; \textbf{Target} uses the prescribed ratio.}
  \input{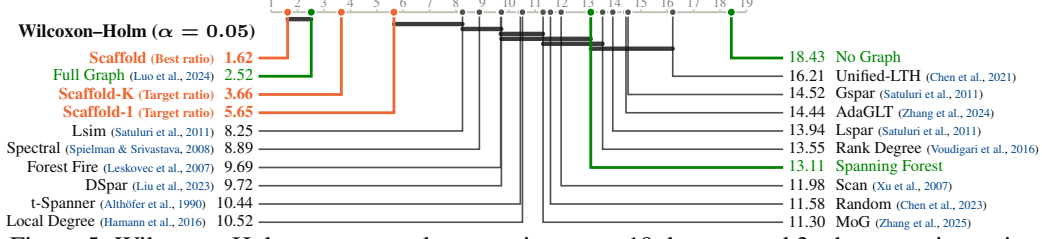}
  \caption{Wilcoxon--Holm average-rank comparison over 19 datasets and 3 edge-retention ratios.}
  \label{fig:scaffold-cd-all-methods-pooled}
\end{figure}

\begin{table*}[t]
\centering
\caption{Accuracy/ROC-AUC at dataset-specific targets $\delta_d\ge0.3$ above the spanning-forest floor. \sgt-1 is ranked at the matched budget; \sgt-$K$ is unranked.}
\label{tab:1andkshot-main-results-aggregated}
\fontsize{5.7pt}{5.9pt}\selectfont
\setlength{\tabcolsep}{0.48pt}
\renewcommand{\arraystretch}{0.90}
\begin{tabular*}{\textwidth}{@{}c@{\hspace{0.4pt}}l@{\extracolsep{\fill}}*{10}{c}@{\hspace{1.0pt}}c@{}}
\noalign{\vskip 0.7ex}
\multicolumn{13}{@{}c@{}}{\color[gray]{0.30}\fontsize{6.6pt}{7.4pt}\selectfont\bfseries\vphantom{Ay}(a) Dataset group I:\enspace Homophilic and Heterophilic} \\[0.55ex]
\toprule
& \textbf{Method} & \textbf{Cora} & \textbf{CiteSeer} & \textbf{Am. Photo} & \textbf{WikiCS} & \textbf{Am. Comp.} & \textbf{Coauth. CS} & \textbf{PubMed} & \textbf{Coauth. Phys.} & \textbf{Chameleon} & \textbf{Squirrel} & \textbf{GM} \\
& {\color[gray]{0.28}\fontsize{5.0pt}{5.2pt}\selectfont Target $\delta_d$} & {\color[gray]{0.18}\fontsize{5.2pt}{5.4pt}\selectfont $0.7$} & {\color[gray]{0.18}\fontsize{5.2pt}{5.4pt}\selectfont $0.7$} & {\color[gray]{0.18}\fontsize{5.2pt}{5.4pt}\selectfont $0.3$} & {\color[gray]{0.18}\fontsize{5.2pt}{5.4pt}\selectfont $0.3$} & {\color[gray]{0.18}\fontsize{5.2pt}{5.4pt}\selectfont $0.3$} & {\color[gray]{0.18}\fontsize{5.2pt}{5.4pt}\selectfont $0.5$} & {\color[gray]{0.18}\fontsize{5.2pt}{5.4pt}\selectfont $0.5$} & {\color[gray]{0.18}\fontsize{5.2pt}{5.4pt}\selectfont $0.3$} & {\color[gray]{0.18}\fontsize{5.2pt}{5.4pt}\selectfont $0.3$} & {\color[gray]{0.18}\fontsize{5.2pt}{5.4pt}\selectfont $0.3$} & \\
\bottomrule
\end{tabular*}
\vspace{1.2pt}
\begin{tabular*}{\textwidth}{@{}c@{\hspace{0.4pt}}l@{\extracolsep{\fill}}*{9}{c}@{\hspace{1.0pt}}c@{}}
\noalign{\vskip 0.7ex}
\multicolumn{12}{@{}c@{}}{\color[gray]{0.30}\fontsize{6.6pt}{7.4pt}\selectfont\bfseries\vphantom{Ay}(b) Dataset group II:\enspace Heterophilic and Large-scale} \\[0.55ex]
\toprule
& \textbf{Method} & \textbf{Minesweeper} & \textbf{Roman} & \textbf{Am. Ratings} & \textbf{Questions} & \textbf{Proteins} & \textbf{Arxiv} & \textbf{Reddit} & \textbf{Pokec} & \textbf{Products} & \textbf{GM} \\
& {\color[gray]{0.28}\fontsize{5.0pt}{5.2pt}\selectfont Target $\delta_d$} & {\color[gray]{0.18}\fontsize{5.2pt}{5.4pt}\selectfont $0.7$} & {\color[gray]{0.18}\fontsize{5.2pt}{5.4pt}\selectfont $0.7$} & {\color[gray]{0.18}\fontsize{5.2pt}{5.4pt}\selectfont $0.7$} & {\color[gray]{0.18}\fontsize{5.2pt}{5.4pt}\selectfont $0.7$} & {\color[gray]{0.18}\fontsize{5.2pt}{5.4pt}\selectfont $0.3$} & {\color[gray]{0.18}\fontsize{5.2pt}{5.4pt}\selectfont $0.3$} & {\color[gray]{0.18}\fontsize{5.2pt}{5.4pt}\selectfont $0.3$} & {\color[gray]{0.18}\fontsize{5.2pt}{5.4pt}\selectfont $0.3$} & {\color[gray]{0.18}\fontsize{5.2pt}{5.4pt}\selectfont $0.3$} & \\
& {\color[gray]{0.28}\fontsize{5.0pt}{5.2pt}\selectfont $|\gV|/|\gE|$} & {\color[gray]{0.28}\fontsize{5.0pt}{5.2pt}\selectfont 10.00k/39.40k} & {\color[gray]{0.28}\fontsize{5.0pt}{5.2pt}\selectfont 22.66k/32.93k} & {\color[gray]{0.28}\fontsize{5.0pt}{5.2pt}\selectfont 24.49k/93.05k} & {\color[gray]{0.28}\fontsize{5.0pt}{5.2pt}\selectfont 48.92k/153.54k} & {\color[gray]{0.28}\fontsize{5.0pt}{5.2pt}\selectfont 132.53k/39.56m} & {\color[gray]{0.28}\fontsize{5.0pt}{5.2pt}\selectfont 169.34k/1.17m} & {\color[gray]{0.28}\fontsize{5.0pt}{5.2pt}\selectfont 232.97k/57.31m} & {\color[gray]{0.28}\fontsize{5.0pt}{5.2pt}\selectfont 1.63m/22.30m} & {\color[gray]{0.28}\fontsize{5.0pt}{5.2pt}\selectfont 2.45m/61.86m} & \\
\bottomrule
\end{tabular*}
\vspace{1pt}
\begin{minipage}{0.99\textwidth}
\scriptsize
Green/blue/amber mark the top three matched-budget methods; references are unranked. Values are mean $\pm$ std.\ (\%); \textbf{GM}: geometric mean. For each dataset and \sgt setting, validation selects the best of \texttt{Batch}/\texttt{Fast}/\texttt{Sample}, marked $^{b/f/s}$. 
\sgt-$K$ arrows compare against \sgt-1.
$^{\dagger}$ROC-AUC; $^{\ddagger}$\textsc{slst} on \textsc{chameleon}/\textsc{squirrel} (\textsc{sf} otherwise). Full results: Appendix~\ref{app:complete-numerical-results}.
\end{minipage}
\end{table*}




\textbf{Performance across retention ratios.}
\label{para:scafvsfullgcn}
Figure~\ref{fig:sgtvstuned} compares \sgt-1 and \sgt-$K$ with the corresponding
full-graph Tuned-GCN as the retention ratio $\delta$ increases, mapped along $X$-axis.
Both variants approach full-graph performance as more edges are retained.
At $\delta=0.3$ and $0.5$, \sgt-1 reaches within one percentage point of full-graph performance on 9/19 and 13/19 datasets, respectively, while \sgt-$K$ reaches 12/19 and 16/19.

\begin{figure}[!t]
    \centering
    \includegraphics[width=\linewidth]{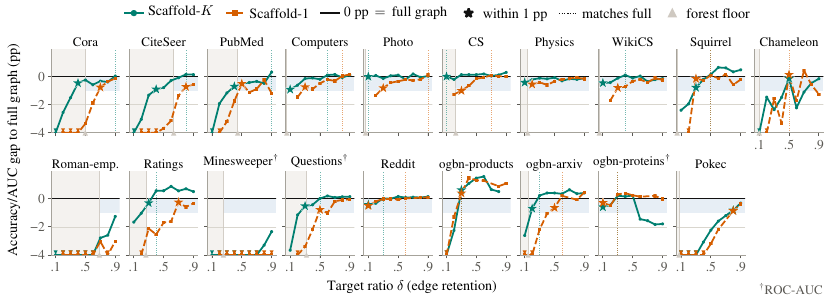}
   \caption{Performance gap of \sgt-1 and \sgt-$K$ to the full graph across $\delta$. Blue marks results within $1\%$; stars mark the first such $\delta$, and gray marks below the spanning-forest.}
    \label{fig:sgtvstuned}
\end{figure}



\subsection{Runtime and Memory Efficiency}
\label{subsec:exp-efficiency}
Figure~\ref{fig:runtime_profiling} shows that reducing the retained-edge ratio directly lowers both message-passing time and peak GPU memory on \textsc{Reddit} and \textsc{ogbn-products}. 
Even after including support graph construction, the reduced \texttt{Gather-Scatter/SpMM} workload yields lower end-to-end training cost at sparse budgets. 
The benefits are greatest at lower retention ratios, on large graphs, and over long runs, where per-epoch savings amortize sparsification overhead.

Figure~\ref{fig:runtime_all} further compares standalone sparsification and total end-to-end costs across large graphs: \sgt remains practical, while several more expensive structural or learned baselines incur substantially higher costs. We note that settings for \sgt and sampling methods can be adapted to trade off runtime versus quality. Here, we use the same settings as in the preceding accuracy experiments. Detailed runtimes, memory measurements, worker
scaling, and per-dataset results are reported in Appendix~\ref{app:runtime-details}.

\begin{figure}[!htbp]
    \centering
    \includegraphics[width=\linewidth]{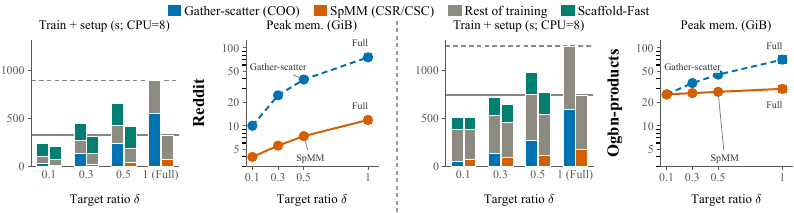}
    \caption{Runtime/memory profile of \sgtf-1 with \texttt{Gather-Scatter/SpMM} aggregation in comparison to the Full-GCN.}
    \label{fig:runtime_profiling}
\end{figure}

\begin{figure}[!htbp]
    \centering
    \includegraphics[width=\linewidth]{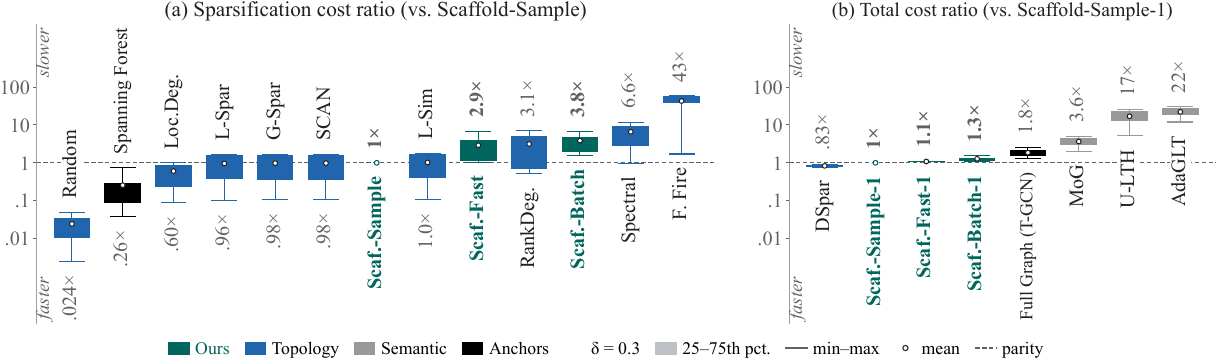}
    \caption{(a) Standalone sparsification cost relative to
    \sgtf, and (b) total end-to-end cost relative to \sgtsa-1 over $300$ epochs at $\delta=0.3$. More details are in Appendix~\ref{app:runtime-details}.}
    \label{fig:runtime_all}
\end{figure}


\subsection{Ablation and Structural Analysis}
\label{subsec:ablation}
\paragraph{Support-score components and structural quality.}
Figure~\ref{fig:components} compares components of Equation~\ref{eq:scaffold-global-objective} under the same support backbone (\texttt{SF}) and budget.
Prioritizing low dilation improves over random filling ($\emptyset$), which highlights the importance of short communication paths. Edge or node congestion alone can favor longer routes and is less effective, whereas combining dilation and edge congestion gives the strongest predictive association with test performance.
Node congestion captures a distinct theoretical failure mode in controlled examples (\S\ref{app:support-gnn-scenarios}), but provides limited additional benefit on these GNN benchmarks.
This pattern is not specific to \sgt. In Figure~\ref{fig:dilconvsaccuracy}, all sparsifiers share the same spanning forest and edge budget, isolating their edge-selection rules.
Methods that better reduce dilation and edge congestion achieve higher GNN accuracy, with dilation showing the strongest association. 
This may partly explain why Random remains competitive once connectivity is guaranteed, as random edges can shorten detours and partially alleviate structural bottlenecks.

\begin{figure}[!t]
    \centering
    \includegraphics[width=\linewidth]
    {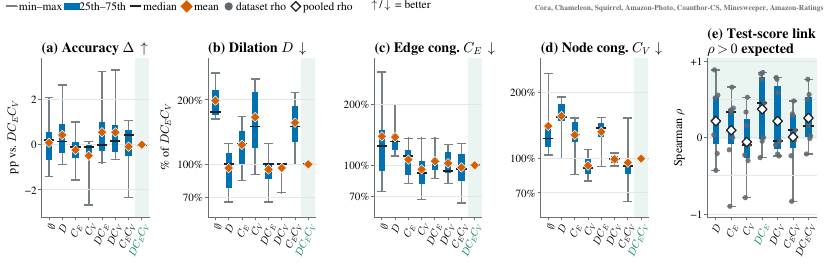}
    \caption{Dilation achieves the best individual performance, while combining
    dilation and edge congestion gives the strongest accuracy--structure
    trade-off and association with test performance.}
    \label{fig:components}
\end{figure}

\begin{figure}[!htbp]
    \centering
    \includegraphics[width=\linewidth]
    {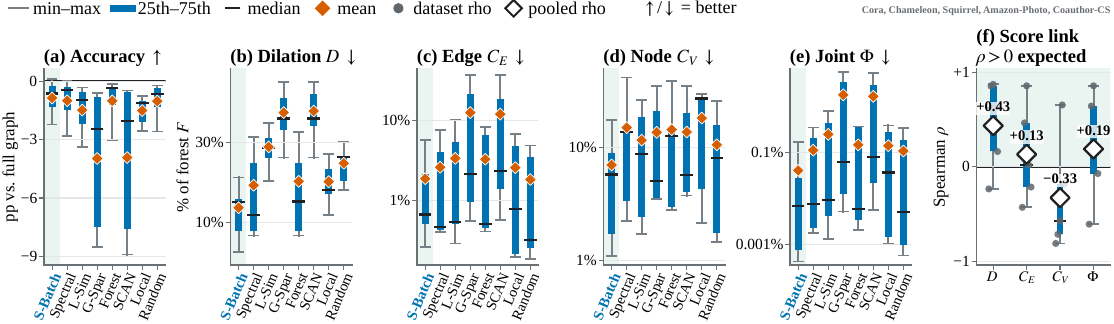}
    \caption{Support quality at $\delta=0.7$ from a common spanning-forest
    backbone. Lower dilation and edge congestion generally correspond to higher
    GNN accuracy.}
    \label{fig:dilconvsaccuracy}
\end{figure}

\paragraph{Scaffold variants and support backbones.}
Figure~\ref{fig:variants_backbones} shows a clear quality-cost trade-off: \texttt{Heap} and \texttt{Batch} provide stronger structural refinement at higher runtime, while \texttt{Fast} and \texttt{Sample} trade some structural quality for much lower construction cost. Predictive performance is not strictly determined by structural score, indicating that backbone choice and stochasticity also matter.

\begin{figure}[!htbp]
    \centering
    \includegraphics[width=\linewidth]    {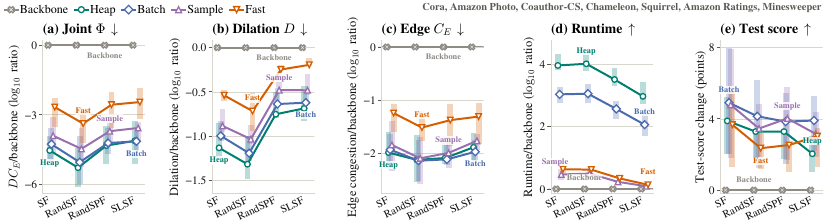}
    \caption{Quality--cost trade-offs across variants and backbones. \texttt{SF}: spanning forest; \texttt{RandSF}: random spanning forest; \texttt{RandSPF}: random shortest-path forest; \texttt{SLSF}: scalable low-stretch forest.}
    \label{fig:variants_backbones}
\end{figure}

\paragraph{Additional Ablations and Analyses.}
Appendix~\ref{app:ablation} extends the main analysis with \texttt{Greedy} variant and backbone comparisons
(\S\ref{app:variant-backbone}),
path-congestion aggregation choices ($p_E,p_V$)
(\S\ref{app:pathcongestion}), and analyses of inference supports $K$ and full-graph evaluation (\S\ref{app:k-sensitivity}).
We further study support-refresh frequency (\S\ref{app:refresh-frequency}),
synchronous versus asynchronous support generation (\S\ref{app:sync-async}), 
the interaction between training-time support refresh and inference topology (\S\ref{app:train-inference}),
and whether repeated structure-aware supports provide benefits beyond uniform random resampling
(\S\ref{app:sparsesamplerefresh}).
Additional experiments examine structural implications
(\S\ref{app:structural-implications}) and compatibility across GNN architectures (\S\ref{app:other-gnns}).

\section{Conclusions}
We introduced \sgt, a novel unsupervised topology-based graph sparsification framework that uses the support-theoretic notions of dilation and congestion to construct sparse communication graphs for GNNs.
Across benchmarks, \sgt maintains strong performance at low edge budgets while reducing memory and end-to-end training cost. These results indicate that, at a fixed edge budget, how omitted connections are represented by the retained graph matters for GNN performance.
Our theory characterizes the resulting structural distortion, but a general approximation guarantee for the full dilation-congestion objective remains open. The current evaluation focuses on static, undirected, unweighted node-classification graphs. Extensions of \sgt to task-aware, localized, and broader graph-learning settings are left for future work.


\clearpage
\newpage

\section*{Acknowledgments}
This work was supported in part by the U.S. Department of Energy, Office of Science,
Office of Advanced Scientific Computing Research, through the Competitive Portfolios
program at Pacific Northwest National Laboratory (PNNL). PNNL is a multi-program
national laboratory operated for the U.S. Department of Energy by Battelle Memorial
Institute under Contract No.~DE-AC05-76RL01830.

\section*{Acknowledgments}

This work was supported in part by the U.S. Department of Energy, Office of Science,
Office of Advanced Scientific Computing Research, through the Competitive Portfolios
program at Pacific Northwest National Laboratory (PNNL). PNNL is a multi-program
national laboratory operated for the U.S. Department of Energy by Battelle Memorial
Institute under Contract No.~DE-AC05-76RL01830.

\section*{Author Contributions}
Mahantesh M. Halappanavar originated the idea of bringing support-graph-theoretic principles to graph sparsification for GNNs and helped shape the conceptual direction of \sgt. Siddhartha Shankar Das led the development of the methodology and algorithms, software implementation, experimental evaluation, and manuscript preparation. Sai Karthik Navuluru assisted with the development and implementation of the proposed methods. All authors contributed to the work and reviewed the manuscript.

\section*{AI Use Disclosure}
Generative AI tools were used to assist with polishing the manuscript and
supporting research execution, including software implementation and experimental workflows. All AI-assisted content and code were
reviewed, revised, and validated by the authors. The authors independently
verified the reported results and claims and take full responsibility for the
final content of this work.

\section*{Reproducibility Statement}
We provide two code repositories to support reproducibility. (i)~\scaffoldlink~contains the standalone, installable \sgt package implementing the proposed graph sparsification methods. (ii)~\scaffoldgnnlink~contains the integration of \sgt with the GNN benchmarks used in this work, together with implementations of the evaluated baselines and the dataset-specific configurations and hyperparameter settings used to reproduce the reported results. Appendix~\ref{app:experimental} provides additional details on the datasets, training procedures, and evaluation protocols. All support, variant, and checkpoint selection is performed exclusively using validation data; test labels are never used for model or method selection.


\bibliography{iclr2027_conference}
\bibliographystyle{iclr2027_conference}
\clearpage
\newpage
\appendix

\startcontents[appendix]
\section*{Appendix Contents}
\printcontents[appendix]{}{1}{\setcounter{tocdepth}{2}}
\section{Supplementary Preliminaries}
\label{app:preliminaries}
This section provides additional support-theoretic background, weighted-path conventions, and examples illustrating the role of support quantities in GNN message passing.

\subsection{Notations}
\begin{table*}[!htbp]
\centering
\caption{Core notation used in \sgt.}
\label{tab:notations}
\footnotesize
\setlength{\tabcolsep}{5pt}
\renewcommand{\arraystretch}{1.04}

\begin{tabularx}{\textwidth}{@{}
  >{\raggedright\arraybackslash}p{0.28\textwidth}
  >{\raggedright\arraybackslash}X
@{}}
\toprule
\textbf{Symbol} & \textbf{Meaning} \\
\midrule

$\gG=(\gV,\gE,\vw)$; $n,m$
& Input graph with $n=|\gV|$ nodes and $m=|\gE|$ edges. \\

\rowcolor[gray]{0.975}
$\gH_t$; $\gF$
& Current sparse support and initial spanning-forest backbone. \\

$\delta$; $q=\lceil\delta m\rceil$
& Edge-retention ratio and target edge budget. \\

\rowcolor[gray]{0.975}
$\gI_t=\gE\setminus\gE_{\gH_t}$
& Candidate edges omitted from the current support. \\

$P_e$
& Supporting path of omitted edge $e$ in $\gH_t$. \\

\rowcolor[gray]{0.975}
$D(\gH),\,C_E(\gH),\,C_V(\gH)$
& Maximum dilation, edge congestion, and node congestion. \\

$\pathDil_t(e),\,\eConPath_t(e),\,\vConPath_t(e)$
& Candidate path dilation and edge/node congestion scores. \\

\rowcolor[gray]{0.975}
$\alpha,\beta_E,\beta_V$
& Weights of the three support-score components. \\

$\score_t(e)$
& Normalized \sgt score of candidate edge $e$. \\

\rowcolor[gray]{0.975}
$\gH_{\mathrm{union}}$
& Union of retained supports used for multi-support inference. \\

\bottomrule
\end{tabularx}
\end{table*}

\subsection{Support-Theoretic Background}
\label{app:support-theory}
Support graph theory was developed for constructing sparse combinatorial
preconditioners for symmetric diagonally dominant linear systems~\citep{Vaidya1991,Boman2003,Bern2006Support}. For graph Laplacians
$\mL_{\gG}$ and $\mL_{\gH}$, the support number is defined as
\[
\sigma(\mL_{\gG},\mL_{\gH})
=
\min\{\tau:\mL_{\gG}\preceq\tau\mL_{\gH}\},
\]
where smaller $\tau$ indicates that the sparse graph $\gH$ provides a
closer spectral support for $\gG$.

For graph Laplacians, this algebraic relation also admits a combinatorial
interpretation. An edge omitted from $\gH$ is represented by a supporting
path in $\gH$ connecting the same pair of endpoints. The quality of this
representation depends on two complementary properties: (i) how long the
supporting path becomes relative to the omitted edge, and (ii) how strongly
supporting paths overlap on the retained graph. These effects are captured
by \emph{dilation} and \emph{congestion}, respectively. 
Classical congestion--dilation arguments relate their joint effect to the quality of the support graph~\citep{Vaidya1991,Boman2003,Bern2006Support}.

\sgt builds directly on this perspective. Rather than using the sparse
support $\gH$ as a preconditioner for solving a linear system, we use it as
the communication graph for GNN message passing. An omitted edge is therefore
well supported when its endpoints remain connected by a short path and when
that path does not pass through heavily shared parts of the retained graph.
This motivates the dilation--congestion criteria used throughout \sgt.

\subsection{Weighted Formulation of \sgt}
\label{app:weighted-support}
The main text uses the unweighted formulation used in our experiments.
More generally, let $\ell_e>0$ denote the path length of edge $e$ and
$\omega_e>0$ its routing load. For an omitted edge $e=(u,v)$, let
\[
P_e=\operatorname{ShortestPath}_{\gH}(u,v;\ell),
\qquad
D_{\ell}(\gH)
=
\max_{e\in\gI}
\frac{\sum_{a\in P_e}\ell_a}{\ell_e}.
\]
Weighted edge and node congestion are
\[
C_E^{\omega}(\gH)
=
\max_{a\in\gE_{\gH}}
\sum_{\substack{e\in\gI\\a\in P_e}}\omega_e,
\qquad
C_V^{\omega}(\gH)
=
\max_{x\in\gV}
\sum_{\substack{e\in\gI\\x\in\operatorname{int}(P_e)}}\omega_e .
\]
The corresponding weighted objective is
\[
\Phi_{\ell,\omega}(\gH)
=
(1+D_{\ell}(\gH))^\alpha
(1+C_E^{\omega}(\gH))^{\beta_E}
(1+C_V^{\omega}(\gH))^{\beta_V}.
\]
Setting $\ell_e=\omega_e=1$ recovers
Equation~\ref{eq:scaffold-global-objective}.

The candidate-wise surrogate extends analogously. For the current candidate
set $\gI_t$, define
\[
D_t^{\mathrm{path},\ell}(e)
=
\frac{\sum_{a\in P_e}\ell_a}{\ell_e},
\qquad
c_{E,t}^{\omega}(a)
=
\sum_{\substack{f\in\gI_t\\a\in P_f}}\omega_f,
\qquad
c_{V,t}^{\omega}(x)
=
\sum_{\substack{f\in\gI_t\\x\in\operatorname{int}(P_f)}}\omega_f.
\]
The path-congestion scores are obtained from these weighted loads using the
same power-mean aggregation as
Equations~\ref{eq:scaffold-path-edge-congestion}--\ref{eq:scaffold-path-node-congestion},
and Equation~\ref{eq:scaffold-score} is applied unchanged after normalization.
Thus, $\ell_e=\omega_e=1$ recovers the unweighted \sgt surrogate used
throughout this work.

The spanning-forest formulation requires
$q\ge n-c(\gG)$ to preserve every connected component. If
$q<n-c(\gG)$, no support can simultaneously satisfy the requested budget and
preserve all components; in this case, we enforce the budget and do not invoke
guarantees requiring a spanning support.

\subsection{Why the Support Quantities Matter for GNNs}
\label{app:support-gnn-scenarios}

\begin{figure}[!t]
    \centering
    \includegraphics[width=\linewidth]
    {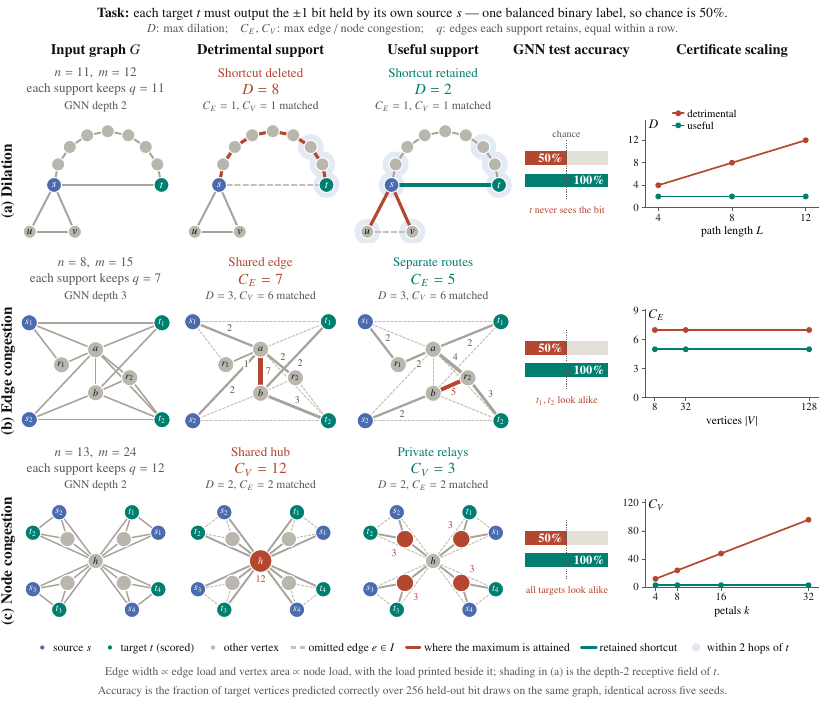}
    \caption{Controlled roles of dilation and congestion in GNN message passing.
    Each row varies one support quantity while matching the other two at the
    same edge budget. High dilation causes finite-hop under-reaching, while
    high edge or node congestion creates shared communication bottlenecks.
    }
    \label{fig:appendix-dil-con-scenario}
\end{figure}

Equation~\ref{eq:scaffold-global-objective} combines dilation, edge
congestion, and node congestion in the \sgt objective. Classical support theory
directly motivates the first two quantities, while node congestion is included
to capture concentration at intermediate vertices during message passing.
Although node congestion provides limited additional benefit on our
node-classification benchmarks, it represents a distinct communication failure
mode that can arise independently of edge congestion.

Figure~\ref{fig:appendix-dil-con-scenario} gives controlled
node-classification examples that isolate the three quantities. In each row,
the detrimental and useful supports retain the same number of edges and are
matched on the other two support quantities, so the structural difference is
restricted to the quantity being studied. Each target must predict the
$\pm1$ bit associated with its source, giving chance accuracy of $50\%$.

For \emph{dilation}, deleting a shortcut increases the source--target detour
from $D=2$ to $D=8$ while both congestion measures remain fixed. The source
then lies outside the two-layer receptive field, reducing accuracy from
$100\%$ to chance. This illustrates finite-hop under-reaching caused by long
supporting paths.

For \emph{edge congestion}, the two supports have the same dilation and node
congestion, but the detrimental support routes more omitted relations through
a shared edge, increasing $C_E$ from $5$ to $7$. The resulting bottleneck
makes the paired targets indistinguishable, whereas separating the routes
recovers perfect prediction.

For \emph{node congestion}, dilation and edge congestion remain fixed while
the detrimental support routes multiple relations through a common hub,
increasing $C_V$ from $3$ to $12$. The shared intermediate representation
causes the targets to become indistinguishable, whereas private relays avoid
this node-level bottleneck.

These constructions show that dilation, edge congestion, and node congestion
capture complementary support-level communication failures, even when their
empirical importance differs across benchmark families.

\section{Proofs and Structural Guarantees}
\label{app:scaffold-theory}

The guarantees below are stated for the unweighted setting used in our
experiments and in the main-text objective. 
We assume that $\gH$ preserves the connected components of
$\gG$, so every omitted edge has a supporting path. The weighted formulation
is given in Appendix~\ref{app:weighted-support}; a weighted extension of the
Laplacian bound is provided in the following sections.

\subsection{Optimization of the Global Objective is NP-Hard}
\label{app:nphardproof}

\begin{theorem}[Hardness of budgeted support sparsification]
\label{thm:scaffold-hardness}
For any $\alpha>0$, optimizing
Equation~\ref{eq:scaffold-global-objective} is NP-hard even for connected
unweighted graphs with $\beta_E=\beta_V=0$. The hardness already holds for
$q=n-1$.
\end{theorem}

\begin{proof}~
Set $\beta_E=\beta_V=0$ and $q=n-1$. Since $\gG$ is connected, every feasible
support is a spanning tree $\gT$, and minimizing the objective is equivalent
to minimizing $D(\gT)$.
For a spanning tree, the maximum stretch satisfies
\[
\max_{u\neq v}
\frac{\operatorname{dist}_{\gT}(u,v)}
     {\operatorname{dist}_{\gG}(u,v)}
=
\max_{(u,v)\in\gE}
\operatorname{dist}_{\gT}(u,v)
=
\max\{1,D(\gT)\}.
\]
The first equality follows by taking a shortest $u$--$v$ path in $\gG$ and
replacing each of its edges by the corresponding path in $\gT$; hence the
stretch of any vertex pair is bounded by the maximum stretch of an original
edge. The second equality follows because retained tree edges have stretch
$1$, while $D(\gT)$ is the maximum stretch over omitted edges.

Thus, minimizing $D(\gT)$ is equivalent to the Minimum Maximum-Stretch
Spanning Tree problem, which is NP-hard
\citep{Cai1995TreeSpanners,emek2009approximating}.
\end{proof}

\subsection{Dilation Reduction is Non-Submodular}
\label{app:scaffold-nonsubmodular}

\begin{proposition}[Dilation reduction is not submodular]
\label{prop:scaffold-nonsubmodular}
For a fixed backbone $\gF$, let
\[
f(S)=D(\gF)-D(\gF\cup S),
\]
where dilation is measured over the edges that remain omitted.
Then $f$ is not submodular in general.
\end{proposition}

\begin{proof}~
Consider three branches
\[
r-a_1-a_2-a_3,\qquad
r-b_1-b_2-b_3,\qquad
r-c_1-c_2,
\]
with candidate edges
$e_a=(r,a_3)$, $e_b=(r,b_3)$, and $e_c=(r,c_2)$.
Their initial dilations are $3,3,2$.

Let $A=\emptyset$, $B=\{e_a\}$, and $e=e_b$.
Adding $e_b$ to $A$ leaves $e_a$ with dilation $3$, so
\[
f(A\cup\{e_b\})-f(A)=0.
\]
After $e_a$ is already retained, adding $e_b$ reduces the maximum remaining
dilation from $3$ to $2$, giving
\[
f(B\cup\{e_b\})-f(B)=1.
\]
Hence the marginal gain increases as the selected set grows, violating the
diminishing-returns condition.
\end{proof}

Thus, the standard submodular-greedy approximation guarantee does not apply.

\subsection{Guarantee for Dilation-Only Selection}
\label{app:returnedsupportproof}

The following result applies only when candidate selection is based on
dilation; it does not directly apply when congestion changes the ordering.

\begin{theorem}[Dilation-only returned-support bound]
\label{thm:scaffold-returned-support}
Let $\gF$ be a spanning-forest backbone of an unweighted graph, with
$|\gE_{\gF}|\le q<m$, and let
\[
T=q-|\gE_{\gF}|.
\]
Order the initial candidate dilations as
$
D_{(1)}\ge D_{(2)}\ge\cdots
$
and define
\[
S_{\gF}
=
\sum_{e\in\gE\setminus\gE_{\gF}}
\pathDil_{\gF}(e).
\]
If either
(i) each insertion chooses a candidate of maximum current dilation, or
(ii) the $T$ candidates with largest initial dilation are selected,
then
\[
D(\gH_T)
\le
D_{(T+1)}
\le
\frac{S_{\gF}}{T+1}.
\]
\end{theorem}

\begin{proof}~
Adding retained edges cannot increase shortest-path dilation.
Let $\theta=D_{(T+1)}$.

For static selection, every remaining candidate initially has dilation at
most $\theta$, and therefore also has final dilation at most $\theta$.

For dynamic selection, suppose some edge remaining after $T$ insertions has
dilation greater than $\theta$. Its dilation was then greater than $\theta$
at every earlier step. Since each step chooses a maximum-dilation candidate,
all $T$ selected edges also had current, and hence initial, dilation greater
than $\theta$. Together with the remaining edge, this gives $T+1$ initial
candidates with dilation greater than $D_{(T+1)}$, a contradiction.
Thus,
\[
S_{\gF}\ge (T+1)D_{(T+1)},
\]
which gives the second inequality.
\end{proof}

This result controls the largest remaining detour, but not maximum
congestion, which motivates including congestion in the full \sgt score.

\subsection{Dilation and Edge Congestion Bound Structural Distortion}
\label{app:supportdistortionproof}

\begin{theorem}[Dilation--congestion support bound]
\label{thm:scaffold-gnn-tradeoff}
Let $\gH\subseteq\gG$ be an unweighted support preserving the connected
components of $\gG$. Then
\[
\mL_{\gH}
\preceq
\mL_{\gG}
\preceq
\tau_{\gH}\mL_{\gH},
\qquad
\tau_{\gH}
=
1+D(\gH)C_E(\gH).
\]
\end{theorem}

\begin{proof}~
Since $\gH\subseteq\gG$,
$\mL_{\gH}\preceq\mL_{\gG}$.

For an omitted edge $e=(u,v)$ with supporting path $P_e$,
Cauchy--Schwarz gives
\[
(x_u-x_v)^2
=
\left(
\sum_{a=(i,j)\in P_e}(x_i-x_j)
\right)^2
\le
|P_e|
\sum_{a=(i,j)\in P_e}(x_i-x_j)^2.
\]
Because $|P_e|\le D(\gH)$,
\[
(x_u-x_v)^2
\le
D(\gH)
\sum_{a=(i,j)\in P_e}(x_i-x_j)^2.
\]
Summing over all omitted edges, each retained edge occurs in at most
$C_E(\gH)$ supporting paths. Therefore
\[
\vx^\top(\mL_{\gG}-\mL_{\gH})\vx
\le
D(\gH)C_E(\gH)\,
\vx^\top\mL_{\gH}\vx.
\]
Adding $\vx^\top\mL_{\gH}\vx$ to both sides proves the result.
\end{proof}

\paragraph{Weighted extension.}
For positive edge conductances $w_e$, let $r_e=1/w_e$ and define
\[
D_R(\gH)
=
\max_{e=(u,v)\in\gI}
w_e\sum_{a\in P_e}\frac{1}{w_a}.
\]
If $\gH$ retains the original edge weights, the same argument with weighted
Cauchy--Schwarz gives
\[
\mL_{\gH}
\preceq
\mL_{\gG}
\preceq
\bigl(1+D_R(\gH)C_E(\gH)\bigr)\mL_{\gH}.
\]
For unit weights, $D_R(\gH)=D(\gH)$.

\subsection{Effective-Resistance, Distance, and Congestion Consequences}
\label{app:structural-consequences}

\paragraph{Effective resistance.}
For connected $\gG$ and $\gH$, Theorem~\ref{thm:scaffold-gnn-tradeoff}
implies
\[
R_{\gG}(u,v)
\le
R_{\gH}(u,v)
\le
\tau_{\gH}R_{\gG}(u,v)
\qquad
\forall\,u,v\in\gV.
\]
On $\mathbf{1}^{\perp}$ the Laplacians are positive definite, so
inverting their ordering gives
\[
\mL_{\gG}^{+}
\preceq
\mL_{\gH}^{+}
\preceq
\tau_{\gH}\mL_{\gG}^{+},
\]
and the result follows from
$
R_{\gX}(u,v)
=
(\ve_u-\ve_v)^\top
\mL_{\gX}^{+}
(\ve_u-\ve_v).
$

\paragraph{Communication distance.}
For unweighted graphs,
\[
\operatorname{dist}_{\gH}(u,v)
\le
\max\{1,D(\gH)\}
\operatorname{dist}_{\gG}(u,v).
\]
Intuitively, take a shortest $u$--$v$ path in $\gG$ and replace every
omitted edge by its supporting path in $\gH$. Each replacement has length
at most $D(\gH)$.

\paragraph{Edge-cut concentration.}
For any $S\subseteq\gV$,
\[
|\partial_{\gG}S|
\le
\bigl(1+C_E(\gH)\bigr)|\partial_{\gH}S|.
\]
Every omitted edge crossing the cut has a supporting path that also crosses
the cut. Assign it to one retained crossing edge on that path. Since each
retained edge lies on at most $C_E(\gH)$ supporting paths, each retained cut
edge receives at most $C_E(\gH)$ such assignments.

\paragraph{Node-neighborhood concentration.}
Let $N_{\gX}(S)\setminus S$ denote the vertices outside $S$ adjacent to a
vertex in $S$ in graph $\gX$. Then
\[
|N_{\gG}(S)\setminus S|
\le
\bigl(1+C_V(\gH)\bigr)
|N_{\gH}(S)\setminus S|.
\]
For each $\gG$-neighbor of $S$ that is not already an $\gH$-neighbor,
follow its omitted edge's supporting path and assign the vertex to the first
vertex of that path outside $S$. This vertex is internal to the supporting
path, and each such vertex receives at most $C_V(\gH)$ assignments.

\paragraph{Scope of the guarantees.}
These results concern the structure of the sparse support. The
dilation-only bound does not apply to the full dilation--congestion selection
rule, and we do not claim a general approximation ratio or monotonic
improvement in downstream GNN accuracy.

\subsection{Support Quality Under Edge Unions}
\label{app:unionproof}

\begin{proposition}[Monotonicity under edge addition]
\label{prop:union-support}
Let
$\gH\subseteq\gH'\subseteq\gG$, with both supports preserving the connected
components of $\gG$ and retaining the original edge weights. Then
\[
\sigma(\mL_{\gG},\mL_{\gH'})
\le
\sigma(\mL_{\gG},\mL_{\gH}).
\]
\end{proposition}

\begin{proof}~
Since adding retained edges adds positive-semidefinite edge Laplacians,
\[
\mL_{\gH}\preceq\mL_{\gH'}.
\]
By definition of the support number,
\[
\mL_{\gG}
\preceq
\sigma(\mL_{\gG},\mL_{\gH})\,\mL_{\gH}
\preceq
\sigma(\mL_{\gG},\mL_{\gH})\,\mL_{\gH'}.
\]
Hence
\[
\sigma(\mL_{\gG},\mL_{\gH'})
\le
\sigma(\mL_{\gG},\mL_{\gH}).
\]
\end{proof}

For the \sgt-$K$ inference union,
$\gH^{(k)}\subseteq\gH_{\mathrm{union}}$ for every retained support, so
\[
\sigma(\mL_{\gG},\mL_{\gH_{\mathrm{union}}})
\le
\min_k
\sigma(\mL_{\gG},\mL_{\gH^{(k)}}).
\]
Thus, taking the union cannot worsen the support-number approximation of any
constituent support. This is a structural statement and does not imply
monotonic improvement in GNN predictive performance.

\section{Supplementary Method Details}
\label{App:method}

This appendix provides additional details on the support backbones, scalable
\sgt variants, runtime, and GNN training/inference procedures described in
Section~\ref{sec:method}.

\subsection{\sgt Algorithms, Backbones, and Variants}
\label{app:scaffold-algorithms}

Let $q=\lceil\delta m\rceil$. For feasible budgets, each \sgt variant starts
from a spanning-forest backbone $\gF$, with one spanning tree per connected
component of $\gG$. If the target lies below the spanning-forest floor,
construction stops at $q$ edges and guarantees requiring a spanning backbone
are not invoked. For a current support $\gH$ and candidate set $\gB$,
\textsc{SupportScores}$(\gG,\gH,\gB)$ computes supporting paths and evaluates
Eq.~\ref{eq:scaffold-score}, with congestion and normalization computed over
$\gB$.

\subsubsection{Support Backbones}
\label{app:scaffold-backbone-details}
The backbone $\gF$ may be deterministic, randomized, weight-aware, or
low-stretch~\citep{elkin2005lower,Abraham2008LowStretch,Abraham2012Nearly}.
Different backbone constructions induce different supporting paths and, consequently, different dilation and congestion profiles. All benchmark graphs in this work are unweighted; weighted extensions are described in Appendix~\ref{app:weighted-support}. The subsequent \sgt scoring principle is otherwise unchanged.

We use spanning-forest terminology throughout: \texttt{SF} denotes the
deterministic construction, \texttt{RandSF} a randomized forest,
\texttt{SPF} a shortest-path forest, and \texttt{GLSF}/\texttt{LLSF}
low-stretch constructions. \texttt{MaxSF} and \texttt{MinSF} denote
maximum- and minimum-weight spanning forests, respectively. Under the unit
weights used in our experiments, the two objectives are identical and,
with the same deterministic tie-breaking, yield the same forest.


\begin{figure}[!t]
    \centering
    \includegraphics[width=\linewidth]
    {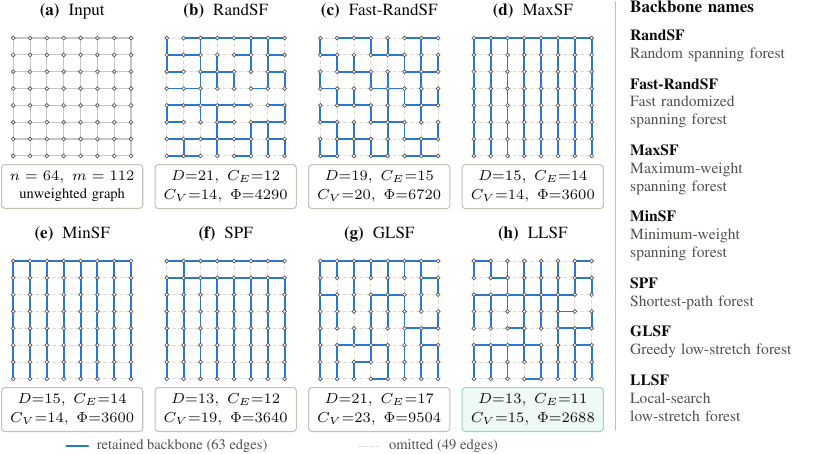}
    \caption{Support-forest backbones on a connected $8\times8$ grid.
    (a) Input; (b)--(h) spanning-forest constructions with $63$ retained
    edges (blue) and $49$ omitted edges (dashes). Since the input is connected,
    each forest contains one tree. Panels report dilation $D$, edge and node
    congestion $C_E,C_V$, and
    $\Phi=(1+D)(1+C_E)(1+C_V)$ before edge additions.
    Green marks the lowest $\Phi$. MaxSF and MinSF coincide under unit
    weights; LLSF refines GLSF through local swaps.}
    \label{fig:sgt-backbones}
\end{figure}

\subsubsection{Variant Mechanisms}
\label{app:scaffold-variant-details}

\sgte is the full dynamic reference and recomputes all candidate paths and
scores after every insertion; its pseudocode is given in
Algorithm~\ref{alg:scaffold-greedy} in the main paper.
The remaining variants progressively reduce this recomputation, as summarized
in Figure~\ref{fig:scaffold-variant-mechanisms}.

\begin{figure*}[!htbp]
    \centering
    \includegraphics[width=\textwidth]{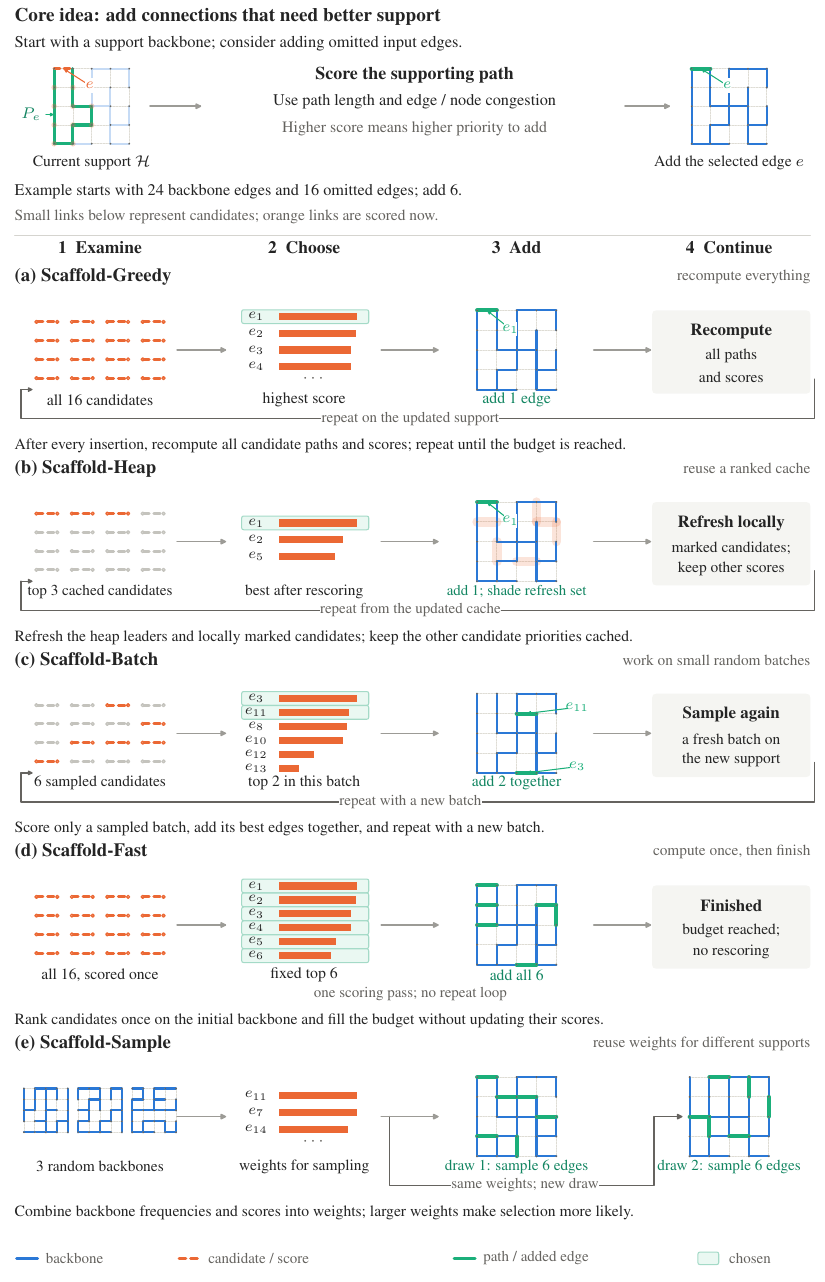}
    \caption{(a) \texttt{Greedy} recomputes all paths and scores after each insertion.
    (b) \texttt{Heap} caches priorities and selectively refreshes affected candidates.
    (c) \texttt{Batch} scores sampled batches and inserts their top-$r$ edges.
    (d) \texttt{Fast} scores candidates once on the initial backbone.
    (e) \texttt{Sample} aggregates randomized-backbone signals for repeated sparse draws.}
    \label{fig:scaffold-variant-mechanisms}
\end{figure*}

\paragraph{\sgth.}
\sgth caches candidate paths and priorities in a max-heap. At each round,
a small active set is refreshed, the highest-scoring candidate is inserted,
and only candidates affected by this insertion are reconsidered. Cached
priorities outside this set may therefore be stale, making \sgth a lazy
approximation to Greedy.

\begin{algorithm}[!htbp]
\small
\caption{\textsc{\sgth}: selective dynamic refresh}
\label{alg:scaffold-heap}
\begin{algorithmic}[1]
\REQUIRE $\gG,\delta$, active size $\kappa$, refresh cap $c$
\ENSURE Sparse support $\gH$
\STATE $q\leftarrow\lceil\delta m\rceil$; choose $\gF$; $\gH\leftarrow\gF$
\STATE initialize cached paths/scores and max-heap $Q$
\WHILE{$|\gE_{\gH}|<q$}
    \STATE $\gA\leftarrow\textsc{PopCandidates}(Q,\kappa)$
    \STATE recompute paths and scores for $\gA$ on $\gH$
    \STATE $e^\star\leftarrow\arg\max_{e\in\gA}\score(e)$
    \STATE $\gH\leftarrow\gH\cup\{e^\star\}$
    \STATE refresh up to $c$ affected candidates and update $Q$
\ENDWHILE
\STATE \textbf{return} $\gH$
\end{algorithmic}
\end{algorithm}

\paragraph{\sgtb.}
\sgtb evaluates only a random subset of the remaining candidates at each
round. The sampled batch is rescored on the current support and its top-$r$
edges are inserted before the next refresh. Depending on batch settings and selection number, we can make \sgtb scalable for large graphs.

\begin{algorithm}[!htbp]
\small
\caption{\textsc{\sgtb}: batched dynamic scoring}
\label{alg:scaffold-batch}
\begin{algorithmic}[1]
\REQUIRE $\gG,\delta$, batch size $b$, insertion size $r$
\ENSURE Sparse support $\gH$
\STATE $q\leftarrow\lceil\delta m\rceil$; choose $\gF$; $\gH\leftarrow\gF$
\WHILE{$|\gE_{\gH}|<q$}
    \STATE $\gI\leftarrow\gE\setminus\gE_{\gH}$
    \STATE sample $\gB\subseteq\gI$, $|\gB|=\min\{b,|\gI|\}$
    \STATE $\widehat{s}\leftarrow
           \textsc{SupportScores}(\gG,\gH,\gB)$
    \STATE $\gS\leftarrow
           \textsc{TopK}
           (\gB,\min\{r,q-|\gE_{\gH}|\};\widehat{s})$
    \STATE $\gH\leftarrow\gH\cup\gS$
\ENDWHILE
\STATE \textbf{return} $\gH$
\end{algorithmic}
\end{algorithm}

\paragraph{\sgtf.}
\sgtf removes dynamic rescoring. Since $\gF$ is a forest, each candidate whose
endpoints lie in the same component has a unique path in $\gF$. LCA queries,
forest-difference accumulation, and root-prefix sums allow all initial path
quantities to be computed together. The resulting ranking is then kept fixed.

\begin{algorithm}[!htbp]
\small
\caption{\textsc{\sgtf}: static forest scoring}
\label{alg:scaffold-fast}
\begin{algorithmic}[1]
\REQUIRE $\gG,\delta$
\ENSURE Sparse support $\gH$
\STATE $q\leftarrow\lceil\delta m\rceil$; choose backbone $\gF$
\STATE $\gI\leftarrow\gE\setminus\gE_{\gF}$
\STATE $s_0\leftarrow\textsc{ForestScores}(\gG,\gF)$
       \COMMENT{score once}
\STATE $\gS\leftarrow
       \textsc{TopK}(\gI,q-|\gE_{\gF}|;s_0)$
\STATE \textbf{return} $\gF\cup\gS$
\end{algorithmic}
\end{algorithm}

\noindent
\textsc{ForestScores} roots each tree in the forest, constructs the
corresponding LCA structures, accumulates candidate-path edge/node loads, and
uses prefix sums to recover the quantities in
Eq.~\ref{eq:scaffold-score}. These are the same initial scores computed by
\sgte at $t=0$; \sgtf differs by keeping them fixed during selection.

\paragraph{\sgtsa.}
\sgtsa converts support information from multiple randomized spanning forests
into reusable sampling weights. Let $f_e$ summarize how frequently edge $e$
occurs in the scoring backbones and $a_e$ its normalized candidate-score
signal. These quantities are precomputed once and combined into sampling
weights $\pi_e$ according to the policy specified in the experimental
settings.

\begin{algorithm}[!htbp]
\small
\caption{\textsc{\sgtsa}: reusable stochastic supports}
\label{alg:scaffold-sample}
\begin{algorithmic}[1]
\REQUIRE $\gG,\delta$, number of scoring backbones $R$
\ENSURE Sparse support $\gH$
\STATE $(f,a)\leftarrow\textsc{CachedForestSignals}(\gG,R)$
       \COMMENT{precompute once}
\STATE $q\leftarrow\lceil\delta m\rceil$; choose $\gF$
\STATE $\gI\leftarrow\gE\setminus\gE_{\gF}$;
       $k\leftarrow q-|\gE_{\gF}|$
\IF{$k=0$}
    \STATE \textbf{return} $\gF$
\ENDIF
\STATE $\pi\leftarrow\textsc{SamplingWeights}(f,a;\gI)$
\STATE sample exactly $k$ distinct edges $\gS\subseteq\gI$ using $\pi$
\STATE \textbf{return} $\gF\cup\gS$
\end{algorithmic}
\end{algorithm}

\subsubsection{Example Support Quality Across \sgt Variants}
\label{app:variant-example}

Figure~\ref{fig:sgtvariants_updates} compares the five \sgt variants
using the same \texttt{RandSF} backbone and edge budget. Greedy achieves the lowest
dilation--congestion objective in this example, with \texttt{Heap} remaining close.
\texttt{Batch}, \texttt{Fast}, and \texttt{Sample} use progressively cheaper selection policies and produce higher objective values here, illustrating how candidate selection affects support quality at fixed sparsity.

\begin{figure}[!t]
    \centering
    \includegraphics[width=\linewidth]
    {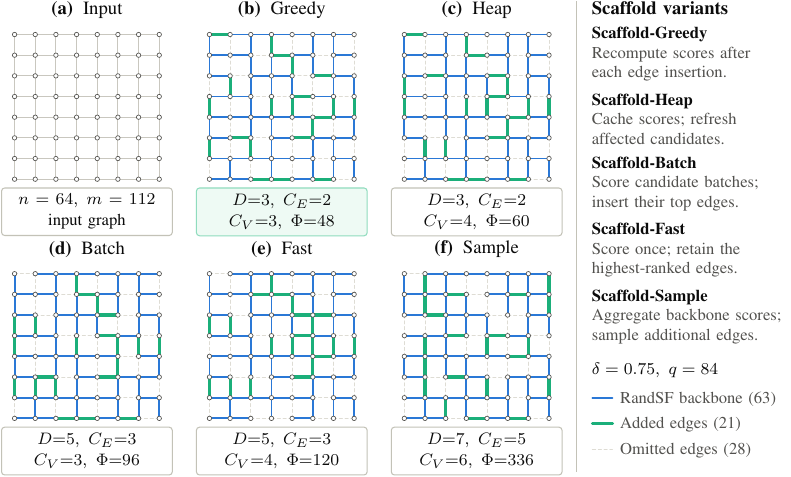}
    \caption{\sgt variants at a fixed edge budget.
    (a) Input grid. (b)--(f) Each variant retains $84$ edges
    ($\delta=0.75$): a shared RandSF backbone (blue) and $21$ added
    edges (green); dashed edges are omitted.
    Panels report dilation $D$, edge and node congestion $C_E,C_V$,
    and $\Phi=(1+D)(1+C_E)(1+C_V)$.
    Green shading marks the lowest $\Phi$ in this example.}
    \label{fig:sgtvariants_updates}
\end{figure}

\subsection{\sgt-1 and \sgt-$K$ Training and Inference}
\label{app:gnntraining}

Section~\ref{subsec:scaffold-gnn} introduces the two usage modes, and Algorithm~\ref{alg:scaffold-gnn} summarizes their procedures.
\sgt-1 constructs a single $q$-edge support and uses it throughout training and inference. 
In contrast, \sgt-$K$ periodically generates new $q$-edge supports and maintains an online bank $\mathcal{B}$ containing at most the $K$ strongest supports according to validation performance.

At inference, the retained supports are combined as $H_{\mathrm{union}}=\bigcup_{H\in\mathcal{B}}H$ and evaluated in a single
forward pass. Since every retained support is a subgraph of
$H_{\mathrm{union}}$, the union cannot worsen its support-number approximation
relative to any constituent support (Proposition~\ref{prop:union-support}).
Realized union sizes are reported in Appendix~\ref{app:inference_size}, while sensitivity to $K$, refresh frequency, asynchronous generation, and inference topology is examined
in Appendix~\ref{app:k-sensitivity}--\ref{app:train-inference}.

\begin{algorithm}[!htbp]
\small
\caption{\sgt-1 and \sgt-$K$ training and inference}
\label{alg:scaffold-gnn}

\begin{minipage}[t]{0.47\linewidth}
\vspace{0pt}
\textbf{(a) \sgt-1}
\begin{algorithmic}[1]
\REQUIRE $\gG,\mX,\vy,\delta$, GNN $f_\theta$
\ENSURE Prediction $\widehat{\vy}$

\STATE $\gH\leftarrow\textsc{\sgt}(\gG,\delta)$
       \COMMENT{$q$ edges}
\STATE train $f_\theta$ on $\gH$
\STATE select checkpoint $\theta^\star$
\STATE $\widehat{\vy}\leftarrow
       f_{\theta^\star}(\mX,\gH)$
\STATE \textbf{return} $\widehat{\vy}$
\end{algorithmic}
\end{minipage}
\hfill
\begin{minipage}[t]{0.51\linewidth}
\vspace{0pt}
\textbf{(b) \sgt-$K$}
\begin{algorithmic}[1]
\REQUIRE $\gG,\mX,\vy,\delta,\rho,K,T_{\mathrm{train}}$, $f_\theta$
\ENSURE Prediction $\widehat{\vy}$

\STATE $\mathcal{B}\leftarrow\emptyset$
       \COMMENT{at most $K$ supports}
\STATE $J\leftarrow\lceil T_{\mathrm{train}}/\rho\rceil$

\FOR{$j=1,\ldots,J$}
    \STATE $\gH_j\leftarrow
           \textsc{\sgt}(\gG,\delta;\text{random seed})$
           \COMMENT{$q$ edges}
    \STATE train on $\gH_j$ for the next $\rho$ epochs
    \STATE evaluate $\gH_j$ and supports in $\mathcal{B}$
           on validation data
    \STATE $\mathcal{B}\leftarrow
           \textsc{TopK}(\mathcal{B}\cup\{\gH_j\},K)$
\ENDFOR

\STATE select checkpoint $\theta^\star$
\STATE re-rank supports in $\mathcal{B}$ using $\theta^\star$
\STATE $\gH_{\mathrm{union}}\leftarrow
       \displaystyle\bigcup_{\gH\in\mathcal{B}}\gH$
\STATE $\widehat{\vy}\leftarrow
       f_{\theta^\star}(\mX,\gH_{\mathrm{union}})$
\STATE \textbf{return} $\widehat{\vy}$
\end{algorithmic}
\end{minipage}

\end{algorithm}

\subsection{Runtime and Scalability}
\label{app:scaffold-runtime}

For $q=\lceil\delta m\rceil$, \sgt construction consists of a
spanning-forest backbone $\gF$ followed by selection of
$T=q-|\gE_{\gF}|$ additional edges. The complexities in
Table~\ref{tab:scaffold-variants} describe the selection stage; total
construction cost additionally includes backbone construction.

\paragraph{Measurement protocol.}
We benchmark seeded unweighted synthetic graphs on a dual AMD EPYC 7282
machine using up to eight CPU workers. Times are medians of three warm complete
calls and include input normalization, fresh backbone construction, clustering
where applicable, candidate scoring/selection, and sparse-output assembly;
graph loading, JIT warm-up, and GNN training are excluded.
\texttt{Sample} includes preprocessing over $R=8$ randomized forests and its
first draw; subsequent draws reuse the precomputed sampling weights.
\texttt{Batch} uses ten fixed BFS-based candidate clusters formed by expanding
from high-degree seed nodes. Each candidate edge is assigned to one cluster,
while supporting paths are evaluated on the shared current support and may
cross cluster boundaries. \texttt{Batch} uses $b=512$ and $r=64$, while
\texttt{Heap} uses $\kappa=16$ and $c=64$.
To stress construction, these benchmarks enable all three score terms,
$(\alpha,\beta_E,\beta_V)=(1,1,1)$. All budgets are above the
spanning-forest floor, and outputs are checked for exact budget and
connectivity.

\paragraph{Parallel scaling.}
Candidate clusters are fixed across worker counts. Workers parallelize supporting-path searches and forest-scoring operations, while dynamic support updates remain sequential. Table~\ref{tab:scaffold-scaling} reports complete construction time for $P\in\{1,4,8\}$ using the same graphs and algorithm settings as Table~\ref{tab:scaffold-variants}.
Note that once \texttt{Sample} pre-processing is cached, subsequent exactly budgeted draws take $0.055$\,s on the $200$K-node benchmark.

\begin{table}[!htbp]
\centering
\caption{Complete \sgt construction time at $\delta=0.2$.}
\label{tab:scaffold-scaling}
\scriptsize
\setlength{\tabcolsep}{5pt}
\renewcommand{\arraystretch}{1.05}

\begin{tabular}{@{}llrrrc@{}}
\toprule
& & \multicolumn{3}{c}{\textbf{Time (s)}} & \\
\cmidrule(lr){3-5}
\textbf{Variant} & \textbf{Graph $(n,m)$}
& $P{=}1$ & $P{=}4$ & $P{=}8$
& \textbf{Speedup} \\
\midrule

\sgte
& $(400,\,3.1\mathrm{K})$
& 1.63 & 1.01 & 0.78 & $2.08\times$ \\

\sgth
& $(4\mathrm{K},\,39.9\mathrm{K})$
& 25.09 & 25.02 & 25.16 & $1.00\times$ \\

\sgtb
& $(10\mathrm{K},\,250\mathrm{K})$
& 206.12 & 57.54 & 31.01 & $6.65\times$ \\

\sgtf
& $(200\mathrm{K},\,2.4\mathrm{M})$
& 3.53 & 2.78 & 2.30 & $1.54\times$ \\

\sgtsa
& $(200\mathrm{K},\,2.4\mathrm{M})$
& 33.77 & 11.81 & 9.02 & $3.74\times$ \\

\bottomrule
\end{tabular}

\vspace{2pt}
\parbox{\linewidth}{\scriptsize
Times include fresh backbone construction; \texttt{Sample} includes
preprocessing and its first draw. Speedup is $T_1/T_8$. Fixed randomness
produces the same selected support across worker counts.}
\end{table}

\paragraph{Backbone construction.}
Table~\ref{tab:scaffold-backbones} isolates backbone construction.
\texttt{MaxSF}/\texttt{MinSF} use Kruskal-style union--find,
\texttt{RandSF} variants randomize edge order, and \texttt{SPF} uses a BFS
forest. \texttt{GLSF} and \texttt{LLSF} provide lower-stretch alternatives,
with \texttt{LLSF} refining a \texttt{GLSF} initialization through local
swaps. Under the unit weights used here, \texttt{MaxSF} and \texttt{MinSF}
coincide given the same tie-breaking.

\begin{table}[!htbp]
\centering
\caption{Support-backbone construction time on unweighted graphs.}
\label{tab:scaffold-backbones}
\scriptsize
\setlength{\tabcolsep}{3.5pt}
\renewcommand{\arraystretch}{1.05}

\begin{tabular}{@{}llrr@{}}
\toprule
\textbf{Backbone}
& \textbf{Complexity}
& \multicolumn{1}{c}{\shortstack{$n=200\mathrm{K}$\\$m=3.12\mathrm{M}$}}
& \multicolumn{1}{c}{\shortstack{$n=400$\\$m=1.54\mathrm{K}$}} \\
\cmidrule(l){3-4}
& & \multicolumn{1}{c}{ms} & \multicolumn{1}{c}{ms} \\
\midrule

\texttt{Fast-MaxSF/MinSF}
& $\mathcal{O}(m\alpha_{\rm UF}(n)+B)$
& 22.72--23.74 & 0.035--0.036 \\

\texttt{MaxSF/MinSF}
& $\mathcal{O}(m\log m)$
& 22.25--22.41 & 0.035 \\

\texttt{Fast-RandSF}
& $\mathcal{O}(m\alpha_{\rm UF}(n))$
& 98.76 & 0.113 \\

\texttt{RandSF}
& $\mathcal{O}(m\alpha_{\rm UF}(n))$
& 162 & 0.234 \\

\texttt{SPF}
& $\mathcal{O}(m+n\log n)$
& 817 & 1.13 \\

\texttt{GLSF}
& $\mathcal{O}(nm\bar c)$
& --- & 12,313 \\

\texttt{LLSF}
& $C_{\rm init}+
   \mathcal{O}(Lm\bar\ell(n+m)\log n)$
& --- & 392,810$^\dagger$ \\

\bottomrule
\end{tabular}

\vspace{2pt}
\parbox{\linewidth}{\scriptsize
$B=256$ and $\alpha_{\rm UF}$ is the inverse Ackermann function;
$\bar c$ and $\bar\ell$ denote the corresponding average construction factors.
\texttt{LLSF} uses \texttt{GLSF} initialization with $L=10$.
$^\dagger$One completed \texttt{LLSF} run; other entries are three-run medians.
Dashes denote unmeasured cases.}
\end{table}

\paragraph{Selection complexity.}
Candidates sharing a source can reuse a shortest-path search. Let
\[
C_{\mathrm{sp}}=
\begin{cases}
\mathcal{O}(n+q), & \text{BFS},\\
\mathcal{O}((n+q)\log n), & \text{Dijkstra}.
\end{cases}
\]
The selection costs are
\[
\begin{aligned}
\sgte &:~ \mathcal{O}(TnC_{\mathrm{sp}}),\\
\sgth &:~ \mathcal{O}\!\left(
(n+T(\kappa+c))C_{\mathrm{sp}}+m\log m\right),\\
\sgtb &:~ \mathcal{O}\!\left(
\left\lceil T/r\right\rceil bC_{\mathrm{sp}}\right),\\
\sgtf &:~ \mathcal{O}((m+n)\log n),\\
\sgtsa &:~ \mathcal{O}(R(m+n)\log n)
\text{ preprocessing},\quad \mathcal{O}(m)\text{ per draw}.
\end{aligned}
\]
\texttt{Heap} and \texttt{Batch} reduce dynamic rescoring, whereas
\texttt{Fast} and \texttt{Sample} avoid repeated full rescoring.
Backbone construction is additional to these selection bounds and is included
in the measured runtimes above.

Detailed GNN training-time and GPU-memory measurements are reported in
Appendix~\ref{app:runtime-details}.

\subsection{Possible Extension: Localized or Partitioned Supports}
\label{app:localized-scaffold}

The current \sgt formulation uses a spanning-forest backbone so that every
omitted edge has a finite supporting path whenever the requested budget is
feasible. A possible extension is to apply the same scoring principle locally
within graph partitions or clusters, constructing a connected support for each
subgraph rather than one global spanning forest.

The main difficulty is the treatment of cross-partition edges. If the
endpoints of an omitted edge are disconnected in the current support, its
supporting-path dilation is infinite and the standard \sgt score is not
directly applicable. A localized variant would therefore require an additional
mechanism for cross-partition connectivity, such as retaining a sparse global
skeleton, explicitly preserving selected boundary edges, or scoring
cross-partition edges separately.

We leave the design and analysis of such localized \sgt constructions to
future work.

\FloatBarrier
\clearpage
\section{Appendix Related Work}
\label{app:extended-related}

Graph sparsification methods differ in the information they use and the properties they aim to preserve. 
Topology methods range from simple degree and similarity
heuristics~\citep{Hamann2016Structure,Voudigari2016Rank,Xu2007SCAN,
Satuluri2011Local,Liu2023DSpar} to classical constructions with explicit
structural objectives. Cut sparsification~\citep{Benczur2004Cut} preserves
cuts, spectral sparsification~\citep{Spielman2011Spectral,
spielman2008graph,Batson2012Twice} preserves Laplacian structure, and
graph spanners~\citep{Althofer1990Spanners} control distance distortion.
Semantic methods instead use additional information to identify useful sparse
structures. NeuralSparse~\citep{Zheng2020NeuralSparse},
SparseGAT~\citep{ye2021sparse}, PTDNet~\citep{Luo2021LearningToDrop},
SGCN~\citep{Li2020SGCN}, Unified-LTH~\citep{Chen2021UnifiedLottery},
AdaGLT~\citep{Zhang2024GraphLotteryAutomated}, and
SGS-GNN~\citep{Das2025SGSGNN} learn graph structure together with downstream
training signals, while MoG~\citep{Zhang2025MoG} learns task-dependent routing
over multiple sparsifier experts.

Sampling methods address a related but different scalability problem.
GraphSAGE~\citep{Hamilton2017GraphSAGE} samples neighborhoods,
FastGCN~\citep{Chen2018FastGCN} and LADIES~\citep{Zou2019LADIES} sample
layer-wise computation, GraphSAINT~\citep{Zeng2020GraphSAINT} samples training
subgraphs, and Cluster-GCN~\citep{Chiang2019ClusterGCN} trains on graph
partitions. These approaches primarily limit the computation performed in each
training iteration rather than construct a single global sparse support.
Other related graph-modification approaches include stochastic edge removal
with DropEdge~\citep{Rong2020DropEdge}, diffusion-based preprocessing with
GDC~\citep{gasteiger2019diffusion}, and learned graph structures such as
ProGNN~\citep{Jin2020ProGNN} and NodeFormer~\citep{Wu2022Nodeformer}.

Recent work continues to explore GNN sparsification for efficient training and inference
~\citep{akkas2025shapley,liguori2025spectral,
liao2024unifews,ding2025large},
alongside broader methods for analyzing and estimating graph
sparsification error~\citep{wang2025empirical}.
Graph/model co-sparsification and lottery-ticket methods
~\citep{Chen2021UnifiedLottery,Wang2022DualLottery,
Hui2023GraphSparsityMatters,Zhang2024GraphLotteryAutomated},
learned graph structure and neural sparsification
~\citep{Zheng2020NeuralSparse,Li2020SGCN,Luo2021LearningToDrop,
Franceschi2019LDS,Jin2020ProGNN,Liu2022UnsupervisedGSL,
Wu2022Nodeformer,Li2024GSLB}, and classical spectral, cut, and
distance-preserving constructions
~\citep{Benczur2004Cut,Spielman2011Spectral,spielman2008graph,
Batson2012Twice,Althofer1990Spanners} further illustrate the breadth
of graph-reduction approaches. Our experiments use representative topology,
semantic, and sampling methods rather than attempting an exhaustive comparison
across all graph-reduction paradigms.

\sgt is a topology method, but differs in how it evaluates an omitted edge:
its importance is determined by the path that must support that edge in the
retained graph. With dilation alone, this is closely related to
distance-preserving sparsification such as graph spanners. Jointly considering
dilation and congestion instead follows the support-theoretic view that even
short replacement paths can create severe bottlenecks when many omitted
connections share the same retained edges. To our knowledge, prior GNN
sparsifiers have not explicitly used this joint support-theoretic
dilation--congestion criterion to construct the graph used directly for message
passing. The different \sgt variants realize this criterion with different quality-cost trade-offs.
\section{Appendix Experimental Details}
\label{app:experimental}

\subsection{Datasets and Shared Training Settings}
\label{app:dataset-settings}

We evaluate on 19 node-classification datasets spanning homophilic, heterophilic, and large-scale regimes. All graphs are converted to undirected, self-loop-free graphs before sparsification. Table~\ref{tab:shortdatasetdescription} reports dataset statistics and evaluation metrics, and Table~\ref{tab:presets} gives the per-dataset TunedGNN~\citep{luo2024classic} configurations shared across comparable methods.

\begin{table*}[!htbp]
\centering
\caption{Dataset statistics for the 19 node-classification benchmarks.}
\label{tab:shortdatasetdescription}
\scriptsize
\setlength{\tabcolsep}{5pt}
\renewcommand{\arraystretch}{1.2}
\begin{tabular*}{\textwidth}{@{\extracolsep{\fill}}lrrrrrrc@{}}
\toprule
& \multicolumn{3}{c}{\textbf{Graph scale}}
& \multicolumn{1}{c}{\textbf{Structure}}
& \multicolumn{3}{c}{\textbf{Task}} \\
\cmidrule(lr){2-4}\cmidrule(lr){5-5}\cmidrule(l){6-8}
\textbf{Dataset} & $|\gV|$ & $|\gE|$ & $d$ &
$\gH_{\mathrm{adj}}$ & $C$ & $F$ & \textbf{Metric} \\
\midrule
\rowcolor[gray]{0.93}\multicolumn{8}{@{}l@{}}{\textit{Heterophilic small/medium}} \\
Roman-empire~\citep{platonov2023critical}    & 22.66K  & 32.93K  & 2.91  & -0.05 & 18 & 300   & Acc \\
\rowcolor[gray]{0.975}Squirrel~\citep{platonov2023critical}       & 2.22K   & 47.00K  & 42.28 & 0.01  & 5  & 2.09K & Acc \\
Minesweeper~\citep{platonov2023critical}     & 10.00K  & 39.40K  & 7.88  & 0.01  & 2  & 7     & AUC \\
\rowcolor[gray]{0.975}Questions~\citep{platonov2023critical}       & 48.92K  & 153.54K & 6.28  & 0.02  & 2  & 301   & AUC \\
Chameleon~\citep{platonov2023critical}       & 890     & 8.85K   & 19.90 & 0.03  & 5  & 2.33K & Acc \\
\rowcolor[gray]{0.975}Amazon-ratings~\citep{platonov2023critical}  & 24.49K  & 93.05K  & 7.60  & 0.14  & 5  & 300   & Acc \\
\addlinespace[1pt]
\rowcolor[gray]{0.93}\multicolumn{8}{@{}l@{}}{\textit{Homophilic small/medium}} \\
WikiCS~\citep{mernyei2020wiki}               & 11.70K  & 215.60K & 36.85 & 0.58  & 10 & 300   & Acc \\
\rowcolor[gray]{0.975}CiteSeer~\citep{sen2008collective}          & 3.33K   & 4.55K   & 2.74  & 0.67  & 6  & 3.70K & Acc \\
Amazon-Computer~\citep{shchur2018pitfalls}  & 13.75K  & 245.86K & 35.76 & 0.68  & 10 & 767   & Acc \\
\rowcolor[gray]{0.975}PubMed~\citep{sen2008collective}            & 19.72K  & 44.32K  & 4.50  & 0.69  & 3  & 500   & Acc \\
Cora~\citep{sen2008collective}              & 2.71K   & 5.28K   & 3.90  & 0.77  & 7  & 1.43K & Acc \\
\rowcolor[gray]{0.975}Coauthor-CS~\citep{shchur2018pitfalls}      & 18.33K  & 81.89K  & 8.93  & 0.78  & 15 & 6.81K & Acc \\
Amazon-Photo~\citep{shchur2018pitfalls}     & 7.65K   & 119.08K & 31.13 & 0.79  & 8  & 745   & Acc \\
\rowcolor[gray]{0.975}Coauthor-Physics~\citep{shchur2018pitfalls} & 34.49K  & 247.96K & 14.38 & 0.87  & 5  & 8.42K & Acc \\
\addlinespace[1pt]
\rowcolor[gray]{0.93}\multicolumn{8}{@{}l@{}}{\textit{Large-scale}} \\
ogbn-proteins~\citep{Hu2020OGB}             & 132.53K & 39.56M  & 597.00 & 0.05 & 112$^{\dagger}$ & 8   & AUC \\
\rowcolor[gray]{0.975}Pokec~\citep{lim2021large}                   & 1.63M   & 22.30M  & 27.32  & 0.42 & 2             & 65  & Acc \\
ogbn-arxiv~\citep{Hu2020OGB}                & 169.34K & 1.17M   & 13.77  & 0.42 & 40            & 128 & Acc \\
\rowcolor[gray]{0.975}ogbn-products~\citep{Hu2020OGB}            & 2.45M   & 61.86M  & 50.52  & 0.46 & 47            & 100 & Acc \\
Reddit~\citep{Hamilton2017GraphSAGE}         & 232.97K & 57.31M  & 492.00 & 0.74 & 41            & 602 & Acc \\
\bottomrule
\end{tabular*}
\vspace{2pt}
\begin{minipage}{\textwidth}
\footnotesize\emph{Note.} $|\gV|$ and $|\gE|$ are the numbers of nodes and
undirected edges; $d$ is average degree; $\gH_{\mathrm{adj}}$ is adjusted
homophily~\citep{platonov2022characterizing}; $C$ and $F$ are the number of
classes (or tasks) and feature dimension. Small/medium graphs use
$\delta\in\{0.3,0.5,0.7\}$ and large-scale graphs use
$\delta\in\{0.1,0.3,0.5\}$. Graphs are made undirected and self-loop-free before
sparsification. $^\dagger$\textsc{ogbn-proteins} has $112$ binary prediction
tasks.
\end{minipage}
\end{table*}



\subsection{Baseline Training Settings}
\label{app:competing-method-settings}

\paragraph{Budget matching and evaluation protocol.}
All ranked fixed-budget sparsifiers use the same target
$q=\lceil\delta m\rceil$ and, whenever applicable, the same data split, GNN
architecture, optimization settings, epoch budget, and model-selection rule
from Table~\ref{tab:presets}; only the graph support changes. Semantic methods
retain their native optimization procedures but are matched to $\delta$,
while sampling methods and reference anchors are reported separately.

\sgt-1 uses one $q$-edge support throughout training and inference and is
included in the matched-budget ranking. \sgt-$K$ trains on $q$ edges at a time
but refreshes supports and infers on their validation-selected union, which may
exceed $q$; it is therefore reported separately. Test labels are never used for
support or checkpoint selection.

Method objectives are summarized in Appendix~\ref{app:extended-related};
implementation-specific deviations and dataset-specific epoch budgets $T$ are
given below and in Table~\ref{tab:presets}.

\begin{table*}[!t]
\centering
\caption{Per-dataset GCN training configurations shared by all methods.}
\label{tab:presets}
\scriptsize
\setlength{\tabcolsep}{3.5pt}
\renewcommand{\arraystretch}{1.04}
\begin{tabular*}{\textwidth}{@{\extracolsep{\fill}}lrrrrlrllc@{}}
\toprule
& \multicolumn{2}{c}{\textbf{Protocol}}
& \multicolumn{3}{c}{\textbf{Architecture}}
& \multicolumn{3}{c}{\textbf{Optimization}}
& \multicolumn{1}{c}{\textbf{Evaluation}} \\
\cmidrule(lr){2-3}\cmidrule(lr){4-6}\cmidrule(lr){7-9}\cmidrule(l){10-10}
\textbf{Dataset} & \textbf{Epochs} & \textbf{Runs} & $L$ & $H$ &
\textbf{Norm./Res.} & \textbf{Dropout} & \textbf{LR} & \textbf{WD} & \textbf{Metric} \\
\midrule
\rowcolor[gray]{0.93}\multicolumn{10}{@{}l@{}}{\textit{Heterophilic small/medium}} \\
Roman-empire     & 2500 & 3  & 9  & 512 & bn, res$^{\ast}$ & 0.50 & 1e-3 & 0    & Acc \\
\rowcolor[gray]{0.975}Squirrel       & 500  & 10 & 4  & 256 & bn, res       & 0.70 & 1e-2 & 5e-4 & Acc \\
Minesweeper      & 2000 & 3  & 12 & 64  & bn, res           & 0.20 & 1e-2 & 0    & AUC \\
\rowcolor[gray]{0.975}Questions      & 1500 & 3  & 10 & 512 & res$^{\ast}$ & 0.30 & 3e-5 & 0    & AUC \\
Chameleon        & 200  & 10 & 5  & 512 & --                & 0.20 & 5e-3 & 1e-3 & Acc \\
\rowcolor[gray]{0.975}Amazon-ratings & 2500 & 3  & 4  & 512 & bn, res       & 0.50 & 1e-3 & 0    & Acc \\
\addlinespace[1pt]
\rowcolor[gray]{0.93}\multicolumn{10}{@{}l@{}}{\textit{Homophilic small/medium}} \\
WikiCS           & 1000 & 3  & 3  & 256 & ln                & 0.50 & 1e-3 & 0    & Acc \\
\rowcolor[gray]{0.975}CiteSeer       & 500  & 5  & 2  & 512 & --            & 0.50 & 1e-3 & 1e-2 & Acc \\
Amazon-Computer  & 1000 & 3  & 3  & 512 & ln                & 0.50 & 1e-3 & 5e-5 & Acc \\
\rowcolor[gray]{0.975}PubMed         & 500  & 5  & 2  & 256 & --            & 0.70 & 5e-3 & 5e-4 & Acc \\
Cora             & 500  & 5  & 3  & 512 & --                & 0.70 & 1e-3 & 5e-4 & Acc \\
\rowcolor[gray]{0.975}Coauthor-CS    & 1500 & 3  & 2  & 512 & ln, res       & 0.30 & 1e-3 & 5e-4 & Acc \\
Amazon-Photo     & 1000 & 3  & 6  & 256 & ln, res           & 0.50 & 1e-3 & 5e-5 & Acc \\
\rowcolor[gray]{0.975}Coauthor-Physics & 1500 & 3 & 2 & 64 & ln, res       & 0.30 & 1e-3 & 5e-4 & Acc \\
\addlinespace[1pt]
\rowcolor[gray]{0.93}\multicolumn{10}{@{}l@{}}{\textit{Large-scale}} \\
ogbn-proteins    & 100  & 2 & 3 & 80  & --      & 0.25 & 1e-2 & 0    & AUC \\
\rowcolor[gray]{0.975}Pokec          & 2000 & 2 & 7 & 256 & bn, res & 0.20 & 5e-4 & 0    & Acc \\
ogbn-arxiv       & 2000 & 2 & 5 & 512 & bn, res          & 0.50 & 5e-4 & 5e-4 & Acc \\
\rowcolor[gray]{0.975}ogbn-products  & 300  & 2 & 5 & 200 & ln      & 0.50 & 3e-3 & 0    & Acc \\
Reddit           & 300  & 2 & 5 & 200 & ln               & 0.50 & 3e-3 & 0    & Acc \\
\bottomrule
\end{tabular*}
\vspace{2pt}
\begin{minipage}{\textwidth}
\footnotesize\emph{Note.} $^{\ast}$ A linear input projection precedes the
backbone.
\end{minipage}
\end{table*}

\begingroup
\setlength{\intextsep}{6pt}

\begin{table}[!t]
\centering
\caption{
Training configurations of evaluated baselines.
$T$ is the dataset-specific epoch budget in Table~\ref{tab:presets};
$\delta$ is the retained-edge ratio.
}
\label{tab:competing-method-settings}

\fontsize{8}{9.2}\selectfont
\setlength{\tabcolsep}{3pt}
\renewcommand{\arraystretch}{1.05}

\newcommand{\CMgroup}[1]{%
  \makebox[12pt]{\rotatebox[origin=c]{90}{\bfseries #1}}}

\newcommand{\CMblock}[1]{%
  \begin{tabularx}{\dimexpr\textwidth-17pt\relax}{@{}
    >{\raggedright\arraybackslash}p{105pt}
    >{\raggedright\arraybackslash}p{75pt}
    >{\raggedright\arraybackslash}X@{}}
  #1
  \end{tabularx}%
}

\begin{tabular}{@{}c@{\hspace{1pt}}c@{}}
\toprule
&
\CMblock{
\textbf{Method(s)} &
\textbf{GNN / budget} &
\textbf{Training and evaluation}
\\}
\\
\midrule

\CMgroup{Anchors}
&
\CMblock{%

No Graph
& MLP; $T$
& Feature-only prediction.
\\[1pt]

Spanning Forest
& Shared GNN; $T$
& One spanning-forest support for training and inference.
\\[1pt]

Full Graph
& TunedGNN; $T$
& Original graph without sparsification.
\\
}
\\

\midrule

\CMgroup{Topology}
&
\CMblock{%

Random, Local Degree, Rank Degree, Forest Fire, SCAN
& Shared GNN; $\delta$; $T$
& One support per run and budget, reused throughout training.
\\[1pt]

L-Spar, G-Spar, L-Sim, Spectral (ER), $t$-Spanner, DSpar
& Shared GNN; $\delta$; $T$
& Each sparsifier is evaluated independently using the shared training preset.
\\
}
\\

\midrule

\CMgroup{Semantic}
&
\CMblock{%

Unified-LTH
& Shared GNN preset
& Two 200-epoch mask-search rounds followed by $T$ ticket-training epochs;
graph retention $\delta$, no weight pruning.
\\[1pt]

AdaGLT
& Shared GNN preset
& 400 mask-search epochs followed by $T$ retraining epochs;
adjacency sparsity $1-\delta$, no weight pruning.
\\[1pt]

MoG
& Shared GNN; $\delta$; $T$
& Three sparsifier experts at retention $[\delta,\delta,\delta]$, trained
jointly with the downstream GNN.
\\
}
\\

\midrule

\CMgroup{Sampling}
&
\CMblock{%

Tuned-GraphSAGE
& Tuned-GraphSAGE; $T$
& Dataset-specific neighborhood sampling on the original graph.
\\[1pt]

Tuned-GraphSAINT
& Shared GNN; $T$
& Random-walk subgraph sampling with dataset-specific settings.
\\
}
\\

\midrule

\CMgroup{Scaffold}
&
\CMblock{%

Scaffold-Greedy
& Shared GNN; $\delta$; $T$
& Full dynamic rescoring after every insertion.
\\[1pt]

Scaffold-Heap
& Shared GNN; $\delta$; $T$
& Cached priorities with selective refresh.
\\[1pt]

Scaffold-Batch
& Shared GNN; $\delta$; $T$
& Scores sampled batches and inserts top-$r$ edges.
\\[1pt]

Scaffold-Fast
& Shared GNN; $\delta$; $T$
& Scores once on the initial spanning-forest backbone.
\\[1pt]

Scaffold-Sample
& Shared GNN; $\delta$; $T$
& Reuses precomputed weights to draw exactly budgeted supports.
\\[1pt]

Scaffold settings
& $(\alpha,\beta_E,\beta_V)=(1,1,0)$
& Refreshed supports may be generated asynchronously.
\\[1pt]

Scaffold-$1$
& Shared GNN; $\delta$; $T$
& One $q$-edge support for training and inference.
\\[1pt]

Scaffold-$K$
& Shared GNN; $\delta$ per support; $T$
& Refreshes every $\rho$ epochs, retains top-$K$ validation supports
($K\in\{3,5,10\}$), and infers once on their union.
\\
}
\\

\bottomrule
\end{tabular}
\end{table}

\endgroup

\textbf{Large-graph adaptations.}
Unified-LTH and AdaGLT use sparse-neighborhood minibatching on large graphs,
\textsc{Coauthor-Physics}, and \textsc{Questions}; AdaGLT additionally uses this route on \textsc{Roman-Empire}. These runs use the dataset-specific $T$ epochs rather than the native mask-search schedules above. MoG uses random-node partitions on large graphs, with train/evaluation partition counts of $1/1$ for \textsc{ogbn-arxiv}, $10/5$ for \textsc{ogbn-proteins}, $3/3$ for \textsc{Pokec}, and $10/10$ for \textsc{Reddit} and \textsc{ogbn-products}.

\textbf{Sampling settings.}
Tuned-GraphSAINT uses random-walk sampling. Small/medium graphs use batch size 500, walk length 2, three steps per epoch, and coverage 100 with normalization.
Large graphs use walk length 4 and 30 steps per epoch, with batch size 6000
(20,000 for \textsc{ogbn-products} and \textsc{Pokec}) and coverage 0 without
normalization. For Tuned-GraphSAGE, we follow the same settings specified by the authors in ~\cite{luo2024classic}.

\section{Complete numerical results}
\label{app:complete-numerical-results}
\subsection{Results on Homophilic Graphs}
Table~\ref{tab:appendix-full-results-homophilic} shows detailed numerical comparison of \sgt with baselines $\delta\in\{0.3,0.5,0.7\}$.
\begin{table*}[!htbp]
\centering
\caption{Homophilic graphs: complete mean $\pm$ standard deviation results (\%) over eight datasets.}
\label{tab:appendix-full-results-homophilic}
\setlength{\tabcolsep}{1.15pt}
\renewcommand{\arraystretch}{0.80}
\fontsize{7pt}{7.6pt}\selectfont
 &\cellcolor[rgb]{0.88,0.95,0.88}\textbf{87.37} \\
\bottomrule
\end{tabular*}
\vspace{2pt}
\begin{minipage}{0.99\textwidth}
\scriptsize Green, blue, and amber mark the top three \emph{matched-budget} methods; Scaf.-1 is included (one fixed support at $\delta$), while Scaf.-$K$ is excluded because its union exceeds the budget and is shown below the rule with its best variant bolded. The ratio-independent block is ranked separately. GM is the geometric mean. $^{\dagger}$Ratio below the dataset's spanning-forest floor $(|\gV|-c)/|\gE|$: no connected support exists. $^{\ddagger}$Support rebuilt asynchronously. $^{\star}$Per-dataset max over $\delta\in[0.1,0.9]$; parentheses give the lowest $\delta$ attaining it.
\end{minipage}
\end{table*}

\subsection{Results on Heterophilic Graphs}
Table~\ref{tab:appendix-full-results-heterophilic} shows detailed numerical comparison of \sgt with baselines $\delta\in\{0.3,0.5,0.7\}$. 
\begin{table*}[!htbp]
\centering
\caption{Heterophilic graphs: complete mean $\pm$ standard deviation results (\%) over six datasets.}
\label{tab:appendix-full-results-heterophilic}
\setlength{\tabcolsep}{1.15pt}
\renewcommand{\arraystretch}{0.88}
\fontsize{7pt}{7.6pt}\selectfont
 &\cellcolor[rgb]{0.89,0.93,0.98}\underline{64.96} \\
\bottomrule
\end{tabular*}
\vspace{2pt}
\begin{minipage}{0.99\textwidth}
\scriptsize Green, blue, and amber mark the top three \emph{matched-budget} methods; Scaf.-1 is included (one fixed support at $\delta$), while Scaf.-$K$ is excluded because its union exceeds the budget and is shown below the rule with its best variant bolded. The ratio-independent block is ranked separately. GM: geometric mean. $^{\dagger}$Below the spanning-forest floor, no connected support exists. $^{\ddagger}$Support rebuilt asynchronously. $^{\star}$Per-dataset max over $\delta\in[0.1,0.9]$; parentheses give the lowest $\delta$ attaining it.
\end{minipage}
\end{table*}

\subsection{Results on Large-scale Graphs}
Table~\ref{tab:appendix-full-results-large} shows detailed numerical comparison of \sgt with baselines $\delta\in\{0.1,0.3,0.5\}$ .
\begin{table*}[!htbp]
\centering
\caption{Large-scale graphs: complete mean $\pm$ standard deviation results (\%) over five datasets.}
\label{tab:appendix-full-results-large}
\setlength{\tabcolsep}{1.15pt}
\renewcommand{\arraystretch}{0.88}
\fontsize{7pt}{7.6pt}\selectfont
 &\cellcolor[rgb]{0.88,0.95,0.88}\textbf{82.65} \\
\bottomrule
\end{tabular*}
\vspace{2pt}
\begin{minipage}{0.99\textwidth}
\scriptsize Green, blue, and amber mark the top three \emph{matched-budget} methods; Scaf.-1 is included (one fixed support at $\delta$), while Scaf.-$K$ is excluded because its union exceeds the budget and is shown below the rule with its best variant bolded. The ratio-independent block is ranked separately. GM: geometric mean. $^{\dagger}$Below the spanning-forest floor, no connected support exists. $^{\ddagger}$Support rebuilt asynchronously. $^{\star}$Per-dataset max over $\delta\in[0.1,0.9]$; parentheses give the lowest $\delta$ attaining it.
\end{minipage}
\end{table*}

\subsection{Case study of performance of different datasets}
\label{app:casestudy}
The results in Tables~\ref{tab:appendix-full-results-homophilic},\ref{tab:appendix-full-results-heterophilic},\ref{tab:appendix-full-results-large}
show cases where edge budget alone does not explain performance.

\paragraph{Budgets below the spanning-forest floor.}
For sparse graphs such as \textsc{roman-empire}, \textsc{citeseer}, and
\textsc{cora}, some evaluated retention ratios fall below the
spanning-forest floor. At these budgets, no method can
simultaneously satisfy the edge budget and preserve all connected components, so the resulting degradation should primarily be interpreted as a consequence
of an infeasible connectivity budget rather than edge-selection quality.

\paragraph{Local neighborhood preservation on \textsc{minesweeper}.}
\textsc{minesweeper} is a notable exception to the aggregate trend:
local similarity-based sparsifiers outperform \sgt-1 at $\delta=0.7$.
This suggests that, for some tasks, preserving the distribution of retained
edges within individual neighborhoods can be more important than optimizing
global support paths. 

\paragraph{Backbone choice on \textsc{chameleon} and \textsc{squirrel}.}
These small heterophilic graphs are the two cases where \sgt-1 uses the
low-stretch SLSF backbone. Here \sgt-1, Spectral, and MoG obtain similar
performance, while several degree- or similarity-based sparsifiers are
noticeably weaker. This is consistent with the importance of preserving
short, well-distributed communication paths on these graphs.

\paragraph{\sgt-$K$ is not uniformly better than \sgt-1.}
Refreshing supports generally improves performance, but the gain is dataset dependent. In particular, \sgt-$K$ slightly underperforms \sgt-1 on \textsc{ogbn-proteins}, \textsc{reddit}, and \textsc{ogbn-products} at the matched target ratios, while improving substantially on datasets such as \textsc{ogbn-arxiv} and \textsc{pokec}. Because \sgt-$K$ trains on individual $q$-edge supports but evaluates on their union, its benefit depends on whether the refreshed supports provide complementary useful structure.  Also, irregular swap of support graph can be harmful. The effects of refresh frequency and training versus inference topology are examined further in Appendix~\ref{app:refresh-frequency}.

\subsection{Realized Inference Union Size of \sgt-$K$}
\label{app:inference_size}

Table~\ref{tab:inference-budget} shows that the realized \sgt-$K$ union varies substantially across datasets and variants, ranging from nearly a single $q$-edge support to almost the full graph. This variability motivates reporting \sgt-$K$ separately from the matched-budget \sgt-1 results.

\begin{table}[!htbp]
\centering
\caption{Realized union size for Scaffold-$K$ at the targets $\delta_d$.}
\label{tab:inference-budget}


\definecolor{IBBatch}{HTML}{000000}
\definecolor{IBFast}{HTML}{000000}
\definecolor{IBSample}{HTML}{000000}

\setlength{\tabcolsep}{2.2pt}
\renewcommand{\arraystretch}{1.00}
\fontsize{7.2pt}{8.4pt}\selectfont

\begin{tabular*}{\textwidth}{@{\extracolsep{\fill}}lc*{3}{rrr}@{}}
\toprule
\textbf{Dataset} & $\delta_d$
& \multicolumn{3}{c}{\textcolor{IBBatch}{\textbf{Scaffold-Batch}}}
& \multicolumn{3}{c}{\textcolor{IBFast}{\textbf{Scaffold-Fast}}}
& \multicolumn{3}{c}{\textcolor{IBSample}{\textbf{Scaffold-Sample}}} \\
\cmidrule(lr){3-5}
\cmidrule(lr){6-8}
\cmidrule(l){9-11}
& & $K$ & $\delta_{\cup}$ & $|E_{\cup}|/q$
& $K$ & $\delta_{\cup}$ & $|E_{\cup}|/q$
& $K$ & $\delta_{\cup}$ & $|E_{\cup}|/q$ \\
\midrule

Cora
& 0.7
& 10 & 0.998 & 1.43
& 10 & 0.870 & 1.24
& 10 & 0.844 & 1.21 \\

CiteSeer
& 0.7
& 10 & 0.981 & 1.40
& 10 & 0.910 & 1.30
& 10 & 0.752 & 1.07 \\

Amazon Photo
& 0.3
& 10 & 0.905 & 3.02
& 10 & 0.808 & 2.69
& 10 & 0.695 & 2.32 \\

WikiCS
& 0.3
& 10 & 0.329 & 1.10
& 10 & 0.300 & 1.00
& 10 & 0.821 & 2.74 \\

Amazon Computers
& 0.3
& 10 & 0.442 & 1.47
& 10 & 0.436 & 1.45
& 5 & 0.570 & 1.90 \\

Coauthor CS
& 0.5
& 10 & 0.989 & 1.98
& 10 & 0.651 & 1.30
& 10 & 0.720 & 1.44 \\

PubMed
& 0.5
& 10 & 0.889 & 1.78
& 10 & 0.500 & 1.00
& 10 & 0.500 & 1.00 \\

Coauthor Physics
& 0.3
& 10 & 0.533 & 1.78
& 10 & 0.498 & 1.66
& 5 & 0.576 & 1.92 \\

Chameleon
& 0.3
& 5 & 0.725 & 2.42
& 10 & 0.300 & 1.00
& 5 & 0.453 & 1.51 \\

Squirrel
& 0.3
& 5 & 0.300 & 1.00
& 10 & 0.300 & 1.00
& 5 & 0.549 & 1.83 \\

Minesweeper
& 0.7
& 10 & 1.000 & 1.43
& 10 & 0.700 & 1.00
& 10 & 0.760 & 1.09 \\

Roman Empire
& 0.7
& 10 & 0.858 & 1.23
& 10 & 0.992 & 1.42
& 10 & 0.761 & 1.09 \\

Amazon Ratings
& 0.7
& 10 & 0.999 & 1.43
& 10 & 0.975 & 1.39
& 10 & 0.898 & 1.28 \\

Questions
& 0.7
& 10 & 0.963 & 1.38
& 10 & 0.973 & 1.39
& 10 & 0.744 & 1.06 \\

ogbn-proteins
& 0.3
& 10 & 0.305 & 1.02
& 10 & 0.324 & 1.08
& 10 & 0.987 & 3.29 \\

ogbn-arxiv
& 0.3
& 10 & 0.583 & 1.94
& 10 & 0.541 & 1.80
& 10 & 0.802 & 2.67 \\

Reddit
& 0.3
& 10 & 0.325 & 1.08
& 10 & 0.325 & 1.08
& 10 & 0.994 & 3.31 \\

Pokec
& 0.3
& 10 & 0.479 & 1.60
& 10 & 0.523 & 1.74
& 10 & 0.840 & 2.80 \\

ogbn-products
& 0.3
& 10 & 0.626 & 2.09
& 10 & 0.585 & 1.95
& 10 & 0.754 & 2.51 \\

\bottomrule
\end{tabular*}

\vspace{2pt}
\begin{minipage}{\textwidth}
\fontsize{6.5pt}{7.5pt}\selectfont
$q=\lceil\delta_d|E|\rceil$ and
$\delta_{\cup}=|E_{\cup}|/|E|$.
\end{minipage}

\end{table}

\FloatBarrier
\clearpage
\section{Runtime and Memory Details with GNN Benchmarks}
\label{app:runtime-details}
The runtime of \sgt variants on synthetic graphs is provided in Appendix~\ref{app:scaffold-runtime}. In this section, we compare them with other sparsifiers and methods when used with GNNs.

\subsection{Sparsification time and Scaling of \sgt}
\label{app:runtime-profiles}

\paragraph{CPU scaling of \sgt support generation.}
Table~\ref{tab:runtime-workers} shows how the scalable \sgt variants benefit from additional CPU workers on \textsc{Reddit} and \textsc{ogbn-products}. 
\texttt{Fast} and \texttt{Batch} remain relatively expensive on these large graphs, motivating asynchronous construction when supports are refreshed during training. 
\texttt{Sample} has a larger one-time startup cost, but subsequent supports are generated much faster by reusing its precomputed sampling state.

\begin{table}[!htbp]
\centering
\caption{CPU scaling of \sgt support generation at $\delta=0.3$.}
\label{tab:runtime-workers}
\scriptsize
\setlength{\tabcolsep}{2pt}
\renewcommand{\arraystretch}{1.05}
\begin{tabular*}{\linewidth}{@{}l@{\extracolsep{\fill}}rrrrrr@{}}
\toprule
\textbf{Variant / stage} & \multicolumn{6}{c}{\textbf{CPU workers}} \\
\cmidrule(lr){2-7}
 & 1 & 2 & 4 & 8 & 16 & 32 \\
\midrule
\multicolumn{7}{@{}l@{}}{\textbf{\textsc{Reddit}}} \\
\textbf{Scaffold-Fast} & 498.38 & 284.13 & 213.08 & 196.54 & 201.54 & 182.08 \\
\textbf{Scaffold-Batch} & 630.43 & 355.01& 265.93 & 242.33 & 232.81 & 216.50 \\
\textbf{Sample}: first support & 504.96 & 358.56 & 232.61 & 139.49 & 124.69 & 87.25 \\
\textbf{Sample}: subsequent support & 1.40 & 1.27 & 1.22 & 1.20 & 1.18 & 1.16 \\
\midrule
\multicolumn{7}{@{}l@{}}{\textbf{\textsc{ogbn-products}}} \\
\textbf{Scaffold-Fast} & 379.32 & 230.34 & 175.69 & 156.10 & 115.04 & 109.39 \\
\textbf{Scaffold-Batch} & 539.84 & 333.42 & 247.68 & 202.80 & 159.35 & 148.05 \\
\textbf{Sample}: first support & 1109.74 & 615.16 & 345.92 & 198.35 & 125.22 & 90.23 \\
\textbf{Sample}: subsequent support & 1.51 & 1.35 & 1.29 & 1.25 & 1.21 & 1.21 \\
\bottomrule
\end{tabular*}

\par\vspace{2pt}
\begin{minipage}{\linewidth}
\fontsize{7pt}{8pt}\selectfont
Medians over 3 completed sessions. Sample subsequent-support time reuses cached weights and sampling state. Dataset loading and warm-up are excluded.
\end{minipage}
\end{table}

\paragraph{Startup and reuse cost of \sgtsa.}
Table~\ref{tab:runtime-sample-stages} further separates the one-time initialization of \sgtsa from repeated support generation. 
Most startup cost comes from constructing the forests and sampling weights, whereas loading saved weights and drawing subsequent supports is substantially cheaper. 
This reuse makes \sgtsa particularly suitable for \sgt-$K$, where sparse supports are refreshed repeatedly during training.

\begin{table}[!htbp]
\centering
\caption{\sgtsa startup and reuse costs at $\delta=0.3$.}
\label{tab:runtime-sample-stages}
\scriptsize
\setlength{\tabcolsep}{2pt}
\renewcommand{\arraystretch}{1.05}
\begin{tabular*}{\linewidth}{@{}l@{\extracolsep{\fill}}rrrrrr@{}}
\toprule
\textbf{Stage} & \multicolumn{6}{c}{\textbf{CPU workers}} \\
\cmidrule(lr){2-7}
 & 1 & 2 & 4 & 8 & 16 & 32 \\
\midrule
\multicolumn{7}{@{}l@{}}{\textbf{\textsc{Reddit}}} \\
First support (total) & 504.96 & 358.56 & 232.61 & 139.49 & 124.69 & 87.25 \\
\quad Forests + weights & 472.57 & 320.16 & 195.16 & 107.40 & 87.36 & 54.80 \\
\quad Save weights & 25.70 & 30.71 & 30.91 & 25.53 & 30.79 & 25.71 \\
\quad Other first-support work & 6.73 & 6.60 & 6.57 & 6.53 & 6.54 & 6.53 \\
Startup from saved weights & 37.02 & 47.40 & 36.21 & 35.74 & 41.83 & 36.11 \\
Subsequent support & 1.40 & 1.27 & 1.22 & 1.20 & 1.18 & 1.16 \\
\midrule
\multicolumn{7}{@{}l@{}}{\textbf{\textsc{ogbn-products}}} \\
First support (total) & 1109.74 & 615.16 & 345.92 & 198.35 & 125.22 & 90.23 \\
\quad Forests + weights & 1097.54 & 600.69 & 332.41 & 185.97 & 113.01 & 75.17 \\
\quad Save weights & 2.09 & 4.35 & 2.49 & 2.35 & 2.13 & 4.72 \\
\quad Other first-support work & 10.29 & 10.14 & 10.09 & 10.04 & 10.01 & 10.02 \\
Startup from saved weights & 14.90 & 15.04 & 14.82 & 15.31 & 14.40 & 15.07 \\
Subsequent support & 1.51 & 1.35 & 1.29 & 1.25 & 1.21 & 1.21 \\
\bottomrule
\end{tabular*}

\par\vspace{2pt}
\begin{minipage}{\linewidth}
\fontsize{7pt}{8pt}\selectfont
Medians over 3 sessions. Indented rows decompose the first-support cost; subsequent supports reuse the existing sampler. Independently computed medians may not sum exactly.
\end{minipage}
\end{table}

\subsection{Training profile breakdown of \sgtf-1 and Full-GCN.}
Table~\ref{tab:runtime-profiles} (data behind Figure~\ref{fig:runtime_profiling}) shows how much time \sgtf-1 and Full-Graph (TunedGNN~\citep{luo2024classic}) spend on sparsification and training. We profile two large datasets, \textsc{Reddit} and \textsc{ogbn-products}, each run for 300 epochs on an H100 80GB GPU with 8 CPU workers.
We further break down training time to show the cost of \texttt{Gather-Scatter} aggregation with the \texttt{COO} format and \texttt{SpMM} with the \texttt{CSR/CSC} format.
The benefits of \sgt are greatest on large graphs, at low retention ratios, and over long runs, where per-epoch savings amortize the sparsification overhead.


\begin{table}[!htbp]
\centering
\caption{\texttt{Gather-Scatter/SpMM} values for Figure~\ref{fig:runtime_profiling}. Times are seconds over 300 epochs; GPU is peak allocated memory (GiB).}
\label{tab:runtime-profiles}
\scriptsize
\setlength{\tabcolsep}{1.5pt}
\renewcommand{\arraystretch}{1.02}
\begin{tabular*}{\linewidth}{@{}l@{\extracolsep{\fill}}rrrrrrrrrr@{}}
\toprule
\textbf{Variant} & $\delta$ & $S$ & \multicolumn{4}{c}{\textbf{Gather-Scatter}} & \multicolumn{4}{c}{\textbf{SpMM}} \\
\cmidrule(lr){4-7}\cmidrule(lr){8-11}
 &  &  & Agg. & Rest & Total & GPU & Agg. & Rest & Total & GPU \\
\midrule
\rowcolor[gray]{0.90}
\multicolumn{11}{@{}l@{}}{\textbf{\textsc{reddit}}} \\
\textbf{Scaffold-Fast-1} & 0.1 & 133.9 & 37.5 & 70.8 & 242.3 & 9.98 & 12.6 & 64.6 & 211.2 & 3.97 \\
\textbf{Scaffold-Fast-1} & 0.3 & 174.8 & 138.4 & 132.1 & 445.3 & 24.52 & 27.2 & 106.9 & 308.9 & 5.54 \\
\textbf{Scaffold-Fast-1} & 0.5 & 223.5 & 236.7 & 190.3 & 650.5 & 39.09 & 41.5 & 146.8 & 411.7 & 7.35 \\
\cmidrule(lr){1-11}
Full-GCN & 1.0 & 0.0 & 553.0 & 336.2 & 889.3 & 75.43 & 77.0 & 249.8 & 326.8 & 11.84 \\

\rowcolor[gray]{0.90}
\multicolumn{11}{@{}l@{}}{\textbf{\textsc{ogbn-products}}} \\
\textbf{Scaffold-Fast-1} & 0.1 & 122.2 & 57.2 & 326.9 & 506.2 & 25.18 & 74.7 & 313.3 & 510.2 & 25.17 \\
\textbf{Scaffold-Fast-1} & 0.3 & 179.7 & 133.4 & 402.3 & 715.4 & 35.21 & 94.2 & 369.0 & 642.8 & 26.17 \\
\textbf{Scaffold-Fast-1} & 0.5 & 224.4 & 274.2 & 475.6 & 974.1 & 45.23 & 116.4 & 424.5 & 765.3 & 27.18 \\
\cmidrule(lr){1-11}
Full-GCN & 1.0 & 0.0 & 591.4 & 655.7 & 1247.0 & 70.33 & 180.1 & 560.4 & 740.5 & 29.69 \\
\bottomrule
\end{tabular*}

\par\vspace{2pt}
\begin{minipage}{\linewidth}
\scriptsize
$S$: support generation; Agg./Rest: aggregation/other training work;
Total $=S+\mathrm{Agg.}+\mathrm{Rest}$.
Fast is the Scaffold-Fast variant. Each entry is one profiling run.
\end{minipage}
\end{table}

\subsection{Sparsification and end-to-end runtime}
\label{app:runtime-construction}

\paragraph{Sparsification time.}
Table~\ref{tab:runtime-sparsification} compares standalone sparsification time
at the same retention ratio, $\delta=0.3$, across five large-scale graphs.
Among the Scaffold variants, \sgtsa is the fastest on all five graphs,
while \sgtf and \sgtb incur higher cost from additional support-specific
computation. 
The benefit of \sgtsa really comes when we need to resparsify at subsequent epochs during training.

\begin{table}[!htbp]
\centering
\caption{Sparsification time (s) at $\delta=0.3$ on five large-scale graphs.}
\label{tab:runtime-sparsification}

\scriptsize
\setlength{\tabcolsep}{4pt}
\renewcommand{\arraystretch}{1.05}

\begin{tabular*}{\linewidth}{@{}l@{\extracolsep{\fill}}rrrrr@{}}
\toprule
\textbf{Method}
& \textbf{Reddit}
& \textbf{\shortstack{ogbn-\\products}}
& \textbf{\shortstack{ogbn-\\arxiv}}
& \textbf{\shortstack{ogbn-\\proteins}}
& \textbf{Pokec} \\
\midrule

Random
& 1.23
& 1.37
& 0.015
& 0.819
& 0.582 \\

Local degree
& 41.20
& 30.39
& 0.543
& 24.01
& 10.02 \\

Rank degree
& 249.83
& 69.20
& 4.05
& 169.86
& 27.50 \\

Forest Fire
& 2501.61
& ---
& 10.33
& 1414.05
& 692.14 \\

SCAN
& 78.66
& 47.57
& 0.649
& 40.52
& 13.32 \\

L-Spar
& 80.46
& 49.19
& 0.594
& 40.83
& 11.48 \\

G-Spar
& 78.30
& 47.63
& 0.654
& 40.20
& 13.43 \\

L-Sim
& 82.70
& 54.08
& 0.664
& 42.61
& 13.04 \\

Spectral$^{\ddagger}$
& 387.50
& 379.95
& 5.80
& 286.24
& 108.09 \\

Spanning tree$^{\dagger}$
& 36.21
& 17.25
& 0.225
& 2.10
& 3.32 \\

\midrule

\textbf{Scaffold-Fast}
& 72.00
& 153.05
& 6.30
& 96.12
& 80.09 \\

\textbf{Scaffold-Batch}
& 208.85
& 250.71
& 9.27
& 115.42
& 78.77 \\

\textbf{Scaffold-Sample}
& 48.83
& 133.89
& 6.06
& 24.09
& 11.68 \\

\bottomrule
\end{tabular*}

\par\vspace{2pt}
\begin{minipage}{\linewidth}\scriptsize
$\dagger$: spanning-tree reference with unmatched retention;
---: unavailable.
\sgtsa includes one-time weight preparation.

$\ddagger$: approximate effective resistance; using the Spielman--Srivastava JL method
\citep{spielman2008graph}, avoiding explicit Laplacian pseudoinversion.
We use $d=\lceil4\ln n\rceil$ Gaussian projections with approximate-Cholesky
solves~\citep{gao2023robust} ($10^{-2}$ tolerance) on the smaller graph,
64 Rademacher projections with PCG on \textsc{ogbn-products},
\textsc{ogbn-proteins}, and \textsc{Pokec}, and
TGT+~\citep{zhang2023efficient} with 128 eigenpairs on \textsc{Reddit}
\end{minipage}
\end{table}

\paragraph{End-to-end runtime.}
Table~\ref{tab:runtime-end-to-end} compares end-to-end runtime at the same
$\delta=0.3$ retention ratio over 300 training epochs, including one-time
sparsification cost. Despite this initial overhead, all three Scaffold variants
reduce total runtime relative to Full-GCN on \textsc{Reddit},
\textsc{ogbn-products}, and \textsc{ogbn-arxiv}. 
For example, \sgtsa
requires 287.1\,s versus 713.4\,s on \textsc{Reddit}, 642.0\,s versus 1154.0\,s on
\textsc{ogbn-products}, and 76.7\,s versus 96.7\,s on \textsc{ogbn-arxiv}. \sgtf and \sgtb show the same trend, indicating that the lower sparse-graph training cost amortizes support construction over the 300-epoch horizon. The marked AdaGLT and Unified-LTH
configurations exceed the 3600\,s run limit and are projected.

\begin{table}[!htbp]
\centering
\caption{End-to-end runtime (s) at $\delta=0.3$ over 300 epochs.
$S$: sparsification; $T_{300}$: training; $E_{300}=S+T_{300}$.}
\label{tab:runtime-end-to-end}

\scriptsize
\setlength{\tabcolsep}{2pt}
\renewcommand{\arraystretch}{1.08}

\begin{tabular*}{\linewidth}{@{}l@{\extracolsep{\fill}}rrrrrrrrr@{}}
\toprule
\textbf{Method}
& \multicolumn{3}{c}{\textbf{Reddit}}
& \multicolumn{3}{c}{\textbf{ogbn-products}}
& \multicolumn{3}{c}{\textbf{ogbn-arxiv}} \\
\cmidrule(lr){2-4}
\cmidrule(lr){5-7}
\cmidrule(lr){8-10}

& $S$ & $T_{300}$ & $E_{300}$
& $S$ & $T_{300}$ & $E_{300}$
& $S$ & $T_{300}$ & $E_{300}$ \\
\midrule

Full-GCN
& 0.0 & 713.4 & 713.4
& 0.0 & 1154.0 & 1154.0
& 0.0 & 96.7 & 96.7 \\

\midrule

Scaffold-Fast-1
& 72.0 & 237.0 & 309.0
& 153.1 & 549.0 & 702.0
& 6.3 & 76.3 & 82.6 \\

Scaffold-Batch-1
& 208.8 & 235.3 & 444.2
& 250.7 & 522.4 & 773.1
& 9.3 & 72.8 & 82.1 \\

Scaffold-Sample-1
& 48.8 & 238.3 & 287.1
& 133.9 & 508.1 & 642.0
& 6.1 & 70.7 & 76.7 \\

\midrule

DSpar
& 4.0 & 224.4 & 228.4
& 4.1 & 491.8 & 496.0
& 0.1 & 70.4 & 70.5 \\

\midrule

MoG
& incl. & 1432.3 & 1432.3
& incl. & 2536.4 & 2536.4
& incl. & 154.1 & 154.1 \\

AdaGLT
& incl. & 8819.6 & 8819.6$^{\dagger}$
& incl. & 15213.3 & 15213.3$^{\dagger}$
& incl. & 916.3 & 916.3 \\

Unified-LTH
& 5.4 & 7144.1 & 7149.5$^{\dagger}$
& 6.7 & 13161.1 & 13167.8$^{\dagger}$
& 0.1 & 395.9 & 396.0 \\

\bottomrule
\end{tabular*}

\par\vspace{2pt}
\begin{minipage}{\linewidth}\scriptsize
$^{\dagger}$ marks AdaGLT/Unified-LTH totals exceeding the 3600\,s run limit and are projected from the epochs they ran;
``incl.'' means graph-selection cost is included in training.
\end{minipage}
\end{table}

\FloatBarrier
\section{Additional Ablation and Structural Analysis}
\label{app:ablation}

\subsection{Scaffold Variants and Support Backbones with \texttt{Greedy}}
\label{app:variant-backbone}
Since \sgte is expensive to use in large graphs,
Figure~\ref{fig:backbonesvariantsgreedy} compares the performance of variants and backbones on two small graphs \textsc{Cora} and \textsc{Chameleon}. 
Across support backbones, \texttt{Greedy}, \texttt{Heap}, and \texttt{Batch} generally provide the strongest reductions in dilation and congestion as expected, while \texttt{Fast} and \texttt{Sample} trade structural refinement for substantially lower construction cost. As in the main results, predictive performance does not strictly follow the structural score and also depends on
the dataset and backbone.

\begin{figure}[!htbp]
    \centering
    \includegraphics[width=\linewidth]{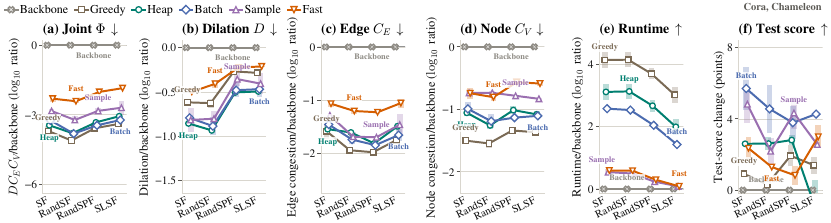}
    \caption{\sgt variants across support backbones.
    Extension of Figure~\ref{fig:variants_backbones} including the full
    Greedy reference on Cora and Chameleon, where its computational cost
    remains tractable.}
    \label{fig:backbonesvariantsgreedy}
\end{figure}

\FloatBarrier
\subsection{Effect of the path node and edge congestion parameter $p_E$, $p_V$}
\label{app:pathcongestion}
We study the effect of the path aggregation parameters $p_E$ and $p_V$ of 
Equation~\ref{eq:scaffold-path-edge-congestion} and \ref{eq:scaffold-path-node-congestion}.
We compare average ($p=1$), quadratic ($p=2$), and maximum ($p=\infty$) aggregation using \sgtb with $(\alpha,\beta_E,\beta_V)=(1,1,1)$, while keeping the support backbone and edge budget fixed.

\begin{figure}[!htbp]
\centering
\includegraphics[width=\textwidth]{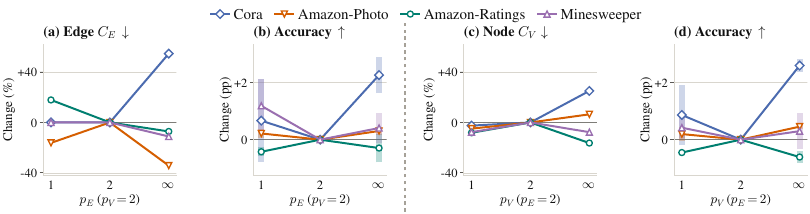}
\caption{Effect of path congestion aggregation.
(a--b) vary $p_E$ with $p_V=2$; (c--d) vary $p_V$ with $p_E=2$.
Changes are relative to $p_E=p_V=2$.}
\label{fig:path-order-concise}
\end{figure}

Figure~\ref{fig:path-order-concise} shows that changing the aggregation
order affects both the resulting edge/node congestion and downstream
accuracy. Larger $p_E$ places more emphasis on highly loaded edges, while
larger $p_V$ emphasizes highly loaded nodes. For these GNN benchmarks, the effect is dataset dependent,
and reductions in maximum congestion do not always translate directly into
higher accuracy. We therefore use $p_E=p_V=2$ as the default.

\FloatBarrier
\subsection{Sensitivity to the number of supports $K$}
\label{app:k-sensitivity}

We vary $K$ while keeping training settings and per-support budgets fixed. Evaluation unions $K$ validation-selected supports; training refreshes supports even at $K=1$.
On \textsc{Cora}, increasing $K$ from $1$ to $20$ raises peak coverage from $70\%$ to nearly $100\%$, while accuracy increases by $0.58$ (\texttt{Fast}) and $1.20$ (\texttt{Sample}) percentage points. On \textsc{Amazon-Computers}, both variants gain $0.85$ points by $K=4$, with little further improvement at $K=20$. On Coauthor-CS, both variants have lower accuracy at $K=20$ than at $K=16$
despite higher peak coverage.
The experiment shows that retaining more support for evaluation is useful; however, it comes at a computational cost. 

\begin{table}[!htbp]
\centering
\definecolor{KSensitivityBest}{HTML}{007F70}
\caption{Sensitivity to $K$. Accuracy is mean $\pm$ standard deviation (\%).
Cov. is the maximum logged union edge coverage (\%), excluding self-loops.
Green boxes pair each column's best accuracy with its coverage at the same $K$.}
\label{tab:k-sensitivity}
\setlength{\tabcolsep}{1.7pt}
\renewcommand{\arraystretch}{1.05}
\fontsize{8.4pt}{10pt}\selectfont
\begin{tabular*}{\textwidth}{@{}r@{\extracolsep{\fill}}*{6}{cc}@{}}
\toprule
\textbf{$K$} & \multicolumn{4}{c}{\tikz[remember picture,overlay]\coordinate (ks-table-top);\textbf{Cora}} & \multicolumn{4}{c}{\textbf{Amazon-Computers}} & \multicolumn{4}{c}{\textbf{Coauthor-CS}} \\
& \multicolumn{4}{c}{{\fontsize{6.8pt}{7.2pt}\selectfont$\delta=0.7$, 5 runs}} & \multicolumn{4}{c}{{\fontsize{6.8pt}{7.2pt}\selectfont$\delta=0.3$, 3 runs}} & \multicolumn{4}{c}{{\fontsize{6.8pt}{7.2pt}\selectfont$\delta=0.3$, 3 runs}} \\
\cmidrule(lr){2-5}\cmidrule(lr){6-9}\cmidrule(lr){10-13}
& \multicolumn{2}{c}{Fast} & \multicolumn{2}{c}{Sample} & \multicolumn{2}{c}{Fast} & \multicolumn{2}{c}{Sample} & \multicolumn{2}{c}{Fast} & \multicolumn{2}{c}{Sample} \\
\cmidrule(lr){2-3}\cmidrule(lr){4-5}\cmidrule(lr){6-7}\cmidrule(lr){8-9}\cmidrule(lr){10-11}\cmidrule(lr){12-13}
& Acc. & Cov. & Acc. & Cov. & Acc. & Cov. & Acc. & Cov. & Acc. & Cov. & Acc. & Cov. \\
\midrule
1 & 82.96{\color[gray]{0.45}\fontsize{6.2pt}{6.4pt}\selectfont\,\textpm0.86} & {\color[gray]{0.30}70.0} & 82.70{\color[gray]{0.45}\fontsize{6.2pt}{6.4pt}\selectfont\,\textpm0.85} & {\color[gray]{0.30}70.0}\tikz[remember picture,overlay]\coordinate (ks-gap-1-left); & \tikz[remember picture,overlay]\coordinate (ks-gap-1-right);91.88{\color[gray]{0.45}\fontsize{6.2pt}{6.4pt}\selectfont\,\textpm0.28} & {\color[gray]{0.30}30.0} & 92.58{\color[gray]{0.45}\fontsize{6.2pt}{6.4pt}\selectfont\,\textpm0.13} & {\color[gray]{0.30}30.0}\tikz[remember picture,overlay]\coordinate (ks-gap-2-left); & \tikz[remember picture,overlay]\coordinate (ks-gap-2-right);95.27{\color[gray]{0.45}\fontsize{6.2pt}{6.4pt}\selectfont\,\textpm0.14} & {\color[gray]{0.30}30.0} & 94.81{\color[gray]{0.45}\fontsize{6.2pt}{6.4pt}\selectfont\,\textpm0.36} & {\color[gray]{0.30}30.0} \\
2 & 82.70{\color[gray]{0.45}\fontsize{6.2pt}{6.4pt}\selectfont\,\textpm0.86} & {\color[gray]{0.30}87.4} & 82.62{\color[gray]{0.45}\fontsize{6.2pt}{6.4pt}\selectfont\,\textpm0.58} & {\color[gray]{0.30}80.3} & 92.35{\color[gray]{0.45}\fontsize{6.2pt}{6.4pt}\selectfont\,\textpm0.08} & {\color[gray]{0.30}33.8} & 93.04{\color[gray]{0.45}\fontsize{6.2pt}{6.4pt}\selectfont\,\textpm0.42} & {\color[gray]{0.30}49.7} & 95.35{\color[gray]{0.45}\fontsize{6.2pt}{6.4pt}\selectfont\,\textpm0.06} & {\color[gray]{0.30}46.7} & 95.35{\color[gray]{0.45}\fontsize{6.2pt}{6.4pt}\selectfont\,\textpm0.16} & {\color[gray]{0.30}37.6} \\
3 & 82.78{\color[gray]{0.45}\fontsize{6.2pt}{6.4pt}\selectfont\,\textpm0.37} & {\color[gray]{0.30}94.3} & 82.72{\color[gray]{0.45}\fontsize{6.2pt}{6.4pt}\selectfont\,\textpm0.99} & {\color[gray]{0.30}85.4} & 92.52{\color[gray]{0.45}\fontsize{6.2pt}{6.4pt}\selectfont\,\textpm0.42} & {\color[gray]{0.30}36.4} & 93.29{\color[gray]{0.45}\fontsize{6.2pt}{6.4pt}\selectfont\,\textpm0.36} & {\color[gray]{0.30}59.6} & 95.81{\color[gray]{0.45}\fontsize{6.2pt}{6.4pt}\selectfont\,\textpm0.04} & {\color[gray]{0.30}58.8} & 95.36{\color[gray]{0.45}\fontsize{6.2pt}{6.4pt}\selectfont\,\textpm0.22} & {\color[gray]{0.30}45.2} \\
4 & 82.82{\color[gray]{0.45}\fontsize{6.2pt}{6.4pt}\selectfont\,\textpm0.70} & {\color[gray]{0.30}97.0} & 82.96{\color[gray]{0.45}\fontsize{6.2pt}{6.4pt}\selectfont\,\textpm0.73} & {\color[gray]{0.30}90.8} & 92.73{\color[gray]{0.45}\fontsize{6.2pt}{6.4pt}\selectfont\,\textpm0.27} & {\color[gray]{0.30}38.3} & 93.43{\color[gray]{0.45}\fontsize{6.2pt}{6.4pt}\selectfont\,\textpm0.32} & {\color[gray]{0.30}68.6} & 95.72{\color[gray]{0.45}\fontsize{6.2pt}{6.4pt}\selectfont\,\textpm0.05} & {\color[gray]{0.30}67.3} & 95.45{\color[gray]{0.45}\fontsize{6.2pt}{6.4pt}\selectfont\,\textpm0.21} & {\color[gray]{0.30}52.8} \\
5 & 83.28{\color[gray]{0.45}\fontsize{6.2pt}{6.4pt}\selectfont\,\textpm1.14} & {\color[gray]{0.30}94.1} & 82.94{\color[gray]{0.45}\fontsize{6.2pt}{6.4pt}\selectfont\,\textpm0.62} & {\color[gray]{0.30}95.1} & 92.65{\color[gray]{0.45}\fontsize{6.2pt}{6.4pt}\selectfont\,\textpm0.18} & {\color[gray]{0.30}39.9} & 93.26{\color[gray]{0.45}\fontsize{6.2pt}{6.4pt}\selectfont\,\textpm0.08} & {\color[gray]{0.30}77.2} & 95.98{\color[gray]{0.45}\fontsize{6.2pt}{6.4pt}\selectfont\,\textpm0.15} & {\color[gray]{0.30}74.0} & 95.50{\color[gray]{0.45}\fontsize{6.2pt}{6.4pt}\selectfont\,\textpm0.05} & {\color[gray]{0.30}58.3} \\
6 & 83.00{\color[gray]{0.45}\fontsize{6.2pt}{6.4pt}\selectfont\,\textpm0.76} & {\color[gray]{0.30}99.0} & 82.86{\color[gray]{0.45}\fontsize{6.2pt}{6.4pt}\selectfont\,\textpm0.96} & {\color[gray]{0.30}91.3} & 92.75{\color[gray]{0.45}\fontsize{6.2pt}{6.4pt}\selectfont\,\textpm0.33} & {\color[gray]{0.30}41.2} & 93.34{\color[gray]{0.45}\fontsize{6.2pt}{6.4pt}\selectfont\,\textpm0.24} & {\color[gray]{0.30}83.2} & 95.91{\color[gray]{0.45}\fontsize{6.2pt}{6.4pt}\selectfont\,\textpm0.16} & {\color[gray]{0.30}79.0} & 95.71{\color[gray]{0.45}\fontsize{6.2pt}{6.4pt}\selectfont\,\textpm0.04} & {\color[gray]{0.30}64.5} \\
8 & 83.14{\color[gray]{0.45}\fontsize{6.2pt}{6.4pt}\selectfont\,\textpm0.44} & {\color[gray]{0.30}99.6} & 82.96{\color[gray]{0.45}\fontsize{6.2pt}{6.4pt}\selectfont\,\textpm0.34} & {\color[gray]{0.30}93.3} & 92.79{\color[gray]{0.45}\fontsize{6.2pt}{6.4pt}\selectfont\,\textpm0.30} & {\color[gray]{0.30}42.1} & 93.30{\color[gray]{0.45}\fontsize{6.2pt}{6.4pt}\selectfont\,\textpm0.37} & {\color[gray]{0.30}89.5} & 95.86{\color[gray]{0.45}\fontsize{6.2pt}{6.4pt}\selectfont\,\textpm0.23} & {\color[gray]{0.30}86.0} & 95.84{\color[gray]{0.45}\fontsize{6.2pt}{6.4pt}\selectfont\,\textpm0.11} & {\color[gray]{0.30}75.0} \\
10 & 82.84{\color[gray]{0.45}\fontsize{6.2pt}{6.4pt}\selectfont\,\textpm1.06} & {\color[gray]{0.30}100.0} & 83.32{\color[gray]{0.45}\fontsize{6.2pt}{6.4pt}\selectfont\,\textpm0.81} & {\color[gray]{0.30}88.7} & \tikz[remember picture,overlay]\coordinate (ks-best-1-fast-10-left);{\color{black}\bfseries 92.84}{\color[gray]{0.45}\fontsize{6.2pt}{6.4pt}\selectfont\,\textpm0.13} & {\color{KSensitivityBest}44.2}\tikz[remember picture,overlay]\coordinate (ks-best-1-fast-10-right); & 93.25{\color[gray]{0.45}\fontsize{6.2pt}{6.4pt}\selectfont\,\textpm0.09} & {\color[gray]{0.30}92.4} & 95.92{\color[gray]{0.45}\fontsize{6.2pt}{6.4pt}\selectfont\,\textpm0.16} & {\color[gray]{0.30}90.1} & 95.62{\color[gray]{0.45}\fontsize{6.2pt}{6.4pt}\selectfont\,\textpm0.34} & {\color[gray]{0.30}76.8} \\
12 & 83.22{\color[gray]{0.45}\fontsize{6.2pt}{6.4pt}\selectfont\,\textpm0.70} & {\color[gray]{0.30}99.0} & 83.22{\color[gray]{0.45}\fontsize{6.2pt}{6.4pt}\selectfont\,\textpm0.67} & {\color[gray]{0.30}95.1} & 92.58{\color[gray]{0.45}\fontsize{6.2pt}{6.4pt}\selectfont\,\textpm0.46} & {\color[gray]{0.30}44.7} & 93.32{\color[gray]{0.45}\fontsize{6.2pt}{6.4pt}\selectfont\,\textpm0.29} & {\color[gray]{0.30}94.0} & 95.77{\color[gray]{0.45}\fontsize{6.2pt}{6.4pt}\selectfont\,\textpm0.00} & {\color[gray]{0.30}93.0} & 95.77{\color[gray]{0.45}\fontsize{6.2pt}{6.4pt}\selectfont\,\textpm0.12} & {\color[gray]{0.30}85.1} \\
16 & \tikz[remember picture,overlay]\coordinate (ks-best-0-fast-16-left);{\color{black}\bfseries 83.80}{\color[gray]{0.45}\fontsize{6.2pt}{6.4pt}\selectfont\,\textpm1.07} & {\color{KSensitivityBest}100.0}\tikz[remember picture,overlay]\coordinate (ks-best-0-fast-16-right); & 83.44{\color[gray]{0.45}\fontsize{6.2pt}{6.4pt}\selectfont\,\textpm0.86} & {\color[gray]{0.30}99.1} & 92.61{\color[gray]{0.45}\fontsize{6.2pt}{6.4pt}\selectfont\,\textpm0.19} & {\color[gray]{0.30}46.6} & \tikz[remember picture,overlay]\coordinate (ks-best-1-sample-16-left);{\color{black}\bfseries 93.61}{\color[gray]{0.45}\fontsize{6.2pt}{6.4pt}\selectfont\,\textpm0.16} & {\color{KSensitivityBest}95.4}\tikz[remember picture,overlay]\coordinate (ks-best-1-sample-16-right); & \tikz[remember picture,overlay]\coordinate (ks-best-2-fast-16-left);{\color{black}\bfseries 95.99}{\color[gray]{0.45}\fontsize{6.2pt}{6.4pt}\selectfont\,\textpm0.20} & {\color{KSensitivityBest}96.4}\tikz[remember picture,overlay]\coordinate (ks-best-2-fast-16-right); & \tikz[remember picture,overlay]\coordinate (ks-best-2-sample-16-left);{\color{black}\bfseries 95.97}{\color[gray]{0.45}\fontsize{6.2pt}{6.4pt}\selectfont\,\textpm0.08} & {\color{KSensitivityBest}85.7}\tikz[remember picture,overlay]\coordinate (ks-best-2-sample-16-right); \\
\tikz[remember picture,overlay]\coordinate (ks-table-bottom);20 & 83.54{\color[gray]{0.45}\fontsize{6.2pt}{6.4pt}\selectfont\,\textpm0.70} & {\color[gray]{0.30}100.0} & \tikz[remember picture,overlay]\coordinate (ks-best-0-sample-20-left);{\color{black}\bfseries 83.90}{\color[gray]{0.45}\fontsize{6.2pt}{6.4pt}\selectfont\,\textpm1.04} & {\color{KSensitivityBest}99.3}\tikz[remember picture,overlay]\coordinate (ks-best-0-sample-20-right); & 92.72{\color[gray]{0.45}\fontsize{6.2pt}{6.4pt}\selectfont\,\textpm0.22} & {\color[gray]{0.30}47.6} & 93.58{\color[gray]{0.45}\fontsize{6.2pt}{6.4pt}\selectfont\,\textpm0.22} & {\color[gray]{0.30}99.6} & 95.56{\color[gray]{0.45}\fontsize{6.2pt}{6.4pt}\selectfont\,\textpm0.39} & {\color[gray]{0.30}97.9} & 95.90{\color[gray]{0.45}\fontsize{6.2pt}{6.4pt}\selectfont\,\textpm0.11} & {\color[gray]{0.30}91.7} \\
\bottomrule
\end{tabular*}
\begin{tikzpicture}[remember picture,overlay]
\coordinate (ks-divider-1) at ($(ks-gap-1-left)!0.5!(ks-gap-1-right)$);
\draw[gray!65,line width=0.4pt,dash pattern=on 2pt off 2pt]
([yshift=7pt]ks-divider-1 |- ks-table-top) -- ([yshift=-3pt]ks-divider-1 |- ks-table-bottom);
\coordinate (ks-divider-2) at ($(ks-gap-2-left)!0.5!(ks-gap-2-right)$);
\draw[gray!65,line width=0.4pt,dash pattern=on 2pt off 2pt]
([yshift=7pt]ks-divider-2 |- ks-table-top) -- ([yshift=-3pt]ks-divider-2 |- ks-table-bottom);
\draw[KSensitivityBest,line width=0.5pt,dash pattern=on 2pt off 1.3pt]
([xshift=-1.5pt,yshift=7.1pt]ks-best-1-fast-10-left) rectangle ([xshift=1.5pt,yshift=-2.1pt]ks-best-1-fast-10-right);
\draw[KSensitivityBest,line width=0.5pt,dash pattern=on 2pt off 1.3pt]
([xshift=-1.5pt,yshift=7.1pt]ks-best-0-fast-16-left) rectangle ([xshift=1.5pt,yshift=-2.1pt]ks-best-0-fast-16-right);
\draw[KSensitivityBest,line width=0.5pt,dash pattern=on 2pt off 1.3pt]
([xshift=-1.5pt,yshift=7.1pt]ks-best-1-sample-16-left) rectangle ([xshift=1.5pt,yshift=-2.1pt]ks-best-1-sample-16-right);
\draw[KSensitivityBest,line width=0.5pt,dash pattern=on 2pt off 1.3pt]
([xshift=-1.5pt,yshift=7.1pt]ks-best-2-fast-16-left) rectangle ([xshift=1.5pt,yshift=-2.1pt]ks-best-2-fast-16-right);
\draw[KSensitivityBest,line width=0.5pt,dash pattern=on 2pt off 1.3pt]
([xshift=-1.5pt,yshift=7.1pt]ks-best-2-sample-16-left) rectangle ([xshift=1.5pt,yshift=-2.1pt]ks-best-2-sample-16-right);
\draw[KSensitivityBest,line width=0.5pt,dash pattern=on 2pt off 1.3pt]
([xshift=-1.5pt,yshift=7.1pt]ks-best-0-sample-20-left) rectangle ([xshift=1.5pt,yshift=-2.1pt]ks-best-0-sample-20-right);
\end{tikzpicture}
\end{table}

\subsubsection{Evaluation on the full graph}
\label{app:full-graph-evaluation}
If memory permits, evaluating on the full graph can be preferable while training on a $q$-edge budget, since evaluation requires less memory. However, the following results show that full-graph evaluation does not consistently yield higher scores. One possible reason is that the best $K$ supports may exclude noisy edges and provide better support for evaluation. For small graphs with larger $K$, the support union can cover most or all of the graph.

\begin{table}[!htbp]
\centering
\definecolor{FullEvalBest}{HTML}{000000}
\definecolor{FullEvalGain}{HTML}{007F70}
\definecolor{FullEvalLoss}{HTML}{D55E00}

\caption{Scaffold evaluation at each dataset's target retention $\delta$.
\sgt-1 uses one fixed support for training and evaluation; \sgt-$K$
refreshes supports and evaluates their validation-selected union ($K=10$).
\sgt-Full trains on sparse supports and evaluates on the full graph.
Full-GCN trains and evaluates on the full graph. Reported mean $\pm$ SD (\%).}
\label{tab:scaffold-full-evaluation}

\setlength{\tabcolsep}{3pt}
\renewcommand{\arraystretch}{1.05}
\fontsize{8.4pt}{10pt}\selectfont

\begin{tabular*}{\textwidth}{@{}l@{\extracolsep{\fill}}cc*{3}{c}@{\hspace{14pt}}l@{}}
\toprule
\textbf{Dataset}
& $\boldsymbol{\delta}$
& \textbf{Variant}
& \textbf{Scaffold-1}
& \textbf{Scaffold-$K$}
& \textbf{Scaffold-Full}
& \tikz[remember picture,overlay]\coordinate (full-eval-ref-top);\textbf{Full-GCN} \\
\midrule

Cora
& 0.7
& Batch
& 83.80{\color[gray]{0.45}\fontsize{6.2pt}{6.4pt}\selectfont\,\textpm0.92}
& \makebox[7pt][l]{{\color{FullEvalLoss}\fontsize{7.2pt}{8pt}\selectfont$\downarrow$}}83.54{\color[gray]{0.45}\fontsize{6.2pt}{6.4pt}\selectfont\,\textpm0.71}
& \makebox[7pt][l]{{\color{FullEvalGain}\fontsize{7.2pt}{8pt}\selectfont$\uparrow$}}{\color{FullEvalBest}\bfseries 84.58}{\color[gray]{0.45}\fontsize{6.2pt}{6.4pt}\selectfont\,\textpm0.56}
& 84.56{\color[gray]{0.45}\fontsize{6.2pt}{6.4pt}\selectfont\,\textpm0.63} \\

CiteSeer
& 0.7
& Fast
& 70.18{\color[gray]{0.45}\fontsize{6.2pt}{6.4pt}\selectfont\,\textpm0.83}
& \makebox[7pt][l]{{\color{FullEvalGain}\fontsize{7.2pt}{8pt}\selectfont$\uparrow$}}71.52{\color[gray]{0.45}\fontsize{6.2pt}{6.4pt}\selectfont\,\textpm0.97}
& \makebox[7pt][l]{{\color{FullEvalGain}\fontsize{7.2pt}{8pt}\selectfont$\uparrow$}}{\color{FullEvalBest}\bfseries 71.54}{\color[gray]{0.45}\fontsize{6.2pt}{6.4pt}\selectfont\,\textpm0.29}
& 72.30{\color[gray]{0.45}\fontsize{6.2pt}{6.4pt}\selectfont\,\textpm0.47} \\

PubMed
& 0.5
& Fast
& 79.34{\color[gray]{0.45}\fontsize{6.2pt}{6.4pt}\selectfont\,\textpm0.99}
& \makebox[7pt][l]{{\color{FullEvalLoss}\fontsize{7.2pt}{8pt}\selectfont$\downarrow$}}79.08{\color[gray]{0.45}\fontsize{6.2pt}{6.4pt}\selectfont\,\textpm0.78}
& \makebox[7pt][l]{{\color{FullEvalGain}\fontsize{7.2pt}{8pt}\selectfont$\uparrow$}}{\color{FullEvalBest}\bfseries 80.46}{\color[gray]{0.45}\fontsize{6.2pt}{6.4pt}\selectfont\,\textpm0.63}
& 81.22{\color[gray]{0.45}\fontsize{6.2pt}{6.4pt}\selectfont\,\textpm1.11} \\

Amazon-Photo
& 0.3
& Batch
& 95.10{\color[gray]{0.45}\fontsize{6.2pt}{6.4pt}\selectfont\,\textpm0.17}
& \makebox[7pt][l]{{\color{FullEvalGain}\fontsize{7.2pt}{8pt}\selectfont$\uparrow$}}{\color{FullEvalBest}\bfseries 95.92}{\color[gray]{0.45}\fontsize{6.2pt}{6.4pt}\selectfont\,\textpm0.15}
& \makebox[7pt][l]{{\color{FullEvalGain}\fontsize{7.2pt}{8pt}\selectfont$\uparrow$}}95.82{\color[gray]{0.45}\fontsize{6.2pt}{6.4pt}\selectfont\,\textpm0.28}\rlap{$^{*}$}
& 95.95{\color[gray]{0.45}\fontsize{6.2pt}{6.4pt}\selectfont\,\textpm0.33} \\

Amazon-Computers
& 0.3
& Sample
& 93.06{\color[gray]{0.45}\fontsize{6.2pt}{6.4pt}\selectfont\,\textpm0.11}
& \makebox[7pt][l]{{\color{FullEvalGain}\fontsize{7.2pt}{8pt}\selectfont$\uparrow$}}{\color{FullEvalBest}\bfseries 93.57}{\color[gray]{0.45}\fontsize{6.2pt}{6.4pt}\selectfont\,\textpm0.18}
& \makebox[7pt][l]{{\color{FullEvalGain}\fontsize{7.2pt}{8pt}\selectfont$\uparrow$}}93.30{\color[gray]{0.45}\fontsize{6.2pt}{6.4pt}\selectfont\,\textpm0.38}
& 93.80{\color[gray]{0.45}\fontsize{6.2pt}{6.4pt}\selectfont\,\textpm0.20} \\

Coauthor-CS
& 0.5
& Fast
& 95.39{\color[gray]{0.45}\fontsize{6.2pt}{6.4pt}\selectfont\,\textpm0.24}
& \makebox[7pt][l]{{\color{FullEvalGain}\fontsize{7.2pt}{8pt}\selectfont$\uparrow$}}{\color{FullEvalBest}\bfseries 96.04}{\color[gray]{0.45}\fontsize{6.2pt}{6.4pt}\selectfont\,\textpm0.06}
& \makebox[7pt][l]{{\color{FullEvalGain}\fontsize{7.2pt}{8pt}\selectfont$\uparrow$}}95.99{\color[gray]{0.45}\fontsize{6.2pt}{6.4pt}\selectfont\,\textpm0.25}
& 95.72{\color[gray]{0.45}\fontsize{6.2pt}{6.4pt}\selectfont\,\textpm0.04} \\

Coauthor-Physics
& 0.3
& Sample
& 97.09{\color[gray]{0.45}\fontsize{6.2pt}{6.4pt}\selectfont\,\textpm0.06}
& \makebox[7pt][l]{{\color{FullEvalGain}\fontsize{7.2pt}{8pt}\selectfont$\uparrow$}}97.22{\color[gray]{0.45}\fontsize{6.2pt}{6.4pt}\selectfont\,\textpm0.14}
& \makebox[7pt][l]{{\color{FullEvalGain}\fontsize{7.2pt}{8pt}\selectfont$\uparrow$}}{\color{FullEvalBest}\bfseries 97.43}{\color[gray]{0.45}\fontsize{6.2pt}{6.4pt}\selectfont\,\textpm0.06}
& 97.50{\color[gray]{0.45}\fontsize{6.2pt}{6.4pt}\selectfont\,\textpm0.12} \\

WikiCS
& 0.3
& Sample
& 79.48{\color[gray]{0.45}\fontsize{6.2pt}{6.4pt}\selectfont\,\textpm0.32}
& \makebox[7pt][l]{{\color{FullEvalGain}\fontsize{7.2pt}{8pt}\selectfont$\uparrow$}}79.97{\color[gray]{0.45}\fontsize{6.2pt}{6.4pt}\selectfont\,\textpm0.63}
& \makebox[7pt][l]{{\color{FullEvalGain}\fontsize{7.2pt}{8pt}\selectfont$\uparrow$}}{\color{FullEvalBest}\bfseries 80.21}{\color[gray]{0.45}\fontsize{6.2pt}{6.4pt}\selectfont\,\textpm0.50}
& 80.27{\color[gray]{0.45}\fontsize{6.2pt}{6.4pt}\selectfont\,\textpm0.44} \\

\midrule

Amazon-Ratings
& 0.7
& Batch
& 52.60{\color[gray]{0.45}\fontsize{6.2pt}{6.4pt}\selectfont\,\textpm0.19}
& \makebox[7pt][l]{{\color{FullEvalGain}\fontsize{7.2pt}{8pt}\selectfont$\uparrow$}}{\color{FullEvalBest}\bfseries 54.73}{\color[gray]{0.45}\fontsize{6.2pt}{6.4pt}\selectfont\,\textpm0.20}
& \makebox[7pt][l]{{\color{FullEvalLoss}\fontsize{7.2pt}{8pt}\selectfont$\downarrow$}}52.36{\color[gray]{0.45}\fontsize{6.2pt}{6.4pt}\selectfont\,\textpm0.36}\rlap{$^{*}$}
& 53.61{\color[gray]{0.45}\fontsize{6.2pt}{6.4pt}\selectfont\,\textpm0.63} \\

Minesweeper$^{\dagger}$
& 0.7
& Fast
& 86.98{\color[gray]{0.45}\fontsize{6.2pt}{6.4pt}\selectfont\,\textpm0.76}
& \makebox[7pt][l]{{\color{FullEvalGain}\fontsize{7.2pt}{8pt}\selectfont$\uparrow$}}{\color{FullEvalBest}\bfseries 93.35}{\color[gray]{0.45}\fontsize{6.2pt}{6.4pt}\selectfont\,\textpm0.34}
& \makebox[7pt][l]{{\color{FullEvalGain}\fontsize{7.2pt}{8pt}\selectfont$\uparrow$}}93.22{\color[gray]{0.45}\fontsize{6.2pt}{6.4pt}\selectfont\,\textpm0.23}
& 97.18{\color[gray]{0.45}\fontsize{6.2pt}{6.4pt}\selectfont\,\textpm0.45} \\

Questions$^{\dagger}$
& 0.7
& Sample
& 77.81{\color[gray]{0.45}\fontsize{6.2pt}{6.4pt}\selectfont\,\textpm0.14}
& \makebox[7pt][l]{{\color{FullEvalGain}\fontsize{7.2pt}{8pt}\selectfont$\uparrow$}}78.38{\color[gray]{0.45}\fontsize{6.2pt}{6.4pt}\selectfont\,\textpm0.13}
& \makebox[7pt][l]{{\color{FullEvalGain}\fontsize{7.2pt}{8pt}\selectfont$\uparrow$}}{\color{FullEvalBest}\bfseries 78.45}{\color[gray]{0.45}\fontsize{6.2pt}{6.4pt}\selectfont\,\textpm0.40}\rlap{$^{*}$}
& \tikz[remember picture,overlay]\coordinate (full-eval-ref-bottom);78.47{\color[gray]{0.45}\fontsize{6.2pt}{6.4pt}\selectfont\,\textpm0.29} \\

\bottomrule
\end{tabular*}

\tikz[remember picture,overlay]{
\draw[
    gray!65,
    line width=0.4pt,
    dash pattern=on 2pt off 2pt
]
([xshift=-7pt,yshift=7pt]full-eval-ref-top)
--
([xshift=-7pt,yshift=-3pt]full-eval-ref-top |- full-eval-ref-bottom);
}

\vspace{3pt}
\begin{minipage}{\textwidth}
\fontsize{7.2pt}{8.5pt}\selectfont
{\color{FullEvalBest}\textbf{Blue bold}}: best Scaffold mean;
{\color{FullEvalGain}$\uparrow$}/{\color{FullEvalLoss}$\downarrow$}: vs.\ Scaffold-1.
$^{\dagger}$AUC; otherwise accuracy. $^{*}$Time-limited training.
Variants chosen by Scaffold-$K$ validation. Training used separate runs and settings.
\end{minipage}
\end{table}

\FloatBarrier
\subsection{Support-Refresh Frequency}
\label{app:refresh-frequency}

Figure~\ref{fig:appendix-cora-support-refresh} studies support-refresh frequency
on \textsc{Cora} using a shared fast approximation of random spanning forest (\texttt{RandSF}) sequence. Refreshing every epoch reduces the gap to the full graph to $0.89$ percentage points, compared with $2.09$ and $3.79$ points when refreshing every 10 and 50 epochs, respectively.
Matched-support controls show a similar effect: with identical support sets, every-epoch refresh improves accuracy by $2.65$ points over every 50 epochs and by $1.06$ points over every 10 epochs.
This suggests that stale supports can hinder optimization beyond the effect of support diversity.
We treat this single-setting experiment as a mechanism ablation.

\begin{figure}[!htbp]
    \centering
    \includegraphics[width=\linewidth]{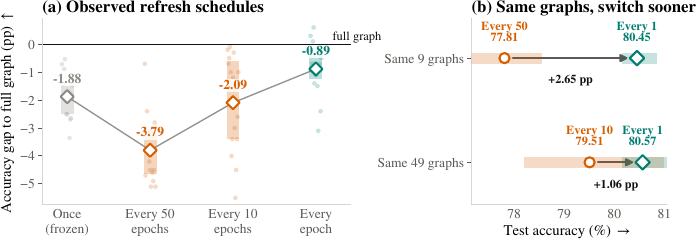}
    \caption{Support-refresh ablation on Cora with fast approximation of random forest (\texttt{RandSF}) at a 70\% edge budget. \textbf{(a)} Accuracy across refresh schedules.
    \textbf{(b)} Matched-support comparisons with identical support sets.}
    \label{fig:appendix-cora-support-refresh}
\end{figure}

\FloatBarrier

\subsection{Impact of Synchronous vs.\ Asynchronous Support Generation}
\label{app:sync-async}

For large graphs, constructing a new sparse support can be expensive and may stall GPU training. \sgtsa mitigates this by precomputing edge-sampling
probabilities, making subsequent support sampling inexpensive. Other variants, such as \texttt{Fast} and \texttt{Batch}, can require substantially more time
to construct each support. We therefore generate the next support asynchronously on CPU workers while the GPU trains on the current one, and swap supports once the new support is ready. As shown in Figure~\ref{fig:appendix-cora-support-refresh}, frequent support refresh is generally preferable to long refresh intervals.

On Cora, synchronization has little effect on runtime. On the relatively larger \textsc{Amazon-Computers} graph, asynchronous generation reduces runtime by
$7.9\times$ for \texttt{Fast} and $1.4\times$ for \texttt{Sample}, with negligible accuracy changes. The smaller gap for \texttt{Sample} reflects its cheap sampling step, whereas variants with more expensive support construction benefit substantially from asynchronous generation. This makes asynchronous refresh particularly important for scaling \texttt{\sgt-$K$} and \texttt{\sgt-Full} to large graphs.

\begin{table}[!htbp]
\centering
\caption{Synchronous versus asynchronous support generation. Accuracy is
mean $\pm$ SD; runtime is mean seconds per run.}
\label{tab:async-vs-sync}

\setlength{\tabcolsep}{4pt}
\renewcommand{\arraystretch}{1.08}
\fontsize{8.4pt}{10pt}\selectfont

\begin{tabular}{@{}lcl@{\hspace{6pt}}cc@{\hspace{6pt}}cc@{}}
\toprule
\textbf{Dataset} & $\boldsymbol{\delta}$ & \textbf{Method}
& \multicolumn{2}{c}{\textbf{Accuracy (\%)}}
& \multicolumn{2}{c}{\textbf{Time (s)}} \\
\cmidrule(lr){4-5}\cmidrule(lr){6-7}
& & & Sync & Async & Sync & Async \\
\midrule
Cora & 0.7 & \textbf{Scaffold-Fast}
& 82.92{\color[gray]{0.45}\fontsize{6.2pt}{6.4pt}\selectfont\,\textpm0.47}
& 82.94{\color[gray]{0.45}\fontsize{6.2pt}{6.4pt}\selectfont\,\textpm0.90}
& 12.6 & 12.3 \\

& & \textbf{Scaffold-Sample}
& 83.28{\color[gray]{0.45}\fontsize{6.2pt}{6.4pt}\selectfont\,\textpm0.26}
& 83.00{\color[gray]{0.45}\fontsize{6.2pt}{6.4pt}\selectfont\,\textpm0.29}
& 11.8 & 11.4 \\

\midrule
Amazon-Computers & 0.3 & \textbf{Scaffold-Fast}
& 92.09{\color[gray]{0.45}\fontsize{6.2pt}{6.4pt}\selectfont\,\textpm0.09}
& 92.18{\color[gray]{0.45}\fontsize{6.2pt}{6.4pt}\selectfont\,\textpm0.15}
& 280.6 & 35.7 \\

& & \textbf{Scaffold-Sample}
& 92.52{\color[gray]{0.45}\fontsize{6.2pt}{6.4pt}\selectfont\,\textpm0.28}
& 92.55{\color[gray]{0.45}\fontsize{6.2pt}{6.4pt}\selectfont\,\textpm0.36}
& 50.0 & 36.8 \\
\bottomrule
\end{tabular}
\end{table}

\subsection{Training Support Refresh and Inference Topology}
\label{app:train-inference}

Figure~\ref{fig:appendix-train-eval-sparse} compares fixed-support and every-epoch support-refresh training under different inference topologies. On \textsc{Cora}, full-graph inference improves both fixed-support and refreshed training, whereas on \textsc{Amazon-Photo} the differences among inference choices are small. On \textsc{Squirrel}, the best-5 support ensemble performs best after support refresh, while full-graph inference provides little benefit.
For \textsc{Chameleon}, full-graph inference substantially improves fixed-support training (38.90\% to 43.93\%), but provides almost no additional
gain after every-epoch refresh (43.22\% versus 43.17\%).
Overall, training-time support refresh and inference topology interact differently across datasets, and full-graph inference is most beneficial when training on a single fixed support.

\begin{figure}[!htbp]
    \centering
    \includegraphics[width=\linewidth]{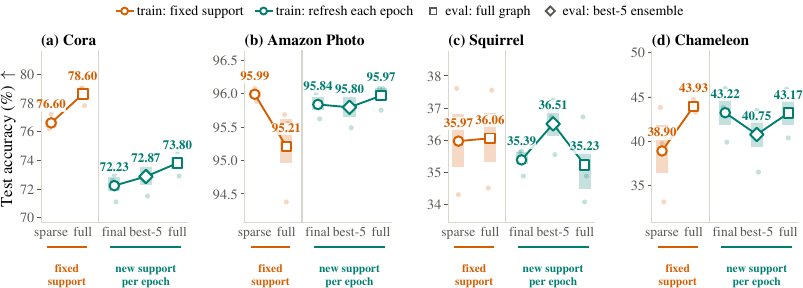}
    \caption{Training support refresh and inference topology at a 70\% edge budget. Models are trained on either a fixed support or a new support each epoch and evaluated on the corresponding sparse support, the best-5 supports, or the full graph.}
    \label{fig:appendix-train-eval-sparse}
\end{figure}

\FloatBarrier
\subsection{\texttt{\sgt-K} vs Random sparse sample at every epoch}
\label{app:sparsesamplerefresh}
To test whether the choice of edges matters beyond repeated sampling, we compare Scaffold-$K$ with Random-$K$. Random-$K$ draws a uniform random edge sample each epoch and evaluates on the union of its $K=5$ samples with the highest validation
accuracy. We match the main-table GCN settings, splits, and training budgets. \texttt{Random-1} is the fixed Random baseline from the main table. Each training sample retains $30\%$ of the edges; a union may exceed this budget.

\begin{table}[!htbp]
\centering
\caption{Sampling comparison at $\delta=0.3$. Test accuracy (mean $\pm$ SD, \%),
with 10 runs per setting. Scaffold uses the Sample variant.}
\label{tab:random-sampling}

\setlength{\tabcolsep}{6pt}
\renewcommand{\arraystretch}{1.0}
\fontsize{8.4pt}{10pt}\selectfont

\begin{tabular}{@{}lcc@{}}
\toprule
\textbf{Method} & \textbf{Chameleon} & \textbf{Squirrel} \\
\midrule

Random-1
& 37.85{\color[gray]{0.45}\fontsize{6.2pt}{6.4pt}\selectfont\,\textpm3.35}
& 40.90{\color[gray]{0.45}\fontsize{6.2pt}{6.4pt}\selectfont\,\textpm1.70} \\

\textbf{Scaffold-1}
& \textbf{42.50}{\color[gray]{0.45}\fontsize{6.2pt}{6.4pt}\selectfont\,\textpm4.38}
& \textbf{44.62}{\color[gray]{0.45}\fontsize{6.2pt}{6.4pt}\selectfont\,\textpm1.67} \\

\midrule

Random-$K$
& 43.09{\color[gray]{0.45}\fontsize{6.2pt}{6.4pt}\selectfont\,\textpm3.30}
& 43.24{\color[gray]{0.45}\fontsize{6.2pt}{6.4pt}\selectfont\,\textpm2.03} \\

\textbf{Scaffold-$K$}
& \textbf{43.89}{\color[gray]{0.45}\fontsize{6.2pt}{6.4pt}\selectfont\,\textpm2.40}
& \textbf{43.99}{\color[gray]{0.45}\fontsize{6.2pt}{6.4pt}\selectfont\,\textpm1.59} \\

\midrule

Full-GCN
& 45.46{\color[gray]{0.45}\fontsize{6.2pt}{6.4pt}\selectfont\,\textpm4.42}
& 44.77{\color[gray]{0.45}\fontsize{6.2pt}{6.4pt}\selectfont\,\textpm2.36} \\

\bottomrule
\end{tabular}
\end{table}
\FloatBarrier
\subsection{Structural implications of \sgt}
\label{app:structural-implications}

Figure~\ref{fig:scaffold-implications} highlights how \sgt preserves communication under sparsification. The spanning backbone preserves connected components above the spanning-forest floor, while low dilation keeps more omitted relations within the finite-hop receptive field of a GNN. Randomized backbones further expose complementary routes across repeated supports: on \textsc{Cora} at $\delta=0.7$, refreshing \texttt{RandSF} improves accuracy by $2.13$ points and approaches full-graph performance while increasing four-hop coverage.

\begin{figure}[!htbp]
    \centering
    \includegraphics[width=\linewidth]{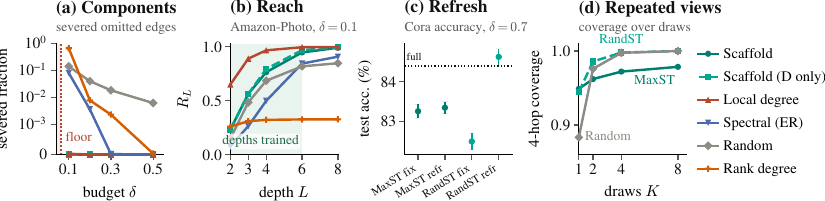}
    \caption{
    The spanning backbone preserves connectivity; low-dilation maintains finite-hop access, and randomized refresh exposes complementary
    communication routes across repeated sparse views.
    }
    \label{fig:scaffold-implications}
\end{figure}

\subsection{Compatibility of Scaffold across GNN architectures}
\label{app:other-gnns}

\textsc{Scaffold} is backbone-agnostic, and it sparsifies the graph without modifying the GNN architecture. 
Table~\ref{tab:backbone-amazon-arxiv} evaluates \sgt with GCN, GAT~\citep{velivckovic2017graph}, GIN~\citep{xu2018powerful}, and GraphSAGE on \textsc{Amazon-Computers} and \textsc{ogbn-arxiv} at $\delta=0.3$.
\sgt-1 uses one fixed support, whereas \sgt-$K$ asynchronously refreshes supports during training and evaluates on the union of $K=5$ validation-best supports.
All methods within each dataset use matched training settings and hardware, so runtime comparisons are made within each dataset and architecture.
Across architectures, \sgt generally reduces runtime while remaining competitive in accuracy. Compared with \texttt{\sgt-1}, \texttt{\sgt-K} can recover some accuracy through support refresh, potentially because different sparse views expose complementary information to the GNN.

The lower panel reports the number of asynchronous support swaps applied during training.


\newcolumntype{A}{>{\centering\arraybackslash}m{1.35cm}}
\newcolumntype{T}{>{\centering\arraybackslash}m{3.35cm}}

\begin{table}[!htbp]
\centering
\caption{\sgt across GNN architectures at $\delta=0.3$. Accuracy is mean $\pm$ SD; time is mean wall-clock runtime. Raised annotations show accuracy gap (pp) and speedup relative to the corresponding full-graph model.}
\label{tab:backbone-amazon-arxiv}

\setlength{\tabcolsep}{2pt}
\renewcommand{\arraystretch}{1.05}
\fontsize{8.2pt}{9.8pt}\selectfont

\begin{tabular}{@{}l A T @{\hspace{6pt}} A T@{}}
\toprule
\textbf{Method}
& \multicolumn{2}{c}{\shortstack{\textbf{Amazon}\\\textbf{Computers}}}
& \multicolumn{2}{c}{\textbf{ogbn-arxiv}} \\

& \multicolumn{2}{c}{{\fontsize{6.8pt}{7.2pt}\selectfont $\delta_d=0.3$}}
& \multicolumn{2}{c}{{\fontsize{6.8pt}{7.2pt}\selectfont $\delta_d=0.3$}} \\

& \multicolumn{2}{c}{{\fontsize{6pt}{6.6pt}\selectfont
$|V|$\,13.8K\quad $|E|$\,245.9K}}
& \multicolumn{2}{c}{{\fontsize{6pt}{6.6pt}\selectfont
$|V|$\,169.3K\quad $|E|$\,1.2M}} \\

\cmidrule(lr){2-3}\cmidrule(lr){4-5}

& {\fontsize{7.2pt}{7.6pt}\selectfont Acc.}
& {\fontsize{7.2pt}{7.6pt}\selectfont Time}
& {\fontsize{7.2pt}{7.6pt}\selectfont Acc.}
& {\fontsize{7.2pt}{7.6pt}\selectfont Time} \\
\midrule

GCN
& 93.77{\color[gray]{0.45}\fontsize{6.2pt}{6.4pt}\selectfont\,\textpm0.03}
& {\color[gray]{0.30}39.4\,s}
& 72.58{\color[gray]{0.45}\fontsize{6.2pt}{6.4pt}\selectfont\,\textpm1.57}
& {\color[gray]{0.30}1.9\,min} \\

\textbf{Scaf.-GCN}$^{1}$
& 92.86{\color[gray]{0.45}\fontsize{6.2pt}{6.4pt}\selectfont\,\textpm0.44}
  \rlap{\smash{\raisebox{0.62ex}{
  \color[rgb]{0.706,0.275,0.059}
  \fontsize{5.2pt}{5.2pt}\selectfont\,-0.91}}}
& {\color[gray]{0.30}34.9\,s}
  \rlap{\smash{\raisebox{0.62ex}{
  \color[rgb]{0.043,0.478,0.357}
  \fontsize{5.2pt}{5.2pt}\selectfont\,1.13$\times$}}}
& 69.83{\color[gray]{0.45}\fontsize{6.2pt}{6.4pt}\selectfont\,\textpm0.36}
  \rlap{\smash{\raisebox{0.62ex}{
  \color[rgb]{0.706,0.275,0.059}
  \fontsize{5.2pt}{5.2pt}\selectfont\,-2.75}}}
& {\color[gray]{0.30}1.8\,min}
  \rlap{\smash{\raisebox{0.62ex}{
  \color[rgb]{0.043,0.478,0.357}
  \fontsize{5.2pt}{5.2pt}\selectfont\,1.06$\times$}}} \\

\textbf{Scaf.-GCN}$^{K,\mathrm{async}}$
& 93.51{\color[gray]{0.45}\fontsize{6.2pt}{6.4pt}\selectfont\,\textpm0.08}
  \rlap{\smash{\raisebox{0.62ex}{
  \color[rgb]{0.706,0.275,0.059}
  \fontsize{5.2pt}{5.2pt}\selectfont\,-0.25}}}
& {\color[gray]{0.30}36.3\,s}
  \rlap{\smash{\raisebox{0.62ex}{
  \color[rgb]{0.043,0.478,0.357}
  \fontsize{5.2pt}{5.2pt}\selectfont\,1.08$\times$}}}
& 72.15{\color[gray]{0.45}\fontsize{6.2pt}{6.4pt}\selectfont\,\textpm0.46}
  \rlap{\smash{\raisebox{0.62ex}{
  \color[rgb]{0.706,0.275,0.059}
  \fontsize{5.2pt}{5.2pt}\selectfont\,-0.43}}}
& {\color[gray]{0.30}1.7\,min}
  \rlap{\smash{\raisebox{0.62ex}{
  \color[rgb]{0.043,0.478,0.357}
  \fontsize{5.2pt}{5.2pt}\selectfont\,1.16$\times$}}} \\

\addlinespace[1pt]
\cmidrule(lr){1-5}
\addlinespace[1pt]

GAT
& 93.49{\color[gray]{0.45}\fontsize{6.2pt}{6.4pt}\selectfont\,\textpm0.26}
& {\color[gray]{0.30}59.0\,s}
& 72.13{\color[gray]{0.45}\fontsize{6.2pt}{6.4pt}\selectfont\,\textpm0.35}
& {\color[gray]{0.30}4.8\,min} \\

\textbf{Scaf.-GAT}$^{1}$
& 92.78{\color[gray]{0.45}\fontsize{6.2pt}{6.4pt}\selectfont\,\textpm0.08}
  \rlap{\smash{\raisebox{0.62ex}{
  \color[rgb]{0.706,0.275,0.059}
  \fontsize{5.2pt}{5.2pt}\selectfont\,-0.71}}}
& {\color[gray]{0.30}35.1\,s}
  \rlap{\smash{\raisebox{0.62ex}{
  \color[rgb]{0.043,0.478,0.357}
  \fontsize{5.2pt}{5.2pt}\selectfont\,1.68$\times$}}}
& 69.24{\color[gray]{0.45}\fontsize{6.2pt}{6.4pt}\selectfont\,\textpm0.00}
  \rlap{\smash{\raisebox{0.62ex}{
  \color[rgb]{0.706,0.275,0.059}
  \fontsize{5.2pt}{5.2pt}\selectfont\,-2.88}}}
& {\color[gray]{0.30}2.5\,min}
  \rlap{\smash{\raisebox{0.62ex}{
  \color[rgb]{0.043,0.478,0.357}
  \fontsize{5.2pt}{5.2pt}\selectfont\,1.92$\times$}}} \\

\textbf{Scaf.-GAT}$^{K,\mathrm{async}}$
& 93.38{\color[gray]{0.45}\fontsize{6.2pt}{6.4pt}\selectfont\,\textpm0.41}
  \rlap{\smash{\raisebox{0.62ex}{
  \color[rgb]{0.706,0.275,0.059}
  \fontsize{5.2pt}{5.2pt}\selectfont\,-0.11}}}
& {\color[gray]{0.30}41.0\,s}
  \rlap{\smash{\raisebox{0.62ex}{
  \color[rgb]{0.043,0.478,0.357}
  \fontsize{5.2pt}{5.2pt}\selectfont\,1.44$\times$}}}
& 70.76{\color[gray]{0.45}\fontsize{6.2pt}{6.4pt}\selectfont\,\textpm1.47}
  \rlap{\smash{\raisebox{0.62ex}{
  \color[rgb]{0.706,0.275,0.059}
  \fontsize{5.2pt}{5.2pt}\selectfont\,-1.36}}}
& {\color[gray]{0.30}2.4\,min}
  \rlap{\smash{\raisebox{0.62ex}{
  \color[rgb]{0.043,0.478,0.357}
  \fontsize{5.2pt}{5.2pt}\selectfont\,1.97$\times$}}} \\

\addlinespace[1pt]
\cmidrule(lr){1-5}
\addlinespace[1pt]

GIN
& 92.58{\color[gray]{0.45}\fontsize{6.2pt}{6.4pt}\selectfont\,\textpm0.10}
& {\color[gray]{0.30}26.0\,s}
& 70.30{\color[gray]{0.45}\fontsize{6.2pt}{6.4pt}\selectfont\,\textpm0.25}
& {\color[gray]{0.30}1.7\,min} \\

\textbf{Scaf.-GIN}$^{1}$
& 92.13{\color[gray]{0.45}\fontsize{6.2pt}{6.4pt}\selectfont\,\textpm0.54}
  \rlap{\smash{\raisebox{0.62ex}{
  \color[rgb]{0.706,0.275,0.059}
  \fontsize{5.2pt}{5.2pt}\selectfont\,-0.45}}}
& {\color[gray]{0.30}21.5\,s}
  \rlap{\smash{\raisebox{0.62ex}{
  \color[rgb]{0.043,0.478,0.357}
  \fontsize{5.2pt}{5.2pt}\selectfont\,1.21$\times$}}}
& 56.56{\color[gray]{0.45}\fontsize{6.2pt}{6.4pt}\selectfont\,\textpm1.51}
  \rlap{\smash{\raisebox{0.62ex}{
  \color[rgb]{0.706,0.275,0.059}
  \fontsize{5.2pt}{5.2pt}\selectfont\,-13.74}}}
& {\color[gray]{0.30}1.7\,min}
  \rlap{\smash{\raisebox{0.62ex}{
  \color[rgb]{0.043,0.478,0.357}
  \fontsize{5.2pt}{5.2pt}\selectfont\,1.06$\times$}}} \\

\textbf{Scaf.-GIN}$^{K,\mathrm{async}}$
& 91.22{\color[gray]{0.45}\fontsize{6.2pt}{6.4pt}\selectfont\,\textpm0.03}
  \rlap{\smash{\raisebox{0.62ex}{
  \color[rgb]{0.706,0.275,0.059}
  \fontsize{5.2pt}{5.2pt}\selectfont\,-1.36}}}
& {\color[gray]{0.30}26.0\,s}
  \rlap{\smash{\raisebox{0.62ex}{
  \color[rgb]{0.043,0.478,0.357}
  \fontsize{5.2pt}{5.2pt}\selectfont\,1.00$\times$}}}
& 68.91{\color[gray]{0.45}\fontsize{6.2pt}{6.4pt}\selectfont\,\textpm0.54}
  \rlap{\smash{\raisebox{0.62ex}{
  \color[rgb]{0.706,0.275,0.059}
  \fontsize{5.2pt}{5.2pt}\selectfont\,-1.39}}}
& {\color[gray]{0.30}83.5\,s}
  \rlap{\smash{\raisebox{0.62ex}{
  \color[rgb]{0.043,0.478,0.357}
  \fontsize{5.2pt}{5.2pt}\selectfont\,1.26$\times$}}} \\

\addlinespace[1pt]
\cmidrule(lr){1-5}
\addlinespace[1pt]

GraphSAGE
& 93.18{\color[gray]{0.45}\fontsize{6.2pt}{6.4pt}\selectfont\,\textpm0.03}
& {\color[gray]{0.30}27.8\,s}
& 71.71{\color[gray]{0.45}\fontsize{6.2pt}{6.4pt}\selectfont\,\textpm0.48}
& {\color[gray]{0.30}46.4\,s} \\

\textbf{Scaf.-SAGE}$^{1}$
& 92.89{\color[gray]{0.45}\fontsize{6.2pt}{6.4pt}\selectfont\,\textpm0.33}
  \rlap{\smash{\raisebox{0.62ex}{
  \color[rgb]{0.706,0.275,0.059}
  \fontsize{5.2pt}{5.2pt}\selectfont\,-0.29}}}
& {\color[gray]{0.30}19.5\,s}
  \rlap{\smash{\raisebox{0.62ex}{
  \color[rgb]{0.043,0.478,0.357}
  \fontsize{5.2pt}{5.2pt}\selectfont\,1.43$\times$}}}
& 68.88{\color[gray]{0.45}\fontsize{6.2pt}{6.4pt}\selectfont\,\textpm0.28}
  \rlap{\smash{\raisebox{0.62ex}{
  \color[rgb]{0.706,0.275,0.059}
  \fontsize{5.2pt}{5.2pt}\selectfont\,-2.82}}}
& {\color[gray]{0.30}45.3\,s}
  \rlap{\smash{\raisebox{0.62ex}{
  \color[rgb]{0.043,0.478,0.357}
  \fontsize{5.2pt}{5.2pt}\selectfont\,1.02$\times$}}} \\

\textbf{Scaf.-SAGE}$^{K,\mathrm{async}}$
& 92.53{\color[gray]{0.45}\fontsize{6.2pt}{6.4pt}\selectfont\,\textpm0.49}
  \rlap{\smash{\raisebox{0.62ex}{
  \color[rgb]{0.706,0.275,0.059}
  \fontsize{5.2pt}{5.2pt}\selectfont\,-0.65}}}
& {\color[gray]{0.30}19.2\,s}
  \rlap{\smash{\raisebox{0.62ex}{
  \color[rgb]{0.043,0.478,0.357}
  \fontsize{5.2pt}{5.2pt}\selectfont\,1.45$\times$}}}
& 70.75{\color[gray]{0.45}\fontsize{6.2pt}{6.4pt}\selectfont\,\textpm0.18}
  \rlap{\smash{\raisebox{0.62ex}{
  \color[rgb]{0.706,0.275,0.059}
  \fontsize{5.2pt}{5.2pt}\selectfont\,-0.96}}}
& {\color[gray]{0.30}41.4\,s}
  \rlap{\smash{\raisebox{0.62ex}{
  \color[rgb]{0.043,0.478,0.357}
  \fontsize{5.2pt}{5.2pt}\selectfont\,1.12$\times$}}} \\

\bottomrule
\end{tabular}

\vspace{6pt}

{
\fontsize{8.2pt}{9.8pt}\selectfont
\textbf{Applied async swaps (seeds 42 / 43)}
\par\vspace{2pt}

\setlength{\tabcolsep}{6pt}
\begin{tabular}{@{}lcc@{}}
\toprule
\textbf{Architecture}
& \shortstack{\textbf{Amazon}\\\textbf{Computers}}
& \textbf{ogbn-arxiv} \\
\midrule
GCN       & 774 / 494 & 923 / 827 \\
GAT       & 832 / 834 & 956 / 950 \\
GIN       & 631 / 619 & 875 / 873 \\
GraphSAGE & 264 / 276 & 269 / 273 \\
\bottomrule
\end{tabular}
}

\end{table}

\end{document}